\documentclass{article}

\usepackage{microtype}
\usepackage{graphicx}
\usepackage{caption}
\usepackage{subcaption}
\usepackage{booktabs}

\usepackage{hyperref}

\usepackage[accepted]{icml2021}

\usepackage{latexsym}
\usepackage{amssymb}
\usepackage{amsfonts}
\usepackage{listings}
\usepackage{mathrsfs}
\usepackage{amsmath}
\usepackage{cases}
\usepackage{algorithm,algorithmic}
\usepackage[T1]{fontenc} 
\usepackage{color}
\usepackage{amsthm}
\newtheorem{Def}{Definition}[section]
\newtheorem{Lem}[Def]{Lemma}
\newtheorem{Cor}[Def]{Corollary}
\newtheorem{theo}[Def]{Theorem}
\newtheorem{prop}[Def]{Proposition}
\newtheorem{Asp}[Def]{Assumption}
\newtheorem{Rem}[Def]{Remark}
\newcommand{\R}{\mathbb{R}}
\newcommand{\E}{\mathbb{E}}
\newcommand{\F}{\mathcal{F}}
\newcommand{\Radon}{\mathcal{M}}
\newcommand{\teach}{\circ} 
\newcommand{\sd}{\mathbb{S}^{d-1}}
\newcommand{\TV}{{\tiny\rm TV}}
\newcommand{\1}{\mbox{1}\hspace{-0.25em}\mbox{l}}
\newcommand{\T}{\mathsf{T}}
\newcommand{\id}{\mathrm{I}_d}
\newcommand{\sign}{\mathrm{sgn}}
\newcommand{\thmsign}{\emph{sgn}}
\newcommand{\X}{\mathcal{X}}
\newcommand{\n}{[n]}
\newcommand{\m}{[m]}
\newcommand{\barf}{\bar{f}}
\newcommand{\dist}{\mathrm{dist}}
\newcommand{\thmdist}{\emph{dist}}
\newcommand{\supp}{\mathrm{supp}}

\newcommand{\op}{\mathrm{op}}
\newcommand{\dif}{\mathrm{d}}

\newcommand{\inner}[2]{\langle{#1}, {#2}\rangle}
\newcommand{\grad}{\nabla_{\sd}}
\newcommand{\h}{\emph{\textsf{h}}}

\newcommand{\thmpoly}{\emph{poly}}
\newcommand{\Lip}{\textrm{Lip}}

\newcommand{\ftrue}{f^{\teach}}
\newcommand{\fhat}{\widehat{f}}
\newcommand{\LPx}{L_2(P_\X)}
\newcommand{\Eqref}[1]{Eq.~\eqref{#1}}

\icmltitlerunning{On Learnability via Gradient Method for Two-Layer ReLU Neural Networks in Teacher-Student Setting}

\begin{document}

\twocolumn[
\icmltitle{On Learnability via Gradient Method \\ for Two-Layer ReLU Neural Networks in Teacher-Student Setting}




\begin{icmlauthorlist}
\icmlauthor{Shunta Akiyama}{to}
\icmlauthor{Taiji Suzuki}{to,ri}
\end{icmlauthorlist}

\icmlaffiliation{to}{Graduate School of Information Science and Technology, The University of Tokyo, Tokyo, Japan}
\icmlaffiliation{ri}{Center for Advanced Intelligence Project, RIKEN, Tokyo, Japan}

\icmlcorrespondingauthor{Shunta Akiyama}{shunta\_akiyama@mist.i-tokyo.ac.jp}
\icmlcorrespondingauthor{Taiji Suzuki}{taiji@mist.i.u-tokyo.ac.jp}

\icmlkeywords{Machine Learning, ICML}

\vskip 0.3in
]



\printAffiliationsAndNotice{}  

\begin{abstract}
Deep learning empirically achieves high performance in many applications, but its training dynamics has not been fully understood theoretically. In this paper, we explore theoretical analysis on training two-layer ReLU neural networks in a teacher-student regression model, in which a student network learns an unknown teacher network through its outputs. We show that with a specific regularization and sufficient over-parameterization, the student network can identify the parameters of the teacher network with high probability via gradient descent with a norm dependent stepsize even though the objective function is highly non-convex. The key theoretical tool is the measure representation of the neural networks and a novel application of a dual certificate argument for sparse estimation on a measure space. We analyze the global minima and global convergence property in the measure space. 
\end{abstract}

\section{Introduction}
Deep learning empirically achieves high performance in many applications, such as computer vision and speech recognition. 
To explain its success from the theoretical view point, we need to reveal its optimization dynamics and the generalization ability of the solution that is obtained by a particular optimization method such as gradient descent.
However, its training dynamics has not been fully understood theoretically and thus the generalization ability of the  solution is still an open question. 
One of the difficulties of this problem is non-convexity of the associated optimization problem \cite{li2018visualizing} for the optimization aspect, and the high dimensionality induced by over-parameterization for the generalization aspect. 
In this study, we tackle these two problems in a teacher-student problem with the ReLU activation under an over-parameterized setting.
In this setting, we need to take care of the non-differentiability of the ReLU activation and the over-specification problem due to the over-parameterization which potentially causes difficulty to show favorable generalization ability such as exact recovery.


The teacher-student setting is one of the most common settings for theoretical studies, e.g., 
 \citet{tian2017analytical,safran2018spurious,goldt2019dynamics,pmlr-v89-zhang19g,safran2020effects,pmlr-v119-tian20a,pmlr-v125-yehudai20a,suzuki2020benefit,arXiv:Zhou+Rong+Jin:2021}
to name a few. 
\citet{zhong2017recovery} studied the case where the teacher and student have the same width,  
showed that the strong convexity holds around the parameters of the teacher network and proposed a special tensor method for initialization to achieve the global convergence to the global optimal.
However, its global convergence is guaranteed only for a special initialization which excludes a pure gradient descent method. Moreover, the over-parameterized setting is not included in their analysis.  
\citet{safran2018spurious} empirically showed that gradient descent is likely to converge to non-global optimal local minima, even if we prepare a student that has the same size as the teacher. 
More recently, \citet{pmlr-v125-yehudai20a} showed that even in the simplest case where the teacher and student have the width {\it one}, there exists distributions and activations in which gradient descent fails to learn. 
\citet{safran2020effects} showed the strong convexity around the parameters of the teacher network in the case where the teacher and student have the same width for Gaussian inputs. 
They also studied the effect of over-parameterization and showed that over-parameterization will change the spurious local minima into the saddle points. However, it should be noted that this does not imply that a gradient descent can reach the global optima.


To alleviate the non-convexity of neural network optimization, over-parameterization is one of the promising approaches.
Indeed, it is fully exploited by (i) Neural Tangent Kernel (NTK) \cite{jacot2018neural,allen2019convergence,arora2019fine,du2019gradient,weinan2020comparative,zou2020gradient} and (ii) mean field analysis \cite{nitanda2017stochastic,chizat2018global,mei2019mean,tzen2020mean,NEURIPS2020:chen,chizat2019sparse,suzuki2020benefit}.
(i) In the setting of NTK, the gradient descent of neural networks can be seen as the convex optimization in RKHS, and thus it is easier to analyze. On the other hand, in this regime, it is hard to explain the superiority of deep learning, because the estimation ability of the obtained estimator is reduced to that of the corresponding kernel. 
(ii) In the setting of the mean field analysis, 
a kind of continuous limit of neural network is considered and its convergence to some specific target functions has been analyzed. This regime is more suitable in terms of a ``beyond kernel'' perspective, but it essentially deals with a continuous limit and hence is difficult to show convergence to a teacher network with a \textit{finite width}. 

In this paper, we make full use of the ``measure representation'' of two-layer ReLU networks as in the mean field analysis, while our approach employs a \textit{sparse} regularization on the measure of parameters to show the convergence of a gradient descent method to the global optimum where the teacher network has a finite width.
The sparse regularization on a measure space is well studied in a so-called \textit{BLASSO} problem \cite{de2012exact}.
Indeed, \citet{chizat2019sparse} analyzed the gradient descent for two layer neural networks from the view point of  BLASSO analyses, and showed the convergence to the global optimal.  
However they assumed several assumptions which are hard to clarify, and excluded a non-smooth activation such as the ReLU activation. 
On the other hand, we explicitly present a realistic condition under which a gradient descent converges to the global optimum. 
More specifically, our contributions can be summarized as follows:
    \begin{itemize}
\item We show that with an appropriate sparse regularization, 
the optimal solution of a regularized empirical risk 
can be arbitrarily close to the true teacher-parameters for a sufficiently small regularization parameter. 
This implies effectiveness of a sparsity inducing regularization in deep learning.
\item We prove that a gradient descent with a norm-dependent step size can converge to the global optimum of the regularized learning problem if the student network is appropriately over-parameterized. 
        \item Combining the above results, we show that a gradient descent method with an over-parameterized initialization can find a network which is arbitrary close to the true teacher network. In particular, the size of the estimated network becomes ``narrow'' even though the initial solution is over-parameterized, which explains the feature learning ability of neural networks leading a better performance than shallow methods such as kernel methods.
    \end{itemize}
    \subsection{Other Related Works}

\paragraph{BLASSO problem} The BLASSO problem \cite{de2012exact} is a regression problem with total variation regularization on a measure space, which is an extension of the LASSO problem to the measure space. One of the main  theoretical interests of BLASSO studies \cite{bredies2013inverse,candes2013super,duval2015exact,poon2018geometry,poon2019support} is to clarify whether the global minima of BLASSO can recover the ``true'' measure in the setting where the true measure is sparse, i.e., given by a sum of Dirac measures.
\citet{duval2015exact} showed that for a sufficiently small sample noise and an appropriate regularization, the global minimum will also be sparse and close to the true measure. A key theoretical tool is a dual certificate, which is motivated by the Fenchel duality. However, their analysis assumes smoothness on the objective function and thus is not directly applied to our setting because of the non-differentiability of the ReLU activation.

\paragraph{Sparse regularization}
It has been shown that explicit or implicit sparse regularization such as $L_1$-regularization is beneficial to obtain better performances of deep learning under certain situations \cite{klusowski2016risk,gunasekar2018implicit,chizat2020implicit,pmlr-v125-woodworth20a}.
However, it is still an open question that a gradient descent can find the teacher model in a regression setting with the ReLU non-linear activation.
\citet{JMLR:v18:14-546} analyzed a neural network model with a sparse regularization ($L_1$-regularization) 
which can be regarded as an extension of Barron class \cite{barron1993universal}, and derived its model capacity. 
It was shown that the Frank-Wolfe type method can estimate a target function in the neural network model,
but unfortunately this does not imply that a gradient descent method can estimate the target function.
Moreover, it is not clear that each update of the Frank-Wolfe method is computationally tractable.

\paragraph{Langevin dynamics approach}
The gradient Langevin dynamics (GLD) is a useful approach to obtain a global optimum of a non-convex objective function 
\cite{Welling_Teh11,raginsky2017non,NIPS2018_8175,suzuki2020benefit}.
This approach can be also applied to neural network optimization but such analysis would not give any information about the landscape of the neural network training. 
Among them, \citet{suzuki2020benefit} considered an infinite dimensional Langevin dynamics, but they excluded a non-differentiable activation such as ReLU and did not give any landscape analysis.

\subsection{Notations}
Here we give some notations used in the paper.
Let $\Radon(\mathcal{C})$ be the set of the Radon measures on a topological space $\mathcal{C}$ (we consider the Borel algebra of $\mathcal{C}$ as the $\sigma$-field on which the Radon measures are defined).
Let $\delta_w(\cdot)$ be the Dirac measure on $w \in \R^d$, i.e., $\int f(x) \delta_w (\mathrm{d} x) = f(w)$.
Let $[m] := \{1,\dots,m\}$ for a positive integer $m$.
Let the inner product between $x,y \in \R^d$ be $\langle x,y \rangle :=\sum_{j=1}^d x_i y_i$.

\section{Problem Settings}

In this section, we give the problem setting and the model that we consider in this paper.
We focus on a regression problem where we observe $n$ training examples $D_n = \{(x_1, y_1),\dots ,(x_n, y_n)\} \subset  \R^d\times \R$ generated by the following model:
        \begin{equation}
            y_i = f^\teach(x_i), 
        \end{equation}
where $\ftrue:\R^d \to \R$ is the unknown true function that we want to estimate, $(x_i)_{i=1}^n$ are independently identically distributed from $P_\X$. Later on we assume that $P_\X$ is the uniform distribution on the unit ball $\sd$ (Assumption \ref{sampleasp}).

        Based on the observed data $D_n$, we construct an estimator $\hat{f}$ which is supposed to be ``close'' to the true function $\ftrue$. As its performance measure, we employ the mean squared error defined by 
$\|\fhat - \ftrue \|_{\LPx}^2 := \E_{X \sim P_\X}[(\fhat(X) - \ftrue(X))^2]$. 
Its empirical version is defined by $\|\fhat - \ftrue\|_n^2 := \frac{1}{n} \sum_{i=1}^n (\fhat(x_i) - \ftrue(x_i))^2$.


    \paragraph{Teacher-Student Model}
In this section, we prepare the teacher-student model that we consider in this paper.
The student model is the two-layer neural network with the ReLU-activation $\sigma(u)=\max\{x,0\}$ \cite{glorot2011deep} and width $M$, which is defined as
        \begin{equation}
            f(x;\Theta) = \sum_{j=1}^Ma_j\sigma(\langle w_j,x\rangle),
        \end{equation}
       where $\Theta=((a_1, w_1), \dots, (a_M, w_M))\in (\R\times\R^d)^M$ is the trainable parameter. 
The teacher model is assumed to be included in the student model but the width could be smaller than $M$:
\begin{equation}
    \ftrue(x) = \sum_{j=1}^m a^\teach_j \sigma(\langle w^\teach_j,x \rangle),
\end{equation}
where $m$ is the width of the teacher model and $(a^\teach_j , w^\teach_j) \in \R\times\R^d~(j \in [m])$.
We consider an over-parameterized setting where $m \leq M$ is assumed to be satisfied.
Hence, the teacher model can be regarded as an element of the student model by setting $a_j =0$ for $j= m+1,\dots,M$.
For notational simplicity, we denote by $\Theta^\teach := (a^\teach_j,w^\teach_j)_{j=1}^m \in (\R \times \R^d)^m$.

        For a neural network model, it is generally difficult to write the close form of the (regularized) empirical risk minimizer. Therefore, we typically optimize $\Theta$ via the gradient descent technique, but due to the non-convexity of the objective function, it is far from trivial that the global minima can be obtained by gradient descent. 

\paragraph{Sparse Regularized Empirical Risk}
%

To estimate the true parameter $\Theta^\teach$, 
we define the following regularized empirical risk minimization problem on the parameter space $(\R\times\R^d)^M$:
\begin{align}
\!\!\! \! \min_{\Theta \in (\R\times\R^d)^M} 
\frac{1}{2n}\sum_{i=1}^n(y_i-f(x_i;\Theta))^2 \!+ \!\lambda\!\sum_{j=1}^M|a_j|\|w_j\|,
\label{parameterproblem}
\end{align}
where $\lambda \geq 0$ is a regularization parameter.
The regularization term $\lambda\sum_{j=1}^M|a_j|\|w_j\|$ can be seen as an $L_1$-regularization which induces sparsity. 
Indeed, by the scale homogeneity of ReLU ($a_j \sigma(\langle w_j, x\rangle) = a_j \|w_j\| \sigma(\langle w_j/\|w_j\|, x \rangle)$), we may reset the parameter as $a_j' = a_j \|w_j\|$ and $w_j' = w_j/\|w_j\|$ and then the regularization term can be rewritten as $\lambda \sum_{j=1}^M |a_j'|$. 
Apparently, this is the $L_1$-norm of $(a'_j)_{j=1}^M$.

In practice, we typically use the $L_2$-regularization $\frac{\lambda}{2} \sum_{j=1}^M (a_j^2 + \|w_j\|^2)$ instead of the $L_1$-regularization as induced above.
However, the arithmetic-geometric mean relation yields that 
\begin{align}\label{eq:L2L1connection}
\textstyle
|a_j|\|w_j\| = \!\!\! \min\limits_{ \substack{(a'_j, w'_j) \in \R \times \R^d: \\  |a_j|\|w_j\| = |a'_j|\|w'_j\|}} 
\frac{1}{2}(|a'_j|^2 + \|w'_j\|^2).  
\end{align}
Therefore, our sparse regularization can be replaced by the $L_2$-regularization. 
In this paper, we directly consider the sparse regularization instead just for simplicity. 
%

\begin{Rem}
We will see that the regularization term $\lambda\sum_{j=1}^M|a_j|\|w_j\|$ corresponds to the {\it total-variation} norm regularization for the measure representation of the network which we refer to in the next section. 
The same type of regularization has been considered in several studies, e.g., \citet{neyshabur2015norm,ma2019priori}. 
In those studies, it plays an important role to show a better performance of deep learning compared with kernel methods.
We further make full use of the sparsity to show the exact recovery of the true parameter $\Theta^\teach$ even under the over-parameterized setting.
        \end{Rem}

\section{Global Minima in the Teacher-Student Setting}
In this section, we show that the minimizer of the regularized empirical risk \eqref{parameterproblem} is arbitrarily close to the teacher network $\ftrue$ for a sufficiently large sample size $n$.
Note that we are not arguing here that the optimal solution can be obtained by the gradient descent, but the computational issue will be addressed in the next section. 
We make the following assumptions for our analysis.
    \begin{Asp}
        \label{sampleasp}
        $(x_i)_{i=1}^n$ are i.i.d. observations from the uniform distribution on $\sd$, that is, $P_\X = \mathrm{Unif}(\sd)$. 
    \end{Asp}
    \begin{Asp}\label{teacherasp}
        The teacher network $\ftrue = \sum_{j=1}^ma^\teach_j\sigma(\langle w^\teach_j,\cdot \rangle)$ satisfies the following conditions:
        \begin{enumerate}
            \item{$a_j^\teach >0~~~(\forall j\in \m)$}.
            \item{$\langle w^\teach_{j_1},w^\teach_{j_2} \rangle = 0~~~(\forall j_1, j_2\in\m, ~j_1\neq j_2)$}.
        \end{enumerate}
    \end{Asp}
The second assumption could be a bit strong, but the same assumption has been considered in several previous researches \cite{zhong2017recovery,tian2017analytical,safran2018spurious,safran2020effects,li2020learning}. 
For example, \citet{safran2020effects} analyzed the landscape of the objective under this assumption and 
showed a negative result that the loss landscape around the global minima is not even \textit{locally convex}. 
On the other hand, they also showed that an over-parameterization turns a non-global optimal point into a saddle-point.
However, they have not shown that a gradient descent can reach the optimal solution.
\citet{li2020learning} showed a global optimality of gradient descent in a specific teacher student setting under this condition. They consider a specific teacher model $\ftrue(x) = \sum_{j=1}^M a_j^\teach |\inner{x}{\theta^\teach_j}|$ for $a^\teach_j > 0$ and a student model $f(x;W) = \frac{1}{M} \sum_{j=1}^M \|w_j\| \sigma(\inner{w_j}{x})$.
This is relevant to ours, but specification of the teacher network is quite different from our setting.

The main ingredient of our analysis is the \textit{measure representation} of the two layer ReLU-neural network.
Using this representation, one can regard the neural network training as a sparse regularized learning on the measure space. This enables us to show (near) exact recovery. 
In particular, the \textit{Beurling-LASSO~(BLASSO)} analysis \cite{de2012exact} which could be seen as an infinite dimensional extension of sparse regularization theory is helpful.

\subsection{Mesure Representation of Two-Layer Neural Networks and BLASSO Problem}
We introduce the measure representation of the two-layer ReLU neural network. 
By using 1-homogeneity of the ReLU activation, it holds that
        \begin{equation}
            \begin{split}
            \sum_{j=1}^M a_j\sigma(\langle w_j,x\rangle) &= \sum_{j=1}^M a_j \|w_j\| \sigma\left(\left\langle \frac{w_j}{\|w_j\|} ,x\right\rangle\right)\\
            & =\int_{\sd}\sigma(\langle\theta,x\rangle)\dif\nu(\theta)
            \end{split}
        \end{equation}
        with $\nu = \sum_{j=1}^m a_j\|w_j\|\delta_{w_j/\|w_j\|} \in \Radon(\sd)$. We call this $\nu$ a measure representation of the two-layer ReLU neural network. In the following, we write 
        \begin{equation}
            f(x;\nu) = \int_{\sd}\sigma(\langle\theta,x\rangle)\dif\nu(\theta).
        \end{equation}
        Under this representation, the teacher network is represented as $\nu^\teach = \sum_{j=1}^m r^\teach_j\delta_{\theta^\teach_j}$ with $r_j = a^\teach_j\|w^\teach_j\|$ and $\theta^\teach_j = w^\teach_j/\|w^\teach_j\|$. 
        \begin{Rem}
            For a more general activation $\sigma$, we need to consider a measure on the product space $\R\times \R^d$. However, thanks to the 1-homogeneity of ReLU,  we only need to consider a measure on $\sd$ which is a compact metric space. 
        \end{Rem}
        With this measure representation, we may consider the following regression problem on the measure space instead: 
        \begin{equation}
            \label{measureproblem}
            \min_{\nu\in \Radon(\sd)} \frac{1}{2n}\sum_{i=1}^n\left(y_i-f(x_i;\nu)\right)^2+\lambda\|\nu\|_\TV,
        \end{equation}
where $\|\cdot \|_{\TV}$ is the {\it total variation} norm of $\nu \in  \Radon(\sd)$ that is defined by $\|\nu\|_{\TV} = \nu_+(\sd) + \nu_-(\sd)$ for the Hahn–Jordan decomposition $\nu(\cdot) =\nu_+(\cdot) - \nu_-(\cdot)$. 
        This can be seen as the continuous version of  the original problem \eqref{parameterproblem}, which is called a BLASSO problem \cite{de2012exact}. Since the measure representation covers any finite-width neural network, the following proposition holds.
        \begin{prop}\label{prop:finitesumequaloptimalmeas}
            Assume that a global minimum of \eqref{measureproblem} is obtained by a measure which is represented as a finite sum of Dirac measures:
            \begin{equation*}
            \textstyle 
                \nu^*=\sum_{j=1}^{m^*} r_j^*\delta_{\theta_j^*}, 
            \end{equation*}
            then for the student network satisfying $M\geq m^*$, the global minima of (\ref{parameterproblem}) can be obtained by the form whose measure representation is written by $\nu^*$. 
        \end{prop}
There have been several studies that focused on the global minimum of the BLASSO problem \eqref{measureproblem}.  \citet{duval2015exact} analyzed this problem in the context of sparse spike deconvolution, in which $f$ is a Gaussian convolution filter and is an element of $L_2(\mathbb{T})$ (where $\mathbb{T}$ denotes the 1-dimensional torus), and showed that under the so-called {\it NDSC condition}, the global minima can be close to underlying measure. 
        \citet{poon2018geometry,poon2019support} analyzed a more general setting and derived a sufficient condition for the NDSC condition.
However, these analyses have required smoothness on the objective. 
Therefore, they can not be applied directly to our setting because of non-differentiability of the ReLU activation. 
We overcome this difficulty by directly deriving the {\it dual certificate} of the optimization problem.

    \subsection{Main Result 1: Global Minima of Regularized Empirical Risk}
        We prove that with a sufficiently small regularization parameter, the global minimizer of \eqref{measureproblem} is close to the teacher network with an arbitrarily small gap. We state this as the following theorem. 
        \begin{theo}
            \label{main}
Assume that Assumptions \ref{sampleasp} and \ref{teacherasp} are satisfied.
Suppose that $n>\thmpoly(m,d,\log1/\delta)$ for $\delta > 0$. Then, with probability at least $1-\delta$, we have that with sufficiently small $\lambda>0$, the optimal solution of \eqref{measureproblem} is uniquely determined and written by the form $\nu^* = \sum_{j=1}^m r_j^*\delta_{\theta_j^*}$ where $(r_j^*,\theta_j^*)_{j=1}^m \subset \R \times \sd$ satisfy 

    \begin{align}
            \label{exactrecovery}
        \begin{cases}
        \sum_{j=1}^m|r_j^\teach-r_j^*|^2\leq \emph{O}(m\lambda^2)\\ \sum_{j=1}^m\thmdist^2(\theta^*_j,\theta^\teach_j)\leq \emph{O}(m\lambda^2)\ \ 
        \end{cases}.
    \end{align}
        \end{theo}
The proof can be found in Appendix A. 
From this theorem and Proposition \ref{prop:finitesumequaloptimalmeas}, we immediately obtain the following corollary.
        \begin{Cor}
            Under the same assumption with Theorem~\ref{main}, for the student network model with more than $m$ nodes, the optimal solution of \eqref{parameterproblem} achieves the same property with Theorem~\ref{main}, i.e., the measure representation of the optimal network satisfies \eqref{exactrecovery}.
        \end{Cor}
Therefore, as long as the network size $M$ is sufficiently large such that $M \geq m$, we can recover the true network with arbitrarily small error by tuning the regularization parameter. The event of this property is uniform over the choice of the accuracy $\epsilon$ and corresponding regularization parameter $\lambda$. Hence, by decreasing $\lambda$ gradually, we can finally recover the teacher model exactly.
This result only characterizes the globally optimal solution and it does not say anything about the algorithmic convergence of a gradient descent method. In the next section, we address this issue.

\paragraph{Proof Strategy: Dual Certificate}
Theorem \ref{main} can be shown through a dual certificate characterization of the optimal solution. 
Let the optimization problem \eqref{measureproblem} be $P_\lambda$. 
By the Fenchel's duality theorem \cite{Rock,vari,duval2015exact}, its dual problem $D_\lambda$ is given by 
\begin{flalign*}
\text{($D_\lambda$)} & & \max_{p \in \R^n: \|f^*(p)\|_\infty \leq 1}~ \frac{1}{n^2} \sum_{i=1}^n y_i p_i - \frac{\lambda}{2n^2} \|p\|^2, & &
\end{flalign*}
where $f^*(p)(\cdot) \in \mathcal{C}(\sd)$\footnote{$\mathcal{C}(S)$ is the set of continuous functions on a topological space $S$.} that is defined by 
$f^*(p)(\theta) := \frac{1}{n} \sum_{i=1}^n \sigma(\inner{\theta}{x_i})$, and the strong duality holds, that is, 
$\nu^*_\lambda$ is the optimal solution of $P_\lambda$ if 
the following optimality condition is satisfied for the {\it unique} solution $p_\lambda$ of $D_\lambda$ (the uniqueness of the dual solution follows from the strong convexity of the dual problem): 
\begin{align*}
\begin{cases}
 f^*(p_\lambda) \in \partial \|\nu^*_\lambda\|_{\mathrm{TV}}, \\
 p_{\lambda,i} = - \frac{1}{\lambda} (f(x_i;\nu^*_\lambda) - y_i)~~~(\forall i \in [n]).
\end{cases}
\end{align*}
We call $ f^*(p_\lambda)$ a {\it dual certificate} for $\nu^*_\lambda$.
Conversely, if this condition is satisfied by $(\nu^*_\lambda, p_\lambda) \in \mathcal{M}(\sd) \times \R^n$, then the pair is the optimal solution of both $P_\lambda$ and $D_\lambda$. 
Therefore, our strategy is to show that the dual certificate $f^*(p_\lambda)$ admits only a primal optimal solution $\nu^*_\lambda$ that satisfies the condition in the theorem, i.e., the support of $\nu^*_\lambda$ consists of only $m$ distinct points each of which is close to the true parameters $(\theta_j^\teach)_{j=1}^m$. 
To prove this, we show that there exist $(\theta_j^*)_{j=1}^m$ such that $(\dist(\theta_j^*, \theta_j^\teach))_{j=1}^m $ are sufficiently small and satisfy 
\begin{align}
\label{eq:fpstarDualCond}
\begin{cases}
 f^*(p_\lambda)(\theta_j^*) = 1~~(\forall j \in [m]), \\
 |f^*(p_\lambda)(\theta)| < 1~~(\forall \theta \in \sd /\{\theta_1^*,\dots,\theta_m^*\})
\end{cases}
\end{align}
for sufficiently small $\lambda$. 
From this inequality, we can show that $(|r_j^* - r_j^\teach|)_{j=1}^m$ will also be sufficiently small. Finally by using the form $\nu^*=\sum_{j=1}^mr_j^*\delta_{\theta_j^*}$ and strong convexity of the empirical risk term in $P_\lambda$ w.r.t. $r^*_j$ and $\theta^*_j$ around the teacher parameters $(r_j^\teach,\theta_j^\teach)_{j=1}^m$, we get the quantitative bound as \Eqref{exactrecovery}. 

For that purpose, we particularly consider a setting where $\lambda = 0$, and consider the minimal norm certificate:
$$
p_0 := \min\{\|p\| \mid \text{$p \in \R^n$ is a feasible solution of $D_0$} \}.
$$
The most difficult pint in our analysis is to show the property \eqref{eq:fpstarDualCond} for the minimal norm certificate $p_0$.
This is accomplished by carefully evaluating the analytic form of $f^*(p_0)$.
Indeed, by using the orthogonality of $(\theta_j^\teach)_{j=1}^m$ and the fact that the input distribution is the uniform distribution, we can write down the minimal norm certificate and analyze it.


\section{Global Convergence of Gradient Method}

In this section, we investigate a gradient descent method for the optimization problem \eqref{parameterproblem}. 
We show that under some assumptions, a gradient descent with a norm-dependent step size converges to the global optimum of the problem. 
We also show that these assumptions for the global convergence are satisfied under the conditions we made in the previous section, which implies the identifiability of the teacher parameters through the gradient descent method. 
        
\subsection{Norm-Dependent Gradient Descent}


We consider a standard gradient descent for optimizing the objective \eqref{parameterproblem}. 
To incorporate the 1-homogeneity of the ReLU activation function, we employ a step size that can be dependent on the norm of each parameter. As we see in proof of the global convergence, this norm dependency is helpful to describe an update in the measure space. 
Let $F$ be the regularized empirical risk given in \eqref{parameterproblem}, that is, 
$F(\Theta):=\frac{1}{2n}\sum_{i=1}^n(y_i-f(x_i;\Theta))^2+\lambda\sum_{j=1}^M|a_j|\|w_j\|$.
Then, the update rule of the norm-dependent gradient descent can be written as 
        \begin{align*}
            a_{j, k+1} = a_{j, k} -\eta_{j,k}g_j(\Theta_k)~\text{for}~ &g_j(\Theta_k) \in \partial_{a_j} F(\Theta_k), \\
            w_{j, k+1} = w_{j, k} - \eta_{j,k}h_j(\Theta_k)~\text{for}~ &h_j(\Theta_k)\in \partial_{w_j} F(\Theta_k),
        \end{align*}
        where 
$\Theta_k=((a_{1,k},w_{j,k}),\dots,(a_{M,k},w_{M,k}))$ is the parameter after $k$ iterations, 
$\eta_{j,k}>0$ is the norm-dependent step size which will be specified below. 
$\partial_{a} F(\Theta)$ denotes the sub-gradient of $F(\Theta)$ as a function of $a$.
The sub-gradient is not always a singleton, but we employ the following one as $g,h$:
        \begin{align*}
g_j(\Theta)\!  = &\frac{1}{n}\sum_{i=1}^n(f(x_i;\Theta) \!- \! y_i)\sigma(\inner{w_j}{x_i}) \! +\!\lambda ~\sign(a_j)\|w_j\|, \\
h_j(\Theta) \! = &\frac{1}{n}\! \sum_{i=1}^n \! (f(x_i;\Theta) \! -\! y_i)a_jx_i\!\1\{\!\inner{w_j}{x_i} \!\geq \! 0\} 
\!+\! \lambda \frac{|a_j| w_j}{\|w_j\|},
        \end{align*}
As for the norm-dependent step size $\eta_{j,k}$, we employ the following representation: 
        \begin{align}\label{eq:NormDepEta}
            \eta_{j,k} &= \alpha\frac{|a_{j,k}|\|w_{j,k}\|}{a_{j,k}^2+\|w_{j,k}\|^2},
        \end{align}
        where $\alpha>0$ is a fixed constant. 
For the initialization, we consider the mean-field setting where each $a_{j,0} = \textrm{O}(1/M)$: 
        \begin{align*}
            a_{j,0} &= \frac{2}{M}~~~~~(1\leq j \leq M/2), \\ 
            a_{j,0} &= -\frac{2}{M}~~~~~(M/2+1\leq j\leq M), \\
            w_{j,0} &\overset{{\rm i.i.d.}}{\sim} {\rm Unif}(\sd). 
        \end{align*}
With the norm-dependent step size, the sign of $a_{j,k}$ will not be changed during the optimization, and thus we need the both positive and negative sign initializations for $(a_{j,0})_{j=1}^M$.
As pointed out by several authors \cite{chizat2018global,mei2019mean,chizat2019sparse,suzuki2020benefit},
it is essentially important to analyze the dynamics of ``feature learning'' in the mean field regime where each node is adaptively updated to represent the target function efficiently.
This is in contrast to NTK analysis (a.k.a., lazy training regime) where the basis functions are almost fixed during the optimization.
The algorithm is summarized in Algorithm \ref{algo:normdependent}.

\begin{algorithm}[tbp]
    \caption{Norm-Dependent Gradient Descent}
    \begin{algorithmic}[1]
        \renewcommand{\algorithmicrequire}{\textbf{Input:}}
        \renewcommand{\algorithmicensure}{\textbf{Output:}}
        \REQUIRE student width $M$ (even), max iteration $K$, stepsize parameter $\alpha>0$.
        \\ 
        \textit{Initialization} : $a_{j,0}=2/M, 1\leq j\leq M/2$, $a_{j,0}=-2/M, M/2+1\leq j\leq M$, $w_{j,0}\sim {\rm Unif}(\sd)$
        \FOR {$k = 1, 2, \dots, K$}
            \FOR {$j = 1, \dots, M$}
                \STATE $\eta_{j,k}= \alpha\frac{|a_{j,k}|\|w_{j,k}\|}{a_{j,k}^2+\|w_{j,k}\|^2}$
                \STATE choose $g_j(\Theta_k) \in \partial_{a_j} F(\Theta_k)$
                \STATE choose $h_j(\Theta_k) \in \partial_{w_j} F(\Theta_k)$	
                \STATE $a_{j, k+1} = a_{j, k} -\eta_{j,k}g_j(\Theta_k)$
                \STATE $w_{j, k+1} = w_{j, k} - \eta_{j,k}h_j(\Theta_k)$
            \ENDFOR
        \ENDFOR
    \end{algorithmic} 
    \label{algo:normdependent}
\end{algorithm}

The global optimality of the gradient descent can be shown through the measure representation of the neural network.
Indeed, we have seen in the previous section that the optimization problem of a neural network model can be generalized to the BLASSO problem on the measure space as presented in \Eqref{measureproblem}. Let $J$ be the BLASSO objective function on the measure space: 
            $J(\nu) = \frac{1}{2n}\sum_{i=1}^n\left(y_i-f(x_i;\nu)\right)^2+\lambda\|\nu\|_\TV.$
Note that in the over-parameterized setting, we cannot formally define the convergence of the parameter $\Theta_k$ to  the true one $\Theta^\teach$ because they have different dimensionality.
Therefore, we consider convergence of the measure corresponding to the parameter $\Theta$ instead. 
We assume ``sparsity'' of the global minima of $J$ on the measure space to ensure the convergence of the measure representation as follows.
        \begin{Asp}
            \label{ass:sparse}ar
            The global minimum of $J$ is uniquely attained by a sum of Dirac measures: 
            \begin{equation}
                \nu^* := \sum_{j=1}^{m^*}r_j^*\delta_{\theta_j^*},
            \end{equation}
            where $m^*$ is a positive integer, $r_j^* \neq 0,~\theta_j^* \in \sd~(j\in[m^*])$ and $\theta_j^*\neq \theta_{j'}^*$ for any $j\neq j'$.
        \end{Asp}
    
\begin{Rem}
Note that this condition can be satisfied under Assumptions \ref{sampleasp} and \ref{teacherasp} by Theorem \ref{main}.
\end{Rem}
    
By the same argument as Proposition \ref{prop:finitesumequaloptimalmeas}, if we set $M\geq m^*$, the sparsity and uniqueness of the global minimum of $J$ leads to the existence of the global minimum of $F$, which is essentially represented by $m^*$ nodes.
Even in this case, by the non-convexity of $F$, it is far from trivial to show the convergence of the gradient method to the global optimal solution. 
As we have stated, we show this through the measure representation of the network. 

To show the result, we prepare some additional notations.
For the intermediate solution $\Theta_k = \{(a_{j,k},w_{j,k})\}_{k=1}^M$, we define  
$r_{j,k} = a_{j,k}\|w_{j,k}\|, \theta_{j,k} = \frac{w_{j,k}}{\|w_{j,k}\|}$ (if $\|w_{j,k}\| = 0$, we set $\theta_{j,k}$ be arbitrary fixed point in $\sd$).  Accordingly, the measure representation corresponds to $\Theta_k$ be 
\begin{equation*}
\textstyle \nu_k := \sum_{j=1}^M r_{j,k}\delta_{\theta_{j,k}}.
\end{equation*} 
    For two Radon measures $\mu_1,\mu_2 \in \mathcal{M}(\sd)$, $W_\infty(\mu_1,\mu_2)$ denotes the Wasserstein distance between them: $W_\infty(\mu_1,\mu_2):=\underset{\gamma \in \Pi(\mu_1,\mu_2)}{\inf}\underset{(\theta_1,\theta_2)\in {\rm supp}(\gamma)}{\sup}~\dist(\theta_1,\theta_2)$, where $\Pi(\mu_1,\mu_2)$ is a set of product measures with marginals $\mu_1$ and $\mu_2$, $\mathrm{supp}(\gamma)$ is the support of $\gamma$, and $\dist(\theta_1,\theta_2) := \arccos(\langle \theta_1, \theta_2 \rangle)$ for $\theta_1,\theta_2 \in \sd$. 
    
Since $f(x;\nu)$ is a linear model with respect to $\nu$ and the squared loss is differentiable, the Fr\'echet subdifferential of $J(\nu)$ on $\mathcal{M}(\sd)$ can be defined and be represented as a set of functions $G(\cdot): \sd \to \R$ defined by
$$
G(\theta) = \frac{1}{n} \sum_{i=1}^n (f(x_i;\nu) - y_i)  \sigma(\inner{\theta}{x_i}) + \lambda \eta(\theta),
$$ 
where $\eta\in \mathcal{C}(\sd)$ satisfies $\|\eta\|_\infty\leq 1 $ and $\int\eta\dif \nu = \|\nu\|_\TV$.
Note that we have that $\partial J(\nu_k) := \{G \in \mathcal{C}(\sd) \mid J(\mu) - J(\nu_k) \geq \int G(\theta) \dif (\mu - \nu_k)~\textrm{for~any}~ \mu \in \mathcal{M}(\sd)\}$ which is well defined because $J(\cdot)$ is a convex function on the measure space $\mathcal{M}(\sd)$.

\subsection{Main Result 2: Global Optimality of Gradient Method}
Here, we give the global convergence property of the norm-dependent gradient descent under a bit milder conditions than those assumed in the previous section. The analysis basically follows that of \citet{chizat2019sparse}, but they assumed smoothness on the activation and excluded the ReLU activation. 
To overcome this difficulty, our norm-dependent step size (\Eqref{eq:NormDepEta}) plays the important role. 
Moreover, we carefully divide the parameter space into ``smooth region'' and ``non-smooth irrelevant-region'' to show a descent property of the objective.
The assumptions below are made under a condition of a training data observation $D_n = (x_i,y_i)_{i=1}^n$.
        \begin{Asp}[Non-orthogonality between $x$ and $\theta$]
            \label{ass:smooth}
            For any $i \in \n, j\in\m^*$, we have $\inner{x_i}{\theta^*_j}\neq 0$. 
        \end{Asp}
        \begin{Asp}[Strong convexity w.r.t. $r$]
            \label{ass:strongconvexity}
            There exists a constant $\kappa>0$ such that for any $r_1, \dots, r_m\in \R$, $\|\sum_{j=1}^mr_j\sigma(\inner{\theta^*_j}{ \cdot})\|_n^2\geq \kappa (r_1^2 + \dots + r_m^2)$.
        \end{Asp}
        \begin{Asp}[Non-degeneracy]
            \label{ass:nondegenerate}
            There exists no $\theta\notin \supp(\nu^*)$ such that $J'(\nu^*)(\theta)= 0$. 
        \end{Asp}
        \begin{Asp}[Boundedness]
            \label{ass:boundedness}
            There exists a constant $C_{F}>0$ such that, for any $k$, it holds that $F(\Theta_k)\leq C_{F}$.
        \end{Asp}
        \begin{Asp}[Boundedness of input]
        \label{ass:bondedinput}
    $\|x_i\|\leq 1$ for all $i \in [n]$.
        \end{Asp}
Assumption \ref{ass:smooth} is satisfied almost surely if $x_i \sim \mathrm{Unif}(\sd)$. 
This is required to ensure the smoothness of the objective around the optimal parameter $(r_j^*,\theta_j^*)_{j=1}^{m^*}$. 
Otherwise the objective function $F$ is non-differentiable at the global optimal with respect to $\theta_j$, which causes difficulty to show the local convergence around the global optimal.
Assumption \ref{ass:strongconvexity} is also almost surely satisfied if the nodes $x \mapsto \sigma(\langle x,\theta^*_j\rangle)~(j \in [m^*])$ are linearly independent in $\LPx$.
Assumption \ref{ass:nondegenerate} is a bit tricky but is assumed in several existing work \cite{duval2015exact,flinth2020linear,chizat2019sparse} ensures that the true parameters $(\theta_j^*)_{j=1}^{m^*}$ are uniquely determined. Assumption~\ref{ass:nondegenerate} is also needed to ensure that in a local convergence phase, which we describe in Theorem~\ref{globalconvergence}, $\nu_k$ vanishes rapidly far away from $(\theta_j)_{j=1}^{m^*}$. This assumption can be verified under the same setting as Theorem \ref{main} by utilizing a dual certificate argument.
Assumption \ref{ass:bondedinput} is just fixing the scaling factor and is satisfied under the setting 
$x_i \sim \mathrm{Unif}(\sd)$ (Assumption \ref{sampleasp}).

    
        \begin{theo}
            \label{globalconvergence}
Assume that Assumptions \ref{ass:sparse}, \ref{ass:smooth}--\ref{ass:bondedinput} hold.
            Let $\tau=\emph{Unif}(\sd)$, $\nu_0^+=2/M\sum_{j=1}^{M/2}\delta_{w_{j,0}}, 
            \nu_0^-=2/M\sum_{j=M/2+1}^{M}\delta_{w_{j,0}}$ and $J^*=J(\nu^*)$. 
            Then, for any $0 < \epsilon < 1/2$, there exist constants $\rho,C,C',C_M>0$, $J_0>J^*$, $\kappa_0 > 0$ such that if $\alpha>0$ satisfies 
            \begin{align*}
            \alpha < \min\{(J_0-J^*)^{1+\epsilon/2}/C, & 1/8C_1,
            \\ & 1/10C_2,\rho/C_2,\lambda^2/8C_F^2\}
            \end{align*}
            with $C_1=2\sqrt{n}C_F+\lambda$ and $C_2=2\sqrt{n}C_F$, the width $M$ is sufficiently over-parameterized as $M\geq C_M\exp(\alpha^{-2})/\alpha$, and the initial solution satisfies 
            \begin{align*}
                &\max\{W_\infty(\tau , \nu^+_0),~W_\infty(\tau , \nu^-_0)\} \leq (J_0-J^*)/C, 
            \end{align*}
            then we have the following convergence properties: \\
(1) Global exploration: 
There exists $k_0\geq C' (J_0-J^*)^{-(2+\epsilon)}$ such that for any $k \geq k_0$, it holds that 
$$
J(\nu_k) - J^* \leq J_0 - J^*.
$$
(2) Local convergence: 
For any $k \geq k_0$, it holds that 
            \begin{equation*}
                J(\nu_k)-J^*\leq (J(\nu_0)-J^*)(1-\kappa_0)^{k-k_0}.
            \end{equation*}
            
Therefore, combining these results, we see that  $J(\nu_k)$ converges to $J(\nu^*)$. 
        \end{theo}
The proof can be found in Appendix B. 
This theorem implies that the norm-dependent gradient descent can converge to the global optimal solution in terms of both the measure on parameters and the function value. Its dynamics consists of two phases: (1) the global exploration regime, and (2) the local linear convergence regime. In the first phase, the gradient descent explores the parameter space to roughly capture the location of the optimal parameters. 
In the second phase, the dynamics enters a local region around the optimal parameters where the objective is locally strongly convex.
After entering this phase, the parameters converge to the optimal solution linearly. 
In that sense, $J_0$ represents a threshold that separates the global region and local near strongly convex region. 
During the optimization, the sparse regularization works for eliminating the amplitudes of nodes that are far away from the optimal parameters.
This kind of ``two phase'' dynamics has been pointed out by several authors (e.g., \citet{NIPS2017_a96b65a7,chizat2019sparse}), 
but it has not been shown for the ReLU fully connected neural networks. 

The condition $\max\{W_\infty(\tau , \nu^+_0),~W_\infty(\tau , \nu^-_0)\} \leq (J_0-J^*)/C$ requires that $M$ is sufficiently over-parameterized. It is known that $W_\infty(\tau , \nu^{\pm}_0) = \mathrm{O}_p((\log M)^{1/(d-1)}M^{-1/(d-1)})$ for $d > 3$ \cite{trillos2015rate}. Therefore, it is implicitly assumed that  $M \geq \Omega((J_0 - J^*)^{-(d-1)} \log_+(1/(J_0 - J^*))^{(d-1)})$\footnote{$\log_+(x)$ denotes $\max\{\log(x) ,1\}$.}. 
The condition $M\geq C_M\exp(\alpha^{-2})/\alpha$ also requires the over-parameterization and the right side may be quite large. This condition is only required for the global exploration~((1) in Theorem~\ref{globalconvergence}). The over-parameterization and the norm-dependency of stepsize ensure that $(\theta_{j,k})_{j=1}^M$ do not move far away from initialization until the function value decrease enough. By this property, the gradient descent can ``identify''  an informative subset of parameters $(\theta_{j,k})_{j=1}^M$, which are close to the optimal parameters $(\theta_j^*)_{j=1}^{m*}$. It may be possible to ensure that under the less number of parameters $M$, the gradient descent ``automatically'' reaches around each of the optimal parameters and can accomplish the global exploration. We leave this issue for future work. 
Finally, we mention a remark on a condition on the constant $\rho$ and the regularization parameter $\lambda$ for Theorem~\ref{globalconvergence}. Roughly speaking, $\rho$ represents a diameter of a local smooth region around each optimal parameter $\theta_j^*$. Under Assumptions~\ref{sampleasp} and \ref{teacherasp}, it suffices to take $\rho=\textrm{O}_p(1/nm)$ if $\theta_j^*$ and $\theta_j^\teach$ are sufficiently close for any $j\in\m$ (see Lemma~\ref{rhoevaluation}). It can be shown that this closeness condition between $\theta_j^*$ and $\theta_j^\teach$ holds with high probability by setting $\lambda=\textrm{O}(1/nm^{3/2})$ by Theorem~\ref{main}. 
These estimates are derived from conservative evaluations and could be larger for each concrete realization of $(x_i)_{i=1}^n$. 

In addition to this convergence property in terms of the objective function, we can show convergence in terms of the $L_\infty$-norm. 
\begin{theo}\label{convInW2NormL2Norm}
Under Assumptions \ref{ass:sparse}, \ref{ass:smooth}--\ref{ass:bondedinput},
there exists $C'' > 0$ such that 
 for all $k \geq k_0$, it holds that 
\begin{align*}
 \|f(x;\nu_k) - f(x;\nu^*)\|_\infty
\leq C'' (J(\nu_0) - J^*)(1 - \kappa_0)^{k-k_0}, 
\end{align*}
where $k_0$ and $\kappa_0$ are those introduced in Theorem \ref{globalconvergence}.
\end{theo}
To show this, we prove that the measure representation $\nu_k$ converges to the optimal representation $\nu^*$ in terms of a modified 2-Wasserstein distance. The details can be found in Section B.6.

\paragraph{Near Exact Recovery by Gradient Descent} 
Finally, combining Theorem \ref{main} and  Theorem \ref{convInW2NormL2Norm}, we obtain the following corollary that asserts that the student network converges near the teacher network by the gradient descent method.
To show this, we need to prove that Assumptions \ref{sampleasp} and \ref{teacherasp} implies Assumptions \ref{ass:sparse}, \ref{ass:smooth}--\ref{ass:bondedinput}. The proof can be found in Section B.6.
\begin{Cor}
\label{cor:estimationerror}
Under Assumptions \ref{sampleasp} and \ref{teacherasp}, 
suppose that $n>\thmpoly(m,d,\log1/\delta)$ for $\delta > 0$, then, with probability at least $1-\delta$, 
it holds that the $\LPx$-norm between $f(\cdot;\nu_k)$ and $\ftrue$ can be bounded as 
\begin{align*}
& \|f(\cdot;\nu_k) - \ftrue \|_{\LPx}^2 \\
&\leq 2{C''}^2 (J(\nu_0) - J^*)^2(1 - \kappa_0)^{2(k-k_0)}
+ \textrm{O}(m\lambda^2), 
\end{align*}dependent on the observation $D_n$.
for all $k \geq k_0$ where $k_0$ and $\kappa_0$ are constants introduced in Theorem \ref{globalconvergence} that could depend on the observation $D_n$.
\end{Cor}
        

\section{Numerical Experiments}
\label{sec:NumericalExp}
In this section, we conduct numerical experiments to justify our theoretical results.

\begin{figure}[tp]
\centering
\includegraphics[width=7.0cm]{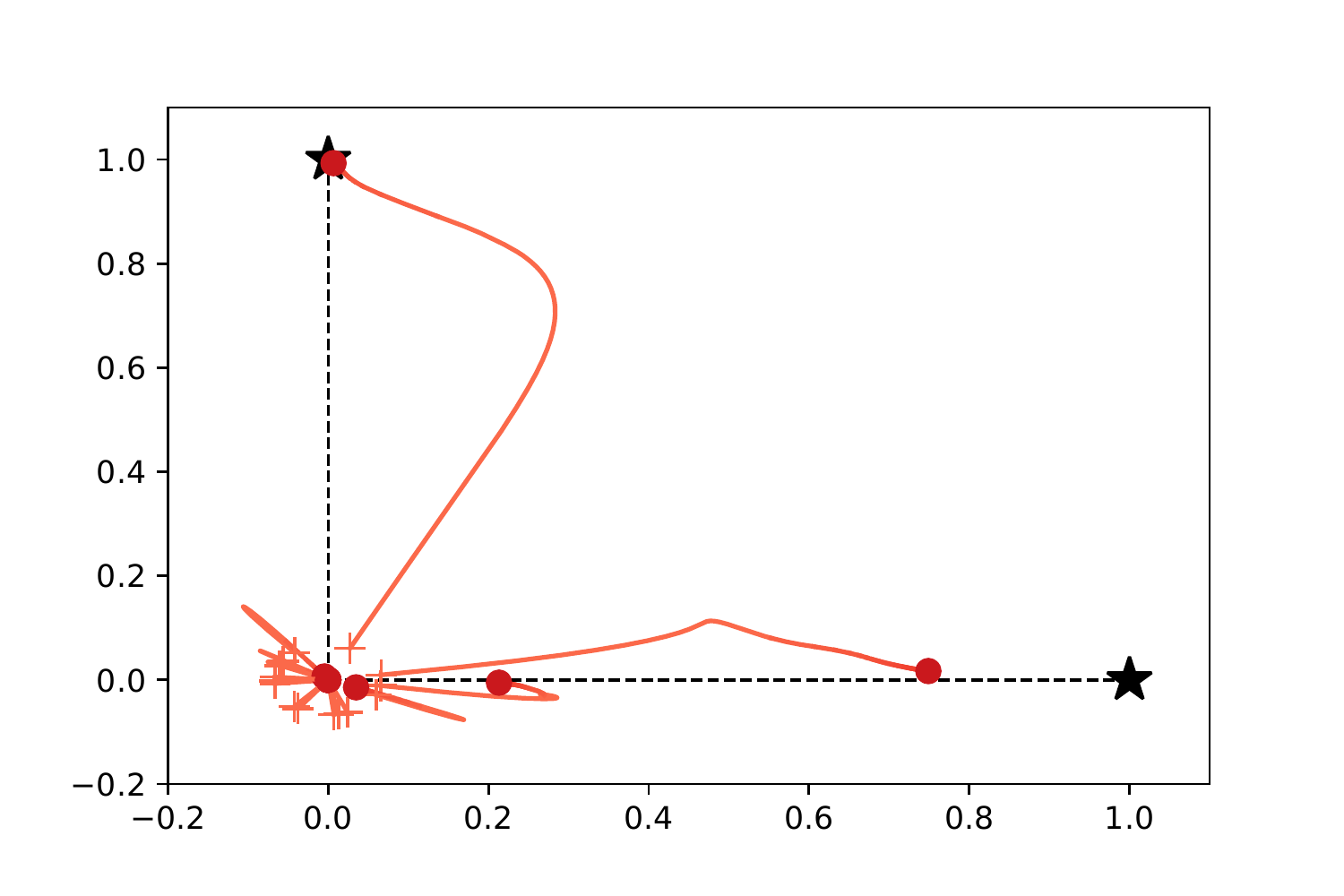}
\vspace{-0.3cm}
\caption{Illustration of the optimization dynamics with $d=2$ and $m=2$.
The true parameters are indicated by $\star$, the initial solution of each node is indicated by an orange $+$, and its final state is indicated by the red $\circ$.
}
\label{fig:ToyExp}
\end{figure}

\paragraph{Illustration in two dimensional space.} 

First, we give an illustrative example in which the dynamics of the student network is depicted in a two dimensional setting $d=2$. 
In this experiment, we employ $m=2$ with $r_1^\teach=r_2^\teach = 1$ and $\theta_1^\teach = (1,0)^\top$, $\theta_2^\teach = (0,1)^\top$, $M = 15$, and $n=100$. 
Figure \ref{fig:ToyExp} shows the optimization trajectory of $(a_{j,k}, w_{j,k})_{j=1}^M$. 
We can see that the nodes with initialization near to a teacher parameter approaches one of the nodes in the teacher network and, on the other hand, the nodes with initialization far away from any teacher node finally vanish.
This behavior is induced by the sparse regularization, that is, the sparse regularization ``selects'' informative nodes and discard non-informative nodes. 
We also see that the selected nodes explore a wide area in the early stage and after that they finally head to the direction of one of the teacher nodes.
This well justifies our theoretical analysis.

\paragraph{Effect of over-parameterization for convergence.}

Next, we investigate how the over-parameterization affects the dynamics. 
In this experiment, we employ $m= 5$ for the teacher width, $d=5$ for the dimensionality and $n=100$ for the sample size.
As for the student network, we compare the dynamics between $M = 5,10,100$.
Figure \ref{fig:ConvExp} depicts the training loss and test loss against the number of  iterations.
Each line corresponds to different setting of $M$.
We can see that a sufficiently over-parameterized network ($M=100$) appropriately estimates the true function
while a narrow network ($M=5$) does not reach the global optimal solution.
We also note that the test loss is almost same as the training loss in the over-parameterized setting while we observe over-fitting for $M=5$ and $M=10$. This means that the solution in the over-parameterized setting $(M=100)$ 
finally converges to the optimal ``sparse'' solution that avoids the over-fitting. 
This is consistent to the findings by the existing studies \cite{safran2018spurious,safran2020effects}.

\paragraph{Comparison of $L_1$ and $L_2$ Regularization}
Inspired by \Eqref{eq:L2L1connection}, we also conduct norm-dependent gradient descent for the $L_2$-regularized problem:
\begin{align}
\!\!\! \! \min_{\Theta \in (\R\times\R^d)^M} 
\frac{1}{2n}\sum_{i=1}^n(y_i-f(x_i;\Theta))^2 \!+ \!\frac{\lambda}{2}\!\sum_{j=1}^M(a_j^2+\|w_j\|^2).
\label{L2parameterproblem}
\end{align}
We give a comparison of the loss evolution between the $L_1$-regularization and $L_2$-regularization in Figure~\ref{fig:compare}. In this experiment, we employ $m=5$ for the teacher width, $d=5$ for the dimensionality, $n=100$ for the sample size and $M=10$ for the student width. We can see that both regularizations show the almost same trajectory of the loss functions. This indicates the usefulness of the practical use of the $L_2$-regularization.




\begin{figure}[tp]
\centering
    \includegraphics[width=7.0cm]{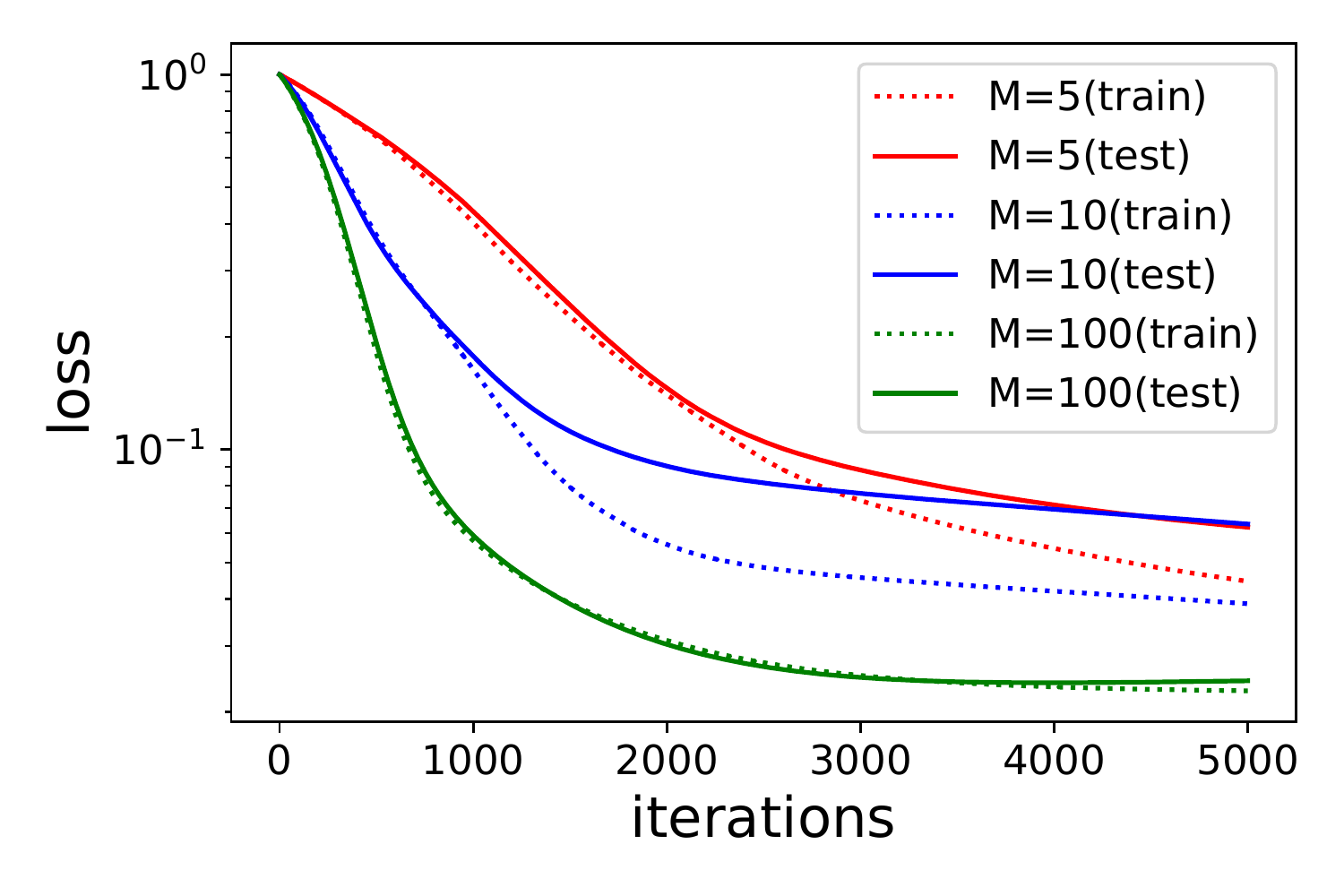}
    \caption{Convergence of the training/test loss for different student width $M=5,10,100$.}
    \label{fig:ConvExp}
\end{figure}

\begin{figure}[tp]
\centering
    \includegraphics[width=7.0cm]{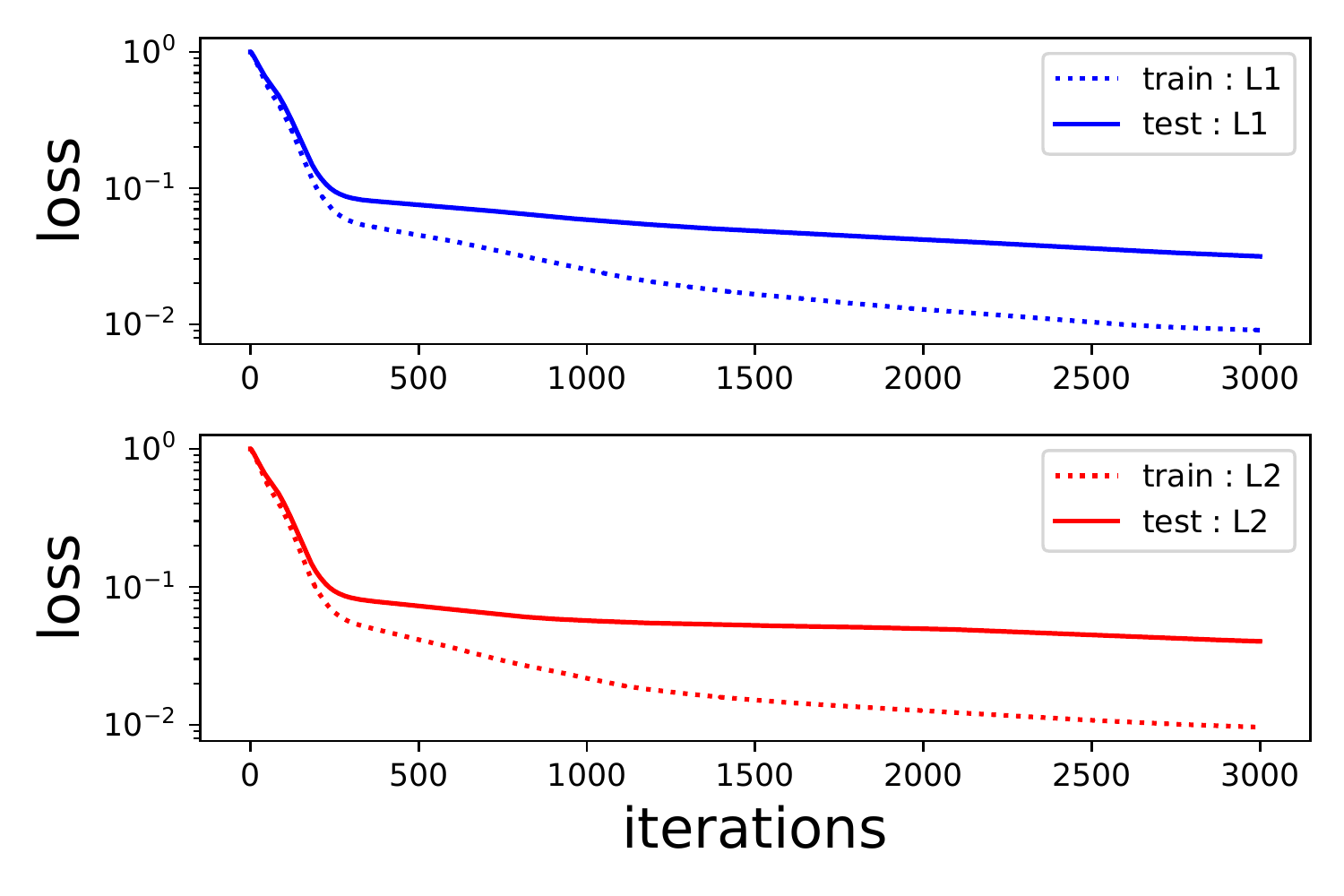}
    \caption{Comparison of $L_1$ and $L_2$-
regularizations.}
    \label{fig:compare}
\end{figure}






\section{Conclusion} 

    In this paper, we have investigated identifiability of the true target function via the gradient descent method for two-layer ReLU neural networks in teacher-student settings.
    We have shown that with the sparse regularization, 
    the global minima can be arbitrarily close to the teacher network. Furthermore, we have proposed 
    a gradient method with norm-dependent step size which is guaranteed to converge to the global minima, and shown that this framework can be applied to the teacher-student setting. 
    The key ingredient in this analysis is the measure representation of the ReLU network. 
    With this perspective, the gradient method can be associated with gradient descent in the measure space. 
    We believe that this analysis gives a new insight into learnability in the teacher-student setting.

\section*{Acknowledgement}
TS was partially supported by JSPS KAKENHI (18H03201, and 20H00576), Japan Digital Design and JST CREST.

\bibliography{sankou,main}
\bibliographystyle{icml2021}

\appendix
\onecolumn

\section{Proof of Theorem~\ref{main} and related topics}
In this section, we give the proof of Theorem~\ref{main} and auxiliary lemmas to prove it.

\subsection{Preliminaries}
    In this section, we give the proof of the main result I (Theorem~\ref{main}). The key tool is the {\it dual certificate} and the {\it NDSC condition} (Definition \ref{NDSC}) which were introduced by \citet{duval2015exact}. We firstly introduce these concepts, and then prove the assertion by using them. 
    \subsubsection{Dual Problem and Optimality Condition}
        As described in \Eqref{measureproblem}, we consider the following optimization problem on the measure space:
        \begin{align*}
        \label{plambda}
            \underset{\nu\in \Radon(\sd)}{\min} \frac{1}{2n}\sum_{i=1}^n\left(y_i-f(x_i;\nu)\right)^2+\lambda\|\nu\|_\TV \tag{$P_\lambda$}. 
        \end{align*}
        By regarding $f$ as a linear operator $f(\cdot):\Radon(\sd) \to \R^n$, $\nu\mapsto  (f(x_1;\nu), \dots, f(x_n;\nu))^\T$, we can define its adjoint operator $f^*:\R^n\to \mathcal{C}(\sd)$ as         
        \begin{align*}
            f^*(p)(\theta) = \frac{1}{n}\sum_{i=1}^np_i\sigma(\langle \theta, x_i\rangle).     
        \end{align*}
        Then, we can obtain the dual problem of ($P_\lambda$)  through the Fenchel duality theorem  \cite{Rock,vari,duval2015exact}: 
        \begin{align*} 
            \label{dlambda}
            \underset{p\in\R^n:\|f^*(p)\|_\infty \leq 1}{\max} \hspace{1em}\frac{1}{n^2}\sum_{i=1}^ny_ip_i -\frac{\lambda}{2n^2}\|p\|^2 \tag{$D_\lambda$}.
        \end{align*}
        This dual problem \eqref{dlambda} can be reformulated as
        \begin{align*} 
            \underset{p\in\R^n:\|f^*(p)\|_\infty \leq 1}{\min} \hspace{1em}\frac{1}{n^2}\left\|p-\frac{1}{\lambda}\left(\begin{array}{c}y_1\\\vdots\\y_n \end{array}\right)\right\|^2 \tag{$\widetilde{D}_\lambda$}.
        \end{align*}
        Note that solutions of this problem are expressed by a projection of $(y_1,\dots,y_n)^\T\in\R^n$ onto a closed convex subset $\{p\in\R^n \mid \|f^*(p)\|_\infty \leq 1\}$ which is uniquely determined by the Hilbert projection theorem. 
        
        By taking the limit of $\lambda \to +0$ in \Eqref{measureproblem}, we obtain the following problem:
        \begin{align*}
        \label{p0}
            \underset{\mu \in \Radon(\sd)}{\min} \|\mu\|_\TV ~~~ \textrm{s.t.} \hspace{1em} f(x_i;\mu) = y_i~~~(\forall i\in [n]). \tag{$P_0$}
        \end{align*}
        The dual problem of this is given by
        \begin{align*} 
        \label{d0}
            \underset{\|f^*(p)\|_\infty\leq 1}{\max}\hspace{1em} \frac{1}{n^2}\sum_{i=1}^ny_ip_i \tag{$D_0$}. 
        \end{align*}
        
        The strong duality between these problems can be characterized by the {\it subdifferential} of the object function. 
        In particular, we require the subdifferential $\partial\|\nu\|_\TV$ of the total variation norm which is expressed by
        \begin{align*}
            \partial\|\nu\|_\TV = \left\{\eta\in \mathcal{C}(\sd)\mid \|\eta\|_\infty\leq 1,\int\eta\dif\mu = \|\nu\|_\TV\right\}.
        \end{align*}
        For $\lambda > 0$, we can show that the strong duality holds between \eqref{plambda} and \eqref{dlambda}, which means that both problems have the same optimal value and any solution $\nu$ of \eqref{plambda} is linked with the unique solution $p$ of \eqref{dlambda} by
        \begin{equation}
            \begin{cases}
            f^*(p)\in \partial\|\nu\|_\TV, \\ p_i = -\frac{1}{\lambda}(f(x_i;\nu)-y_i)~~~~(\forall i\in  [n]).
            \end{cases}
            \label{optimalitycond-lambda}
        \end{equation}
        Conversely, if there exists a pair $(\nu,p)\in\Radon(\sd)\times \R^n$ satisfying \Eqref{optimalitycond-lambda}, then $\nu$ is an optimal solution of \eqref{plambda} and $p$ is the unique solution of \eqref{dlambda}.
        
        Strong duality also holds between \eqref{p0} and \eqref{d0}. If an optimal solution $p^*$ of \eqref{d0} exists, then it is linked to any solution $\nu$ of \eqref{p0} by
        \begin{align}
            \label{optimalitycond:0}
            \begin{cases}
            f^*(p) \in \partial\|\nu\|_\TV, \\ f(x_i;\mu) = y_i~~~(\forall i \in \n), 
            \end{cases}
        \end{align}
        and similarly, if there exists a pair $(\nu,p)\in\Radon(\sd)\times \R^n$ satisfying \Eqref{optimalitycond:0}, then $\nu$ is an optimal solution of \eqref{p0} and $p$ is a solution of \eqref{d0}.
        
        In particular when $\nu$ is written by a sum of Dirac measures as $\nu = \sum_{j=1}^mr_j\delta_{\theta_j}$, $f^*(p) \in \partial\|\nu\|_\TV$ is equivalent to
        \begin{equation}\label{eq:DualitySimple}
            \begin{cases}
                f^*(p)(\theta_j) = \sign(r_j)&(\forall j\in \m),\\
                |f^*(p)(\theta)| \leq 1 &(\forall \theta\in\sd).
            \end{cases}
        \end{equation} 
        We use the next proposition to prove the main theorem. The proof is remained to the latter of this section.
        \begin{prop}
            \label{main1}
            Let $n>\thmpoly(m,d,\log1/\delta)$. Then, with probability at least $1-\delta$, for any $\epsilon>0$, there exists $\lambda=\lambda(\epsilon)$ such that the optimal solution $p_\lambda$ of ($D_\lambda$) satisfies 
            \begin{equation*}
                \begin{cases}
                    f^*(p)(\theta^*_j) = 1&(\forall j\in \m),\\
                    |f^*(p)(\theta)| < 1 & (\forall \theta \in \sd,~\theta \neq \theta_j^*),
                \end{cases}
            \end{equation*}
            where $(\theta_j^*)_{j \in [m]} \subset \sd$ satisfying $\|\theta_j^\teach-\theta_j^*\|<\epsilon~(\forall j\in \m)$.
            Moreover, the global minima of \eqref{plambda} is written by $\nu^*=\sum_{j=1}^mr^*_j\delta_{\theta^*_j}$, where $(r_j^*, \theta_j^*)_{\in [m]}$ satisfies \Eqref{exactrecovery}. 
        \end{prop}
        \begin{Rem}
            Since $f^*(p)$ is piecewise-linear function for any $p$ (following from the same property of ReLU), we know that the global minima of \eqref{plambda} is expressed by a sum of at most O($n^{d+1}$) Dirac measures independently of the sample $(x_i)_{i=1}^n$. This result can be extended to any other 1-homogeneous activation function. The same result is derived in \citet{de2020sparsity} by another approach, our argument above gives another perspective to the characterization of the optimal solution. 
        \end{Rem}

    \subsubsection{Non Degenerate Source Condition}
    
        Unlike \eqref{dlambda}, \eqref{d0} does not always have a unique solution. Then we consider the following concept, which is crucial in this proof. 
        \begin{Def}[minimal norm certificate \cite{duval2015exact}]
            The minimal norm certificate associated with \eqref{plambda} is defined as $f^*(p_0)$, where $p_0$ is the minimum norm solution of ($D_0$) if it exists,  i.e.,
            \begin{equation*}
                p_0 = \arg\min~\{\|p\|~|~ p\ \textrm{is\ a\ solution\ of\ }\ (D_0)\}. 
            \end{equation*}
        \end{Def}  
        
        Minimum norm certificate is linked with the unique solution of \eqref{dlambda} in the following sense:
        \begin{Lem}[\citet{duval2015exact}]
            \label{convergenve}
            Let $p_\lambda$ be the unique solution of \eqref{dlambda}. Then $p_\lambda$ converge to $p_0$ as $\lambda \to +0$, where $p_0$ is the minimal norm solution of $D_0$. 
        \end{Lem}
        Using this Lemma, we can show that under $\lambda \to +0$, the global minima of \eqref{plambda} has its support which is arbitrary close to that of \eqref{p0}.
        Therefore we focus on \eqref{p0} and introduce the following concept. 
        Let $\grad:=(\id-\theta\theta^\T)\nabla$ which represents the derivative on $\sd$. We note that $\grad f(\theta)=0$ means $\nabla f(\theta)=a\theta$ for some $a\in \R$.
        \begin{Def}[NDSC (Non-Degenerate Source Condition) \cite{duval2015exact}]
            \label{NDSC}
            We say that $\nu=\sum_{j=1}^mr_j\delta_{\theta_j}$ satisfies NDSC if the minimal norm certificate $f^*(p_0)$ satisfies the following condition:
            \begin{itemize}
                \item{$f^*(p_0)(\theta_j)= \thmsign(r_j)$~~($\forall j\in \m$),}
                \item{$|f^*(p_0)(\theta)|< 1$ for any $\theta \in \sd$ such that $\theta \neq \theta_j~(\forall j\in \m)$,}
                \item{$\grad^2 f^*(p_0)(\theta_j)$ is invertible for any $j\in \m$.}
            \end{itemize}
        \end{Def}
        Through the second and the third conditions, we can verify that for the unique solution $p_\lambda$ of \eqref{dlambda} $f^*(p_\lambda)(\theta)=1$ holds only in the neighborhood of $\theta=\theta_j$. Hence, the optimal solution of $(P_\lambda)$ has its  support only around $\theta_j$. This yields that $\theta_j^*$ is close to the teacher parameter $\theta_j^\teach$ for sufficiently small $\lambda$. Therefore, we just need to show NDSC for $p_0$, but it is hard to obtain the closed form of $p_0$. To overcome this difficulty, we consider a ``loose'' version of $p_0$, which is called {\it pre-certificate}. 
        \begin{Def}[pre-certificate \cite{duval2015exact}]
            The pre-certificate associated with \eqref{plambda} is defined as $f^*(p^\dagger)$, where
            \begin{equation*}
                p^\dagger=\arg\min~\{\|p\|~|~1\leq \forall j\leq m, f^*(p)(\theta_j)=1, \grad f^*(p)(\theta_j)=0\}. 
            \end{equation*} 
        \end{Def}
        Pre-certificate can be expressed by the minimal norm solution of a linear equation as we see below. 
        If the pre-certificate $f^*(p^\dagger)$ satisfies the conditions in NDSC by replacing $p_0$ with $p^\dagger$, then $p^\dagger$ is an optimal solution of $D_0$ by the optimality condition \eqref{eq:DualitySimple}. 
        Moreover, by noticing that $\|p^\dagger\|\leq \|p_0\|$, if $f^*(p^\dagger)$ achieves 
        the conditions in NDSC, it holds that $p^\dagger = p_0$ and thus the NDSC condition holds for $\nu$, which yields the optimality of $\nu$. Therefore, we show that the pre-certificate $f^*(p^\dagger)$ satisfies the conditions in NDSC instead of directly showing it for the minimal norm certificate $f^*(p_0)$.

    \subsection{NDSC in the teacher-student settings}
        As we discussed in the previous section, we show the following property:
        \begin{prop}[NDSC in the teacher-student setting]
            \label{p0cond}
            Under Assumptions \ref{sampleasp} and \ref{teacherasp}, for $n>\thmpoly(m,d,\log(1/\delta))$ with $\delta > 0$, the pre-certificate associated with the teacher-student settings satisfies the following conditions with probability at least $1-\delta$:
            \begin{itemize}
                \item{$f^*(p^\dagger)(\theta^\teach_j)= 1~~(\forall j\in \m$)}.
                \item{$|f^*(p^\dagger)(\theta)|< 1$ for any $\theta \neq \theta^\teach_j~(\forall j\in \m)$}.
                \item{$f^*(p^\dagger)$ is strictly concave at $\theta=\theta_j^\teach~(\forall j\in \m)$}.
            \end{itemize}
        \end{prop}
        From now on, we show this proposition. At first, we consider how the pre-certificate can be characterized in this setting. When $f^*(p)(\cdot)$ is differentiable at $\theta_j^\teach$ as a function of $\theta$ ($\Leftrightarrow$ there is no $x_i$ that is orthogonal to $\theta_j^\teach$, which holds a.s. for all $j\in [m]$), the extremality condition is given as follows:
        \begin{equation*}
            \begin{cases}
                f^*(p)(\theta^\teach_j)= 1, \\ \nabla f^*(p)(\theta^\teach_j) = \alpha \theta^\teach_j. 
            \end{cases}
        \end{equation*}
        For the ReLU activation, these are equivalent to
        \begin{equation}\label{eq:OptCondThetajTeach}
            \nabla f^*(p_0)(\theta^\teach_j) = \theta^\teach_j, 
        \end{equation}
        since it holds that $\inner{\theta^\teach_j}{\nabla f^*(p_0)(\theta^\teach_j)} = f^*(p_0)(\theta^\teach_j)$.
        By writing down this equation, we get
        \begin{equation*}
            \label{PC}
            \frac{1}{n}\sum_{i=1}^np_ix_i\1\langle \theta^\teach_j, x_i\rangle\geq 0\} = \theta^\teach_j.
        \end{equation*}
        By considering the same equation for all $j\in[m]$ and combining them, we get a linear equation about $p$ as 
        \begin{equation*}
            \frac{1}{n}\left(
            \begin{array}{cccc}
                x_1\1\{\langle \theta^\teach_1, x_1\rangle\geq 0\} & x_2\1\{\langle \theta^\teach_1, x_2\rangle\geq 0\} & \dots & x_n\1\{\langle \theta^\teach_1, x_n\rangle\geq 0\}\\
                x_1\1\{\langle \theta^\teach_2, x_1\rangle\geq 0\} & x_2\1\{\langle \theta^\teach_2, x_2\rangle\geq 0\} & \dots & x_n\1\{\langle \theta^\teach_2, x_n\rangle\geq 0\}\\
                \vdots & \vdots & \ddots & \vdots \\
                x_1\1\{\langle \theta^\teach_m, x_1\rangle\geq 0\} & x_2\1\{\langle \theta^\teach_m, x_2\rangle\geq 0\} & \dots & x_n\1\{\langle \theta^\teach_m, x_n\rangle\geq 0\}\\
            \end{array}
            \right)\left( \begin{array}{c}
                p_1 \\p_2 \\ \vdots \\ p_n
            \end{array}\right)  = \left(
                \begin{array}{c}
                    \theta^\teach_1 \\\theta^\teach_2 \\ \vdots \\ \theta^\teach_m
                \end{array}\right) . 
        \end{equation*}
        By definition, $p^\dagger$ is the minimum norm solution of this equation and represented by
        \begin{equation*}
            p^\dagger=nX_0^\dagger\left(
                \begin{array}{c}
                    \theta^\teach_1 \\\theta^\teach_2 \\ \vdots \\ \theta^\teach_m
                \end{array}\right),
        \end{equation*}
        where
        \begin{equation*}
            X_0 = \left(\begin{array}{cccc}
                x_1\1\{\langle \theta^\teach_1, x_1\rangle\geq 0\} & x_2\1\{\langle \theta^\teach_1, x_2\rangle\geq 0\} & \dots & x_n\1\{\langle \theta^\teach_1, x_n\rangle\geq 0\}\\
                x_1\1\{\langle \theta^\teach_2, x_1\rangle\geq 0\} & x_2\1\{\langle \theta^\teach_2, x_2\rangle\geq 0\} & \dots & x_n\1\{\langle \theta^\teach_2, x_n\rangle\geq 0\}\\
                \vdots & \vdots & \ddots & \vdots \\
                x_1\1\{\langle \theta^\teach_m, x_1\rangle\geq 0\} & x_2\1\{\langle \theta^\teach_m, x_2\rangle\geq 0\} & \dots & x_n\1\{\langle \theta^\teach_m, x_n\rangle\geq 0\}\\
            \end{array}\right) \in\R^{md\times n}
        \end{equation*}
        and $X_0^\dagger$ denotes the Moore-Penrose inverse. Especially when $X_0$ has full row rank (which we verify in the latter w.h.p.), it holds that
        \begin{align}
            \label{pdagger}
            p^\dagger = nX_0^\T(X_0X_0^\T)^{-1}\left(
                \begin{array}{c}
                    \theta^\teach_1 \\\theta^\teach_2 \\ \vdots \\ \theta^\teach_m
                \end{array}\right).
        \end{align} 
        Therefore, we get the closed form of $f^*(p^\dagger)$ as follows. 
        \begin{Lem}
            Suppose that $X_0$ has full row rank.
        	Let $X(\theta)$ be
        	{\rm
            \begin{align*}
                X(\theta) = 
                \begin{pmatrix}
                    x_1\1\{\langle \theta, x_1\rangle\geq 0\}, & x_2\1\{\langle \theta, x_2\rangle\geq 0\}, & \dots~~~, & x_n\1\{\langle \theta, x_n\rangle\geq 0\}
                \end{pmatrix}.
            \end{align*}
            }
           Then the following equality holds:
            \begin{align*}
                f^*(p^\dagger)(\theta) = \frac{1}{n}\theta^\T \biggl(X(\theta)X_0^\T\biggr)\left(\frac{1}{n}X_0X_0^\T\right)^{-1}\left(\begin{array}{c}
                    \theta^\teach_1 \\\theta^\teach_2 \\ \vdots \\ \theta^\teach_m
                \end{array}\right) . 
            \end{align*}
        \end{Lem}
        Each matrix in the expression of $f^*(p^\dagger)$ of the above lemma can be written as follows:
        \begin{align*}
            \frac{1}{n}X_0X_0^\T = \left(\begin{array}{cccc}K_{1, 1}&K_{1, 2}&\dots&K_{1, m}\\
                K_{2,1}&\ddots & &\vdots\\
                \vdots& &\ddots&K_{m-1,m}\\
                K_{m, 1}&\dots&K_{m,m-1}&K_{m, m}\\
            \end{array}\right)\in \R^{dm\times dm},
        \end{align*}
        where 
        \begin{align*}
            K_{j_1, j_2}=\frac{1}{n}\sum_{i=1}^nx_ix_i^\T \1\{\langle \theta^\teach_{j_1},x_i\rangle\geq 0, \langle \theta^\teach_{j_2},x_i\rangle\geq 0\} \in\R^{d\times d},
        \end{align*}
        and 
        \begin{align*}
            \frac{1}{n}X(\theta)X_0^\T = \left(K_1(\theta), K_2(\theta),\dots,K_m(\theta)\right)\in \R^{d\times dm},
        \end{align*}
        where 
        \begin{align*}
            K_j(\theta) = \frac{1}{n}\sum_{i=1}^nx_ix_i^\T \1\{\langle \theta^\teach_j,x_i\rangle\geq 0, \langle\theta,x_i\rangle\geq 0\}\in\R^{d\times d}.
        \end{align*}
        Since these two matrices $\frac{1}{n}X_0X_0^\T$ and $\frac{1}{n}X(\theta)X_0^\T $ depend on the sample observation $(x_i)_{i=1}^n$, it is hard to obtain its close form expression. On the other hand, these are empirical versions of $\E_{D_n}\left[\frac{1}{n}X_0X_0^\T\right]$ and $\E_{D_n}\left[\frac{1}{n}X(\theta)X_0^\T\right]$, respectively. Fortunately, we can write them down by closed forms, and thus we consider the population version $\bar{f}(\theta)$ of $f^*(p^\dagger)(\theta)$ instead, i.e.,
        \begin{equation}
            \label{barfdef}
            \bar{f}(\theta)=\theta^\T\E_{D_n}\left[\frac{1}{n}X(\theta)X_0^\T\right]\E_{D_n}\left[\frac{1}{n}X_0X_0^\T\right]^{-1} \left(\begin{array}{c}
                \theta^\teach_1 \\\theta^\teach_2 \\ \vdots \\ \theta^\teach_m
            \end{array}\right). 
        \end{equation}
        
        \begin{Lem} Under the Assumption~\ref{sampleasp}, the matrices in \eqref{barfdef} are written by follows:
            \label{EXP}
            \begin{align*}
                \E_{D_n}\left[\frac{1}{n}X_0X_0^\T\right] &= \frac{1}{d}
                \left( \begin{array}{cccc}
                \frac{1}{2}\id &\frac{1}{4}\id & \dots & \frac{1}{4}\id\\
                \frac{1}{4}\id &\frac{1}{2}\id & \dots & \frac{1}{4}\id\\
                    \vdots & \vdots & \ddots & \vdots\\
                    \frac{1}{4}\id &\frac{1}{4}\id & \dots & \frac{1}{2}\id 
                    \end{array}  \right) 
                    +\frac{1}{2\pi d}\left(\begin{array}{cccc}0_d & E_{1,2} & \dots &E_{1,m}\\
                        E_{2,1} & \ddots &  &\vdots\\
                        \vdots &  & \ddots & E_{m-1,m} \\
                        E_{m,1} &\dots & E_{m,m-1} &0_d
                        \end{array}
                        \right), 
            \end{align*}
            where $E_{j_1, j_2}$ is the symmetric matrix $\theta^\teach_{j_1}\theta_{j_2}^{\teach\T}+\theta^\teach_{j_2}\theta_{j_1}^{\teach\T} $. 
            \begin{align*}
                \theta^\T\E_{D_n}\left[\frac{1}{n}X(\theta)X_0^\T\right] = \frac{1}{2d}
                \begin{pmatrix}
                \frac{\pi-\phi_1}{\pi}\theta^\T+\frac{\sin\phi_1}{\pi}\theta_1^{\teach\T},&  \frac{\pi-\phi_2}{\pi}\theta^\T+\frac{\sin\phi_2}{\pi}\theta_2^{\teach\T},&\dots~~, & \frac{\pi-\phi_m}{\pi}\theta^\T+\frac{\sin\phi_m}{\pi}\theta_m^{\teach\T} 
                \end{pmatrix}
            \end{align*}
            where $\phi_j= \arccos (\langle\theta, \theta^\teach_j\rangle)$. 
        \end{Lem}
    We know the matrix $\E_{D_n}\left[\frac{1}{n}X_0X_0^\T\right]$ is a positive definite. Indeed, \citet[Theorem~3.2]{safran2020effects} shows that 
    \begin{align*}
    \left( \begin{array}{cccc}
                \frac{1}{2}\id &\frac{1}{4}\id & \dots & \frac{1}{4}\id\\
                \frac{1}{4}\id &\frac{1}{2}\id & \dots & \frac{1}{4}\id\\
                    \vdots & \vdots & \ddots & \vdots\\
                    \frac{1}{4}\id &\frac{1}{4}\id & \dots & \frac{1}{2}\id 
                    \end{array}  \right) 
                    +\frac{1}{2\pi }\left(\begin{array}{cccc}0_d & E_{1,2} & \dots &E_{1,m}\\
                        E_{2,1} & \ddots &  &\vdots\\
                        \vdots &  & \ddots & E_{m-1,m} \\
                        E_{m,1} &\dots & E_{m,m-1} &0_d
                        \end{array}
                        \right) \succeq \left(\frac{1}{4}-\frac{1}{2\pi}\right)I_{md},
    \end{align*}
    which leads to
    \begin{align}
    \label{eq:positivedefinite}
        \E_{D_n}\left[\frac{1}{n}X_0X_0^\T\right]\succeq \frac{1}{d}\left(\frac{1}{4}-\frac{1}{2\pi}\right)I_{md}.
    \end{align}
        By the straight forward calculation, we can check that
        \begin{equation}
            \label{inverse}
            \left(\E_{D_n}\left[\frac{1}{n}X_0X_0^\T\right]\right)^{-1}\left(\begin{array}{c}
                \theta^\teach_1 \\\theta^\teach_2 \\ \vdots \\ \theta^\teach_m
            \end{array}\right) = ad\left(\begin{array}{c}
                \theta^\teach_1 \\\theta^\teach_2 \\ \vdots \\ \theta^\teach_m
            \end{array}\right)
                +bd\left(\begin{array}{c}
                    \sum_{j=1}^m\theta^\teach_j \\\sum_{j=1}^m\theta^\teach_j \\ \vdots \\ \sum_{j=1}^m\theta^\teach_j
                \end{array}\right),
        \end{equation}
        where $a$ and $b$ satisfy
        \begin{equation}
            \label{abequation}
            \begin{cases} \left(1+\frac{m-1}{\pi}\right)a+\left(\frac{m+1}{2}+\frac{m-1}{\pi}\right)b=1, \\ a+\left(\frac{2}{\pi}+m+1\right)b = 0.\end{cases}
        \end{equation}
        
        By solving the equation(\ref{abequation}),
        we get the closed form of $a, b$ as
        \begin{equation*}
            a = \frac{2\pi (\pi m+\pi+2)}{2\pi m^2+(\pi^2-2\pi+4)m+\pi^2+4\pi-4}, b = -\frac{2\pi^2}{2\pi m^2+(\pi^2-2\pi+4)m+\pi^2+4\pi-4}.
        \end{equation*}
        Note that for any integer $m$, it holds that $a>0, b<0$ and $a = -\left(\frac{2}{\pi}+m+1\right)b$. 
            
        By combining Lemma~\ref{EXP} and \Eqref{inverse}, we can write $\bar{f}$ by an explicit form given as 
        \begin{equation*}
            \label{barf}
            \bar{f}(\theta) = (a+b)\sum_{j=1}^m\left(\frac{\pi-\phi_j}{\pi}\cos \phi_j +\frac{\sin\phi_j}{\pi}\right)+b\sum_{j=1}^m\sum_{j'\neq j}\frac{\pi-\phi_j}{\pi}\cos\phi_{j'}. 
        \end{equation*}

        By the construction, it is expected that the function $f^*(p^\dagger)$ converges to $\bar{f}$ with $n\to \infty$.
        Indeed, we can show that 
        \begin{enumerate}
            \item{$\bar{f}$ satisfies the conditions of NDSC. }
            \item{$f^*(p^\dagger)$ converges to $\bar{f}$ while satisfying the conditions in NDSC. }
        \end{enumerate}
        At first, we give the first assertion. 
        \begin{Lem}
            \label{expectedNDSC}
            $\bar{f}(\cdot)$ satisfies
            \begin{equation*}
                \begin{cases} \bar{f}(\theta_j^\teach) = 1 & (\forall j \in \m), \\ 0 < \bar{f}(\theta) < 1 & (\theta\neq \theta_j^\teach~( \forall j \in \m),~\theta \in \sd). \end{cases}
            \end{equation*}
        \end{Lem}
        \begin{proof}
            Let us consider the induction on $m$. If $m = 1$, Lemma holds clearly with
            \begin{equation*}
                \barf(\theta) = \frac{\pi-\phi_1}{\pi}\cos\phi_1+\frac{\sin\phi_1}{\pi}. 
            \end{equation*}
            Below we consider the case $m\geq 2$ and assume that the conclusion holds for $m-1$. At first, if $\theta=\theta^\teach_j$ for a $j \in \m$, it holds that 
            \begin{align*}
                \barf(\theta) = (a+b)\left(\frac{\pi-0}{\pi}\cos 0 + \frac{\sin 0}{\pi}\right) &+(a+b)\sum_{j\neq j}\left(\frac{\pi-\pi/2}{\pi}\cos \pi/2 +\frac{\sin\pi/2}{\pi}\right)\\
                    &+ b\sum_{j'\neq j}\frac{\pi-\pi/2}{\pi}\cos 0\\
                    &= \left(1+\frac{m-1}{\pi}\right)a + \left(\frac{m+1}{2}+\frac{m-1}{\pi}\right)b = 1,
            \end{align*}
            which gives the first equality. To prove the other case, we consider the expansion
            \begin{align*}
                \theta = \sum_{j=1}^mk_j\theta^\teach_j +\left(\textrm{orthogonal\ term\ to}\ \textrm{span}\{\theta^\teach_1, \theta^\teach_2,\dots,\theta^\teach_m\}\right).
            \end{align*}
            Because of the orthogonality of $(\theta_j^\teach)_{j=1}^{m}$, for each $\theta\in \sd$, $(k_j)_{j=1}^m$ are uniquely determined and satisfy the inequality $\sum_{j=1}^mk^2_j\leq 1$. \\
            Then, because the orthogonal term does not affect the value of $\barf$, we can write
            \begin{equation*}
                \barf(k_1, \dots, k_m) = (a+b)\sum_{j=1}^m\left(\frac{\pi-\arccos(k_j)}{\pi}k_j +\frac{\sqrt{1-k_j^2}}{\pi}\right)+b\sum_{j=1}^m\sum_{j'\neq j}\frac{\pi-\arccos(k_j)}{\pi}k_{j'}. 
            \end{equation*}
            Firstly we show $\barf<1$. Suppose that there exists $j\in\m$ such that $k_j=0$. Without loss of generality, we consider the case $k_m=0$. Then we have
            \begin{align*}
                &\barf(k_1, \dots, k_{m-1},0) \\
                &= (a+b)\sum_{j=1}^{m-1}\left(\frac{\pi-\arccos(k_j)}{\pi}k_j +\frac{\sqrt{1-k_j^2}}{\pi}\right)\\
                &+(a+b)\frac{1}{\pi}+b\sum_{j=1}^{m-1}\sum_{\substack{j'\neq j\\1\leq j'\leq m-1}}\frac{\pi-\arccos(k_j)}{\pi}k_{j'}
                +\frac{1}{2}b\sum_{j=1}^{m-1}k_j\\
                &\leq -b\left\{\left(m-1+\frac{2}{\pi}\right)\sum_{j=1}^{m-1}\left(\frac{\pi-\arccos(k_j)}{\pi}k_j+\frac{\sqrt{1-k_j^2}}{\pi}\right)-\sum_{j=1}^{m-1}\sum_{\substack{j'\neq j\\1\leq j'\leq m-1}}\frac{\pi-\arccos(k_j)}{\pi}k_{j'}\right\}\\
                &-b\sum_{j=1}^{m-1}\left(\frac{\pi-\arccos(k_j)}{\pi}k_j+\frac{\sqrt{1-k_j^2}}{\pi}\right)+(a+b)\frac{1}{\pi}+\frac{1}{2}b\sum_{j=1}^{m-1}k_j.
            \end{align*}
            By the induction assumption, the first term takes maximum value only when $k_j=1$ for some $1\leq j\leq m-1$. For the rest term, we have
            \begin{align*}
                -b\sum_{j=1}^{m-1}&\left(\frac{\pi-\arccos(k_j)}{\pi}k_j+\frac{\sqrt{1-k_j^2}}{\pi}\right)+(a+b)\frac{1}{\pi}+\frac{1}{2}b\sum_{j=1}^{m-1}k_j\\
                &=-b\left\{\sum_{j=1}^{m-1}\left(\frac{\pi-\arccos(k_j)}{\pi}k_j+\frac{\sqrt{1-k_j^2}}{\pi}\right)-\left(m+\frac{2}{\pi}\right)\frac{1}{\pi}-\frac{1}{2}\sum_{j=1}^{m-1}k_j\right\}\\
                &\leq -b\left\{\sum_{j=1}^{m-1}\left(\frac{1}{\pi}+\frac{1}{2}k_j+\left(\frac{1}{2}-\frac{1}{\pi}\right)k_j^2\right)-\left(m+\frac{2}{\pi}\right)\frac{1}{\pi}-\frac{1}{2}\sum_{j=1}^{m-1}k_j\right\}\\
                &=-b\left\{-\left(1+\frac{2}{\pi}\right)\frac{1}{\pi}+\left(\frac{1}{2}-\frac{1}{\pi}\right)\right\},
            \end{align*}
            where we use the inequality 
            \begin{align*}
                \frac{\pi-\arccos(k_j)}{\pi}k_j+\frac{\sqrt{1-k_j^2}}{\pi} \leq \frac{1}{\pi}+\frac{1}{2}k+\left(\frac{1}{2}-\frac{1}{\pi}\right)k^2,
            \end{align*}
            which holds with equality if $k=0$ or $1$. Therefore $\barf$ takes maximum value at $\theta=\theta_j^\teach$ for some $1\leq j\leq m$, which gives the conclusion. Then we consider the case where $k_j\neq 0$ for all $j\in\m$.
            We only need to consider a case $\sum_{j=1}^mk_j\geq 0$. Indeed, it holds that
            \begin{align*}
                &\barf(k_1,\dots,k_m) -\barf(-k_1,\dots,-k_m) \\
                &= (a+b)\sum_{j=1}^m\left(\frac{\pi-\arccos(k_j)}{\pi}k_j +\frac{\sqrt{1-k_j^2}}{\pi}-\frac{\arccos(k_j)}{\pi}(-k_j) -\frac{\sqrt{1-k_j^2}}{\pi}\right)\\
                &+b\sum_{j=1}^m\sum_{j'\neq j}\left(\frac{\pi-\arccos(k_j)}{\pi}k_{j'}-\frac{\arccos(k_j)}{\pi}(-k_{j'})\right)\\
                &=(a+b)\sum_{j=1}^mk_j+b\sum_{j=1}^m\sum_{j'\neq j}k_j=-b\left(1+\frac{2}{\pi}\right)\sum_{j=1}^mk_j>0.
            \end{align*}
            Now we consider the conversion $(k_1,\dots,k_{j_1},\dots,k_{j_2},\dots,k_m)\mapsto (k_1,\dots,\sqrt{k_{j_1}^2+k_{j_1}^2},0,\dots,k_m)$ for some $j_1\neq j_2$. For the notation simplicity, we consider that $j_1=1,j_2=2$. Since $\barf$ is permutation-invariant, this does not lose the generality. Let $r:=\sqrt{k_1^2+k_2^2}>0$, then 
            \begin{align*}
                &\barf(k_1,\dots,k_{j_1},\dots,k_{j_2},\dots,k_m)-\barf(k_1,\dots,\sqrt{k_{j_1}^2+k_{j_1}^2},0,\dots,k_m)\\
                &=(a+b)\left(\frac{\pi-\arccos(k_1)}{\pi}k_1 +\frac{\sqrt{1-k_1^2}}{\pi}+\frac{\pi-\arccos(k_2)}{\pi}k_2 +\frac{\sqrt{1-k_2^2}}{\pi}\right.\\
                &\left.-\frac{\pi-\arccos(r)}{\pi}r +\frac{\sqrt{1-r^2}}{\pi}-\frac{1}{\pi}\right)+b\left(\frac{\pi-\arccos(k_1)}{\pi}k_2+\frac{\pi-\arccos(k_2)}{\pi}k_1\right)\\
                &-b\frac{1}{2}k_r+b\left(\sum_{j=3}^m\frac{\pi-\arccos(k_j)}{\pi}\right)(k_1+k_2-r)\\
                &+b\sum_{j=3}^mk_j\frac{-\arccos(k_1)-\arccos(k_2)+\arccos(r)+\pi/2}{\pi}.
            \end{align*}
            By using Lemma~\ref{aaa}, this value is upper bounded by
            \begin{align}
                \begin{split}
                &(a+b)\frac{1}{2}(k_1+k_2-r)\\
                &+b\left(\frac{\pi-\arccos(k_1)}{\pi}k_2+\frac{\pi-\arccos(k_2)}{\pi}k_1\right)
                -b\frac{1}{2}r+b\left(\sum_{j=3}^m\frac{\pi-\arccos(k_j)}{\pi}\right)(k_1+k_2-r)\\
                &+b\sum_{j=3}^mk_j\frac{-\arccos(k_1)-\arccos(k_2)+\arccos(r)+\pi/2}{\pi}
                \end{split} \notag \\
                \begin{split}
                =&-b\left\{\left(m+\frac{2}{\pi}\right)\frac{1}{2}(k_1+k_2-r)\right.\\
                &\left. -\left(\frac{\pi-\arccos(k_1)}{\pi}k_2+\frac{\pi-\arccos(k_2)}{\pi}k_1\right)
                +\frac{1}{2}r-\left(\sum_{j=3}^m\frac{\pi-\arccos(k_j)}{\pi}\right)(k_1+k_2-r)\right.\\
                &\left. -\sum_{j=3}^mk_j\frac{-\arccos(k_1)-\arccos(k_2)+\arccos(r)+\pi/2}{\pi}\right\}.
                \end{split}\nonumber\\
                \begin{split}
                \label{barf0}
                =&-b\left\{\left[m+\frac{2}{\pi}-\left(\sum_{j=3}^m\frac{\pi-\arccos(k_j)}{\pi}\right)-\frac{1}{2}\right]\frac{1}{2}(k_1+k_2-r)\right.\\
                &\left. -\left(\frac{\pi/2-\arccos(k_1)}{\pi}k_2+\frac{\pi/2-\arccos(k_2)}{\pi}k_1\right)
                \right.\\
                &\left. -\sum_{j=3}^mk_j\frac{-\arccos(k_1)-\arccos(k_2)+\arccos(r)+\pi/2}{\pi}\right\}.
                \end{split}                
            \end{align}
            In the latter we ignore the multiplied constant $-b>0$. Then we consider the two cases: (i) we can take $k_1>0$,$k_2<0$ (ii) $k_j>0$ for all $j\in\m$. Note that since $\sum_{j=1}^mk_j\geq 0$, there must be a integer $j\in\m$ such that $k_j>0$. Firstly we consider the case (i). In this case it holds that $\sum_{j=3}^mk_j\geq -1$. 
            
            Firstly, we consider to evaluate the term 
            \begin{align*}
                \frac{1}{2}\left(m+\frac{2}{\pi}\right)-\left(\sum_{j=3}^m\frac{\pi-\arccos(k_j)}{\pi}\right)-\frac{1}{2}
            \end{align*}
            Firstly, we have an inequality for $k\geq 0$,
            \begin{align*}
                \frac{1}{2}+\frac{1}{\pi}k\leq \frac{\pi-\arccos(k)}{\pi}\leq \frac{1}{2}+\frac{1}{\pi}k+\left(\frac{1}{2}-\frac{1}{\pi}\right)k^2
            \end{align*}
            and for for $k\leq 0$,
            \begin{align}
                \label{kpositive}
                \frac{1}{2}+\frac{1}{\pi}k-\left(\frac{1}{2}-\frac{1}{\pi}\right)k^2\leq \frac{\pi-\arccos(k)}{\pi}\leq \frac{1}{2}-\frac{1}{\pi}k,
            \end{align}
            which gives 
            \begin{align*}
                \frac{m-2}{2}+\frac{1}{\pi}\sum_{j=3}^mk_j-\left(\frac{1}{2}-\frac{1}{\pi}\right)\leq \sum_{j=3}^m\frac{\pi-\arccos(k)}{\pi}\leq \frac{m-2}{2}+\frac{1}{\pi}\sum_{j=3}^mk_j+\left(\frac{1}{2}-\frac{1}{\pi}\right).
            \end{align*}
            Then it follows that 
            \begin{align*}
                \frac{2}{\pi}-\frac{1}{\pi}\sum_{j=3}^mk_j\leq \frac{1}{2}\left(m+\frac{2}{\pi}\right)-\left(\sum_{j=3}^m\frac{\pi-\arccos(k_j)}{\pi}\right)-\frac{1}{2}
            \end{align*}
            and 
            \begin{align*}
                \frac{1}{2}\left(m+\frac{2}{\pi}\right)-\left(\sum_{j=3}^m\frac{\pi-\arccos(k_j)}{\pi}\right)-\frac{1}{2}\leq 1-\frac{1}{\pi}\sum_{j=3}^mk_j.
            \end{align*}
            Then by using the inequality $k_1+k_2-r<0$ and  $-\arccos(k_1)-\arccos(k_2)+\arccos(r)+\pi/2<0$, we can get the inequality $\eqref{barf0}\leq 0$.
            For the case (ii), we consider the case $k_1\leq k_2\leq \dots\leq k_m$  Then using Lemma~\ref{ccc}, we can show that
            \begin{align*}
                \begin{split} (\ref{barf0}) \leq&\left\{\frac{1}{2}\left(m+\frac{2}{\pi}\right)-\left(\sum_{j=3}^m\frac{\pi-\arccos(k_j)}{\pi}\right)-\frac{1}{2}\right\}(k_1+k_2-r)\\
                &-\left(\frac{1}{2}-\frac{\arccos(k_1)}{\pi}\right)k_2-\left(\frac{1}{2}-\frac{\arccos(k_2)}{\pi}\right)k_1\end{split}\\
                &\leq 0.
            \end{align*}
            Thus we only need to treat the case where there exists $j$ such that $k_j=0$, which gives the conclusion. 
        \end{proof}
        
        Then we give the proof to the closeness of $f^*(p^\dagger)$ and $\barf$. We show this in two perspective, global and local. In global we show that $\|f^*(p^\dagger)-\bar{f}\|_\infty$ will be small with sufficiently large $n$. 
        We need another explanation to local, where $\theta$ close to $\theta_j$ for $j\in\m$, because only the global discussion, there may be the point where $f^*$ takes the value larger than 1. 
        Firstly we give the global result. 
        \begin{Lem}[Global concentration]
            \label{global}
            Under the Assumption~\ref{sampleasp}, there exists a constant $C$ independent with $m,n,d$,  for any $0<\delta<1$, with probability more than $1-\delta$, it holds
            \begin{equation*}   
                \underset{\theta}{\sup} |f^*(p^\dagger)(\theta)-\barf(\theta)| \leq C\left(md\sqrt{\frac{d\log n}{n}}+md\sqrt{\frac{md\log1/\delta}{n}}\right). 
            \end{equation*}
        \end{Lem}
        
        Remark that $\barf$ is defined as the expected function of $f^*(p^\dagger)$, in the sense that we take expected value of two matrices, $\frac{1}{n}X(\theta)X_0^\T$ and $\frac{1}{n}X_0X_0^\T$. Therefore we aim to concentration inequalities for these respectively.
        \begin{Lem}[The concentration of $\frac{1}{n}X_0X_0^\T$]
            \label{matrixbound}
            We assume the Assumption~\ref{sampleasp} holds. Let $\hat{K}_0 = \frac{1}{n}X_0X_0^\T$, then for any $0<t < 1/12d$, we have
            \begin{align}
                \label{matrixbound1}
                \Pr\left(\left\|\hat{K_0}-\E_{D_n}\left[\hat{K_0}\right]\right\|_\op \geq  t\right)&\leq 2md\exp\left(-\frac{nt^2}{4(2m^2+2mt/3)}\right), \\
                \label{matrixbound2}
                \Pr\left(\left\|\hat{K_0}^{-1}-\E_{D_n}\left[\hat{K_0}\right]^{-1}\right\|_\op \geq t\right)&\leq 2md\exp\left(-\frac{nt^2}{600(300d^4m^2+2d^2mt/3)}\right).
            \end{align}
        \end{Lem}

        \begin{proof}
            Note that $\hat{K}_0$ is decomposed as 
            \begin{equation*}
                \hat{K}_0=\frac{1}{n}\sum_{i=1}^nA_i:=\frac{1}{n}\sum_{i=1}^n\Biggl(x_ix_i^\T\1\{\inner{\theta_{j_1}}{x_i}\geq 0\cap \inner{\theta_{j_2}}{x_i}\geq 0\}\Biggr)_{\substack{(j_1-1)d\leq i\leq j_1d\\(j_2-1)d\leq j\leq j_2d}}.
            \end{equation*}
            For each component $A_i\in\R^{md\times md}$, it holds that $\E_{D_n}[\frac{1}{n}(A_i-\E_{D_n}[A_i])]=0_d$ and
            \begin{equation*}
                \left|\frac{1}{n}(A_i-\E_{D_n}[A_i])\right|\leq \frac{2}{n}m,
            \end{equation*}
            which is obtained by 
            \begin{align*}
                \left\|\frac{1}{n}(A_i-\E_{D_n}[A_i])\right\|_\op&\leq \left\|\frac{1}{n}A_i\right\|_\op+\left\|\frac{1}{n}\E_{D_n}[A_i]\right\|_\op\\
                &\leq\left\|\frac{1}{n}A_i\right\|_\textrm{F}+\frac{1}{n}\E_{D_n}[\|A_i\|_\textrm{F}]\leq \frac{2}{n}m, 
            \end{align*}
            where we use Jensen's inequality,$\|A\|_\op\leq \|A\|_\textrm{F}$ for a matrix $A$ and  $\|x_i\|=1$.
            Therefore we can apply Lemma~\ref{matrixbernstein} with $X_i=\frac{1}{n}(A_i-\E_{D_n}[A_i])$. As a consequence, it holds that for any $t>0$,  
            \begin{equation*}
                \Pr\left(\|\hat{K}_0-\E_{D_n}[\hat{K}_0]\|_\op\geq t\right) \leq 2md\exp\left(-\frac{nt^2}{4(2m^2+2mt/3)}\right). 
            \end{equation*}
             This gives \Eqref{matrixbound1}. Next we consider \Eqref{matrixbound2}. Since it holds that
            $E_{D_n}[\hat{K}_0]\succeq (1/4-1/2\pi)1/dI_{md}$ by \Eqref{eq:positivedefinite}, if $\left\|\hat{K_0}-\E_{D_n}\left[\hat{K_0}\right]\right\|_2 \leq t$ holds, we have $\hat{K}_0\succeq ((1/4-1/2\pi)1/d-t)I_{md}$, which gives
            \begin{align*}
                \left\|\hat{K}_0^{-1}-\E_{D_n}[\hat{K}_0]^{-1}\right\|_\op\leq \left\|\hat{K}_0\right\|_\op^{-1}\left\|\hat{K}_0-\E_{D_n}[\hat{K}_0]\right\|_\op\left\|\E_{D_n}[\hat{K}_0]\right\|_\op^{-1}&\leq \frac{d^2t}{(1/4-1/2\pi)(1/4-1/2\pi-dt)}\\
                &\leq \frac{144d^2t}{1-12dt}, 
            \end{align*}
            where we use $1/6>1/2\pi$. This leads to 
            \begin{equation*}
                \Pr\left( \left\|\hat{K_0}^{-1}-\E_{D_n}\left[\hat{K_0}\right]^{-1}\right\|_\op\geq \frac{144d^2t}{1-12dt}\right)\leq 2md\exp\left(-\frac{nt^2}{4(2m^2+2mt/3)}\right).
            \end{equation*}
            By replacing $t$ by $\frac{t}{144d^2+12dt}\leq \frac{t}{150d^2}$, we get the conclusion. 
        \end{proof}

        \begin{Lem}[The concentration of $\frac{1}{n}X(\theta)X_0^\T$]
            \label{functionbound}
            Let $\hat{K}(\theta) = \frac{1}{n}X(\theta)X_0^\T \in\R^{d\times md}$. Then there exists a constant $C>0$ independent of $n$ and $d$, for any $\delta>0$, with probability at least $1-\delta$, it holds that
            \begin{equation*}
                \underset{\theta}{\sup}\left\|\hat{K}(\theta)-\E_{D_n}\left[\hat{K}(\theta)\right]\right\|_\op\leq C\sqrt{\frac{md\log n}{n}}+ 2\sqrt{\frac{m\log(1/\delta)}{2n}}. 
            \end{equation*}
        \end{Lem}

        \begin{proof}
            At first, we remark that
            \begin{align*}
                \underset{\theta}{\sup}\left\|\hat{K}(\theta)-\E_{D_n}\left[\hat{K}(\theta)\right]\right\|_{\op} &= \underset{\substack{\theta,\|w\|\leq 1,\|v\|\leq 1\\ w\in\R^{md},v\in\R^{d}}}{\sup}\left\langle v,\left(\hat{K}(\theta)-\E_{D_n}\left[\hat{K}(\theta)\right]\right)w\right\rangle\\
                &=\underset{\substack{\theta,\|w\|\leq 1,\|v\|\leq 1\\ w\in\R^{md},v\in\R^{d}}}{\sup}\left\{\left\langle v,\hat{K}(\theta)w\right\rangle-\E_{D_n}\left[\left\langle v,\hat{K}(\theta)w\right\rangle\right]\right\}.
            \end{align*}
            In the above equation, it holds that
            \begin{align*}
                \left\langle v,\hat{K}(\theta)w\right\rangle =\frac{1}{n}\sum_{i=1}^n\inner{v}{x_i}\inner{w}{\tilde{x}_i}\1\{\inner{\theta}{x_i}\geq 0\},
            \end{align*}
            where $\tilde{x}_i:=(x_i^\T\1\{\inner{\theta_1^*}{x_i}\geq 0\},\dots,x_i^\T\1\{\inner{\theta_m^*}{x_i}\geq 0\})\in\R^{md}$.
            Therefore we consider to bound the Rademacher complexity of
            \begin{align*}
                \F:=\left\{r(\theta,v,w)=\frac{1}{n}\sum_{i=1}^n\inner{v}{x_i}\inner{w}{\tilde{x}_i}\1\{\inner{\theta}{x_i}\geq 0\}\mid (\theta,v,w)\in \sd\times \R^d\times\R^{md},\|v\|\leq 1,\|w\|\leq 1\right\}.
            \end{align*} 
            Let $\|f\|_n=\sqrt{\frac{1}{n}f^2(x_i)}$. For two pairs $(\theta,v,w)$, $(\theta',v',w')\in (\sd\times \R^d\times\R^{md})^2$, we have
            \begin{align*}
                \|r(\theta,v,w)-r(\theta',v',w')\|^2_n &= \frac{1}{n}\sum_{i=1}^n\left(\inner{v}{x_i}\inner{w}{\tilde{x}_i}\1\{\inner{\theta}{x_i}\geq 0\}-\inner{v'}{x_i}\inner{w'}{\tilde{x}_i}\1\{\inner{\theta'}{x_i}\geq 0\}\right)^2\\
                &=\frac{1}{n}\sum_{i=1}^n\biggl(\inner{v-v'}{x_i}\inner{w}{\tilde{x}_i}\1\{\inner{\theta}{x_i}\geq 0\} \\
                &\ \ \ +\inner{v'}{x_i}\inner{w-w'}{\tilde{x}_i}\1\{\inner{\theta}{x_i}\geq 0\}\\
                &\ \ \ +\inner{v'}{x_i}\inner{w-w'}{\tilde{x}_i}(\1\{\inner{\theta}{x_i}\geq 0\}-\1\{\inner{\theta'}{x_i}\geq 0\})\biggr)^2\\
                &\leq 3\left(\|v-v'\|^2m+\|w-w'\|^2m+\Big|\1\{\inner{\theta'}{x_i}\geq 0\}-\1\{\inner{\theta}{x_i}\geq 0\}\Big|^2m\right),
            \end{align*}
            where we use $\|x_i\|=1$ and $\|\tilde{x}_i\|\leq \sqrt{m}$ for any $i\in \n$, and $(a+b+c)^2\leq 3(a^2+b^2+c^2)$. By this inequality, we get an upper bound of the covering number of $\F$ as
            \begin{equation*}
                \mathcal{N}(\mathcal{F},\epsilon,\|\cdot\|_n) \leq \begin{cases}
                           C_0n^{d+1}\left(\frac{1}{\epsilon}\sqrt{\frac{1}{m}}\right)^{md+d} &\epsilon < \frac{3}{\sqrt{m}},\\
                           1 &otherwise,
                \end{cases}
            \end{equation*}
            for a some constant $C_0>0$ which is independent of the other parameters. Note that the term $n^{d+1}$ is derived by the covering over $\theta\in \sd$ and the term $\frac{1}{\epsilon}\left(\sqrt{\frac{n}{m}}\right)^{md+d}$ is derived by the covering over $(v,w)\in\R^d\times\R^{md}$. Therefore by using Dudley integral argument, we get an upper bound of Rademacher complexity as 
            \begin{eqnarray}
                R_n(\F)\leq \frac{c}{\sqrt{n}}\int_0^{\frac{3}{\sqrt{m}}} \sqrt{(d+1)\log n+\frac{md+d}{2}\log\frac{1}{m}-(md+d)\log\epsilon+\log C_0}\dif \epsilon\leq C\sqrt{\frac{md\log n}{n}}
            \end{eqnarray}
            for a some constant $C>0$. Finally by using the standard Rademacher complexity bound, we get the conclusion. 
        \end{proof}

        Combining these concentration inequalities, we give the proof to global concentration.

        \begin{proof}[Proof of Lemma~\ref{global}]
            We consider the decomposition as
            \begin{align*}
                \underset{\theta}{\sup} |f^*(p^\dagger)(\theta)-\bar{f}(\theta)| &= \underset{\theta}{\sup}\left|\theta^\T\hat{K}(\theta)\left(\hat{K_0}^{-1}-\E_{D_n}\left[\hat{K_0}\right]^{-1}\right)\left(\begin{array}{c}
                \theta^\teach_1 \\\theta^\teach_2 \\ \vdots \\ \theta^\teach_m
            \end{array}\right)\right.\\ &\left.\ \ \ + \theta^\T\left(\hat{K}(\theta)-\E[\hat{K}(\theta)]\right)\E_{D_n}\left[\hat{K_0}\right]^{-1}\left(\begin{array}{c}
                \theta^\teach_1 \\\theta^\teach_2 \\ \vdots \\ \theta^\teach_m
            \end{array}\right)\right|\\
                &\leq \left\|\hat{K_0}^{-1}-\E_{D_n}\left[\hat{K_0}\right]^{-1}\right\|_{\op}m+ \underset{\theta}{\sup}\left\|\hat{K}(\theta)-\E_{D_n}[\hat{K}(\theta)]\right\|_{\op}\sqrt{m}\left(\frac{1}{4}-\frac{1}{2\pi}\right)^{-1}d,
            \end{align*}
            By using \Eqref{matrixbound2} in Lemma~\ref{matrixbound} with $t = C\sqrt{\frac{md^2\log1/\delta}{n}}$ and Lemma~\ref{functionbound}, we get the conclusion. 
        \end{proof} 
        
        Next we show the local evaluation. More precisely, we show that it holds that $f^*(p^\dagger)(\theta)\leq 1$ around $\theta_j^\teach$ and equality holds only at $\theta=\theta_j^\teach$.

        \begin{Lem}[Local evaluation]
            \label{local}
            Let $C>0$ be a constant and  $n>\thmpoly(m,d,\log1/\delta)$. Then with probability at least $1-\delta$, for all $j\in \m$, if $\thmdist(\theta,\theta_j)<Cn^{-1/4}$ and $\theta\neq\theta^\teach_j$ it holds that $\bar{f}(\theta)<1$ . 
        \end{Lem}
        To prove Lemma~\ref{local}, we focus on the gradient $\nabla f^*(p^\dagger)(\theta)$ and utilize an equality $\inner{\theta}{\nabla f^*(p^\dagger)(\theta)}=f^*(p^\dagger)(\theta) $, which is derived by the 1-homogeneity of ReLU. For simplicity, we $p$ as $p\dagger$ in the latter of this section. At first, we see that $p$ is given by the form 
        \begin{align*}
            p_i = \left(x_i^\T\1\{\langle\theta^\teach_1, x_i\rangle\geq 0\}, \dots, x_i^\T\1\{\langle\theta^\teach_m, x_i\rangle\geq 0\}\right) \left(\frac{1}{n}X_0X_0^\T\right)^{-1}\left(\begin{array}{c}
                \theta_1^\teach \\\theta_2^\teach \\ \vdots \\ \theta_m^\teach \end{array}\right),
        \end{align*}
        if the matrix $\frac{1}{n}X_0X_0^\T$ is invertible, where $p_i$ denotes the $i$'s component of $p$.  As a preliminary, we consider the ``expected value'' of $p$ as
        \begin{equation*}
            \label{qequation*}
            q_i = \left(x_i^\T\1\{\langle\theta_1^\teach, x_i\rangle\geq 0\}, \dots, x_i^\T\1\{\langle\theta^\teach_m, x_i\rangle\geq 0\}\right) \left(\E_{D_n}\left[\frac{1}{n}X_0X_0^\T\right]\right)^{-1}\left(\begin{array}{c}
                \theta^\teach_1 \\\theta^\teach_2 \\ \vdots \\ \theta^\teach_m
            \end{array}\right).
        \end{equation*} 

        \begin{Lem}
            \label{qpositive}
            For any $i\in \n$, it holds that $0\leq q_i\leq (a-bm)d\sqrt{m}$. 
        \end{Lem}

        \begin{proof}
            Using \Eqref{inverse}, $q_i$ is expressed by
            \begin{align*}
                q_i &= \left(x_i^\T\1\{\langle\theta^\teach_1, x_i\rangle\geq 0\}, \dots, x_i^\T\1\{\langle\theta^\teach_m, x_i\rangle\geq 0\}\right)\left(ad\left(\begin{array}{c}
                    \theta^\teach_1 \\\theta^\teach_2 \\ \vdots \\ \theta^\teach_m
                \end{array}\right)
                    +bd\left(\begin{array}{c}
                        \sum_{j=1}^m\theta^\teach_j \\\sum_{j=1}^m\theta^\teach_j \\ \dots \\ \sum_{j=1}^m\theta^\teach_j
                    \end{array}\right)\right)\\
                    &= ad\sum_{j=1}^m\sigma(\langle\theta^\teach_j, x_i\rangle)+bd\sum_{j_1=1}^m\sum_{j_2=1}^m\langle\theta^\teach_{j_1}, x_i\rangle\1\{\langle\theta^\teach_{j_2}, x_i\rangle\geq 0\}.
            \end{align*}
            Then the upper bound is obtained clearly. For the lower bound, we have that
            \begin{align*}
                q_i &\geq ad\sum_{j=1}^m\sigma(\langle\theta_j, x_i\rangle)+bd\sum_{j_1=1}^m\sum_{j_2=1}^m\sigma(\langle\theta_{j_1}, x_i\rangle)\\
                &= d\sum_{j=1}^m\biggl(a\sigma(\langle\theta_j, x_i\rangle) + mb\sigma(\langle\theta_j, x_i\rangle)\biggr)\\
                &= -\left(\frac{2}{\pi}+1\right)b\sum_{j=1}^m\sigma(\langle\theta_j, x_i\rangle)\geq 0. 
            \end{align*}
            In the last inequality we use fact that $b <0$ and $a+mb=-\left(\frac{2}{\pi}+1\right)b$. 
        \end{proof} 
        
    Next we give a bound on the distance between $p_i$ and $q_i$, which can be evaluated trough the concentration inequality of $\frac{1}{n}X_0X_0^\T$ (Lemma~\ref{matrixbound}).

    \begin{Lem}
        \label{qpbound}
        On the distance between $p_i$ and $q_i$, we have the following inequality:
        \begin{equation*}
            \Pr\left(\underset{i\in\n}{\max}\{|p_i-q_i|\}\geq t\right)\leq  2md\exp\left(-\frac{nt^2}{600(300d^4m^2+2d^2mt/3)}\right).
        \end{equation*}
    \end{Lem}

    \begin{proof}
        By Lemma~\ref{matrixbound2}, it holds that 
        \begin{equation*}
            \Pr\left(\left\|\hat{K_0}^{-1}-\E_{D_n}\left[\hat{K_0}\right]^{-1}\right\|_\op \geq t\right)\leq 2md\exp\left(-\frac{nt^2}{600(300d^4m^3+2d^2m^2t/3)}\right). 
        \end{equation*}
        Since $\left\|\left(\begin{array}{c}
            \theta_1 \\\theta_2 \\ \vdots \\ \theta_m
        \end{array}\right)\right\|=\sqrt{m}$ and $\left\|\left(\begin{array}{c} x_i\1\{\langle\theta^\teach_1, x_i\rangle\geq 0\}\\ x_i\1\{\langle\theta^\teach_2, x_i\rangle\geq 0\}\\ \vdots \\ x_i\1\{\langle\theta^\teach_m, x_i\rangle\geq 0\}\end{array}\right)\right\|\leq \sqrt{m}$, we get
        \begin{equation*}
            \Pr\left(\underset{i\in\n}{\max}\{|p_i-q_i|\}\geq mt\right)\leq  2md\exp\left(-\frac{nt^2}{600(300d^4m^2+2d^2mt/3)}\right).
        \end{equation*} 
        By replacing $t$ by $t/m$, we get the conclusion.
    \end{proof}
    
    We prepare another Lemma, which is needed to evaluate variation of the gradient $\nabla f^*(p)(\theta)$ around each $\theta_j^\teach$.
    
    \begin{Lem}
        \label{binomial}
        Assume the Assumption~\ref{sampleasp} holds. For any $\theta\in\sd$ and $\tau>0$, let $A_\tau:=\{x_i\mid\left|{\thmdist}(x, \theta)-\frac{\pi}{2}\right|<\tau\}$, then for $0\leq t\leq 1$, it holds that
        \begin{equation*}
            \Pr\left(\frac{\#A_\tau}{n}\geq t+\frac{d\tau}{\sqrt{\pi}} \right)\leq \exp\left(-\frac{nt^2}{2(d\tau/\sqrt{\pi}+t/3)}\right). 
        \end{equation*}
    \end{Lem}

    \begin{proof}
        For any given $\theta\in\sd$, \citet[Lemma~12]{cai2013distributions} shows that for each $i$, $\dist{\it (\theta_j^\teach,x_i)}$ is distributed on $[0,\pi]$ with density
    \begin{align*}
        h(\varphi) = \frac{1}{\sqrt{\pi}}\frac{\Gamma\left(\frac{d}{2}\right)}{\Gamma\left(\frac{d-1}{2}\right)}(\sin\varphi)^{d-2}.
    \end{align*}
    This has a maximum value $h(\pi/2)=\frac{1}{\sqrt{\pi}}\frac{\Gamma\left(\frac{d}{2}\right)}{\Gamma\left(\frac{d-1}{2}\right)}$. This leads to
    \begin{align*}
        \Pr\left(\left|\dist{\it (\theta_j^\teach,x_i)}-\frac{\pi}{2}\right|\leq \tau\right)\leq \frac{1}{\sqrt{\pi}}\frac{\Gamma\left(\frac{d}{2}\right)}{\Gamma\left(\frac{d-1}{2}\right)}2\tau\leq \frac{d\tau}{\sqrt{\pi}}
    \end{align*}
    for any $t\in [0,\frac{\pi}{2}]$.
        This gives that $\#A_\tau\sim B(n, \textsf{prob})$ with $\textsf{prob}\leq \frac{d\tau}{\sqrt{\pi}}$, where $B(\cdot, \cdot)$ denotes the Binomial distribution. Then by using Bernstein's inequality, we get the conclusion. 
    \end{proof}
    Combining Lemma~\ref{qpositive}--\ref{binomial}, we give a proof of Lemma~\ref{local}.

    \begin{proof}[Proof of Lemma~\ref{local}]
        At first, we have
        \begin{equation*}
            \theta^\teach_j = \frac{1}{n}\sum_{i:\langle\theta_j, x_i\rangle\geq 0}p_ix_i.
        \end{equation*}
        At the place $\theta$, let $I_1$(resp. $I_2$) be the subset of $\n$ such that $\langle\theta^\teach_j, x_i\rangle\geq 0$ and $\langle\theta, x_i\rangle< 0$ (resp. $\langle\theta^\teach_j, x_i\rangle<0$ and $\langle\theta, x_i\rangle\geq0$), the gradient at $\theta$ is expressed as 
        \begin{align*}
            \left\langle\theta, \nabla f^*(p)\right\rangle &=\left\langle\theta, \theta_j -\frac{1}{n}\sum_{I_1}p_ix_i +\frac{1}{n}\sum_{I_2}p_ix_i\right\rangle\nonumber\\
            &=\left\langle\theta, \theta_j\right\rangle-\frac{1}{n}\sum_{I_1}p_i\left\langle\theta, x_i\right\rangle+\frac{1}{n}\sum_{I_2}p_i\left\langle\theta, x_i\right\rangle  \nonumber\\
            &=
            \begin{aligned}[t]\left\langle\theta, \theta_j\right\rangle-&\frac{1}{n}\sum_{I_1}q_i\left\langle\theta, x_i\right\rangle+\frac{1}{n}\sum_{I_2}q_i\left\langle\theta, x_i\right\rangle\\
                &-\frac{1}{n}\sum_{I_1}(p_i-q_i)\left\langle\theta, x_i\right\rangle+\frac{1}{n}\sum_{I_2}(p_i-q_i)\inner{\theta}{x_i}
            \end{aligned}.
        \end{align*}
        Let $\left\langle\theta, \theta^\teach_j\right\rangle:= 1-T$ for $T>0$, then we have $\|\theta-\theta^\teach_j\|\leq T$ and it holds that
        \begin{itemize}
            \item{For $i \in I_1$, $-T\leq \left\langle\theta, x_i\right\rangle< 0$}, 
            \item{For $i\in I_2$, $0\leq \left\langle\theta, x_i\right\rangle\leq T$}. 
        \end{itemize}
        Then by using the fact $q_i\geq 0$ and Lemma~\ref{qpbound}, we get
        \begin{eqnarray}           
            \left\langle\theta, \nabla f^*(p)(\theta)\right \rangle&\leq 1-T \begin{aligned}[t] &+T\frac{\max q_i}{n}\#\{I_1\cup I_2\}\\
            &+\left|\frac{1}{n}\sum_{I_1}(p_i-q_i)\left\langle\theta, x_i\right\rangle\right|+\left|\frac{1}{n}\sum_{I_2}(p_i-q_i)\left\langle\theta, x_i\right\rangle\right|\\
            \end{aligned}\\
            &\leq 1-T\left\{1-\left(\frac{\max q_i}{n}+\frac{t}{n}\right)\#\{I_1\cup I_2\}\right\}
        \end{eqnarray}
        with probability at least RHS of Lemma~\ref{qpbound}. It remains to show that $1-\left(\frac{\max q_i}{n}+\frac{t}{n}\right)\#\{I_1\cup I_2\}>0$ while $T\leq C/n^{1/4}$. By the definition, $\#\{I_1\cup I_2\}$ is upper bounded by $\#A_T$ associated with $\theta_j^\teach$. Thus we can show that $\left(\frac{\max q_i}{n}+\frac{t}{n}\right)\#\{I_1\cup I_2\}=\textrm{O}(T)$ w.h.p. under $n>poly(m,d,\log1/\delta)$ and this gives the conclusion.
    \end{proof}
    
\subsection{Proof of Theorem~\ref{main}}
    Combining the discussion in the previous section, we give the proof of Theorem~\ref{main}. At first, we show that NDSC holds in the teacher student setting w.h.p. (Proposition~\ref{p0cond}).
    
    \begin{proof}[proof of Proposition~\ref{p0cond}]
        At first, by \Eqref{matrixbound1} in Lemma~\ref{matrixbound}, $\frac{1}{n}X_0X_0^\T$ is positive definite with probability at least $1-Cm\sqrt{\log (md)/n}$ for a constant $C>0$. Suppose that this holds, $p^\dagger$ exists and is written by \Eqref{pdagger}. In this case, $f^*(p^\dagger)(\theta_j^\teach)=1$ holds clearly by the construction for any $j\in [m]$. 
        
        Next we show the concavity around $\theta_j^\teach$ for each $j$. Note that $\nabla f^*(p^\dagger)(\theta_j^\teach)=\theta_j^*$.
        Therefore it holds that $\nabla f^*(p^\dagger)(\theta)=\theta_j^*$ for $\theta$ sufficiently close to $\theta_j^\teach$ to satisfy $\sign(\inner{\theta}{x_i})=\sign(\inner{\theta_j^\teach}{x_i})$ for all $i\in [n]$, since
        \begin{align*}
        \nabla f^*(p^\dagger)(\theta)=\frac{1}{n}\sum_{i:\inner{\theta}{x_i}\geq 0}p^\dagger x_i=\frac{1}{n}\sum_{i:\inner{\theta_j^\teach}{x_i}\geq 0}p_i^\dagger x_i= \nabla f^*(p_i^\dagger)(\theta_j^\teach).
        \end{align*}
        Hence, it holds that $\inner{\theta}{\nabla f^*(p^\dagger)(\theta)}=\inner{\theta}{\theta_j^\teach}$ around $\theta_j^\teach$ and this is clearly a concave. Finally we show the second condition, i.e., $|f^*(p^\dagger)(\theta)|<1$ for any $\theta\neq\theta_j^\teach$($\forall j\in [m]$). By Lemma~\ref{global}, we know for sufficiently large $n$, if there exists a point where $|f^*(p^\dagger)(\theta)|\geq 1$, it must be around $\theta_j^\teach$. Moreover, by Lemma~\ref{local}, we can ensure that there must be no point other than $\theta_j^\teach$ until the function value decreases to $1-\textrm{O}(T)=1-\textrm{O}(n^{-1/4})$. Combining these results, we get the conclusion.
    \end{proof}
    
    \begin{proof}[proof of Proposition~\ref{main1}]
        Assume that the NDSC holds, which is ensured by Proposition~\ref{p0cond}. Let $p_\lambda$ be the unique solution of $\eqref{plambda}$. By Lemma~\ref{convergenve}, for sufficiently small $\lambda>0$, it holds that $f^*(p_\lambda)(\theta)$ only takes value $1$ at $ \theta = \theta_j^*$ ($j\in[m]$) which satisfies $\sign(\inner{\theta_j^*}{x_i})=\sign(\inner{\theta_j^\teach}{x_i})$ for all $i\in [n]$. Moreover, since $\|\theta^\teach_j-\theta_j^*\|$ can be arbitrary small as $\lambda\to +0$, we get the first conclusion. To complete the proof, we discuss $(r^*_j)_{j=1}^m$.
        
        Firstly, we show that $(r^*_j)_{j=1}^m$ are uniquely determined. By the optimality condition \eqref{optimalitycond-lambda}, it holds that
        \begin{align*}
            (p_\lambda)_i = -\frac{1}{\lambda}(f(x_i;\nu^*)-y_i)\ \ \ (\forall i\in [n]).
        \end{align*}
        Remind that $p_\lambda$ is uniquely determined. Let $\nu^*=\sum_{j=1}^mr_j^*\delta_{\theta_j^*}$ and rearranging this equation, we have
        \begin{align}
            \label{requation}
            \sum_{j=1}^mr_j^*\sigma(\inner{\theta_j^*}{x_i}) = -\lambda (p_\lambda)_i+y_i\ \ \ (\forall i\in [n]).
        \end{align}
        This can be seen as a linear equation about $r^*:=(r_1^*,\dots,r_m^*)^\T$, that is,
        \begin{align*}
            Ar^* = -\lambda p_\lambda+y,
        \end{align*}
        where $y=(y_1,\dots,y_n)\in\R^n$ and $A=(\sigma(\inner{\theta_j^*}{x_i}))_{i,j}\in\R^{n\times m}$. We can show that $n\geq poly(m,d,\log1/\delta)$ and sufficiently small $\lambda$, $A$ has column full rank with probability at least $1-\delta$. Indeed, $A$ can be decomposed as 
        \begin{align*}
            A = X_0^\T\left(\begin{array}{cccc}
            \theta_1^* & 0 & \cdots & 0 \\
            0 & \theta_2^* & \cdots  & 0 \\
            \vdots & \vdots & \ddots & \vdots \\
            0 & \cdots & 0 & \theta_m^*
            \end{array}
            \right)
            ,
        \end{align*}
        where the second matrix has column full rank and we have already shown that $X_0$ has column full rank w.h.p.. Consequently, we can show the uniqueness of $(r^*_j)_{j=1}^m$. Moreover, taking the limit $\lambda\to +0$ in \Eqref{requation},
        we have $\sum_{j=1}^mr_j^*\sigma(\inner{\theta_j^*}{x_i})\to y_i=\sum_{j=1}^mr_j^\teach \sigma(\inner{\theta_j^\teach}{x_i})$ $(i\in [n])$. Then, by using $\theta_j^*\to \theta_j^\teach$ and linear independent of  $(\sigma(\inner{\theta_j^\teach}{\cdot}))_{j=1}^m$ which holds w.h.p., we get $(r_j^*)_{j=1}^m\to (r_j^\teach)_{j=1}^m$ as $\lambda\to +0$. This gives the conclusion.
    \end{proof}
    
To complete the proof, we need to evaluate how close $(r_j^*,\theta_j^*)_{j=1}^m$ and teacher parameters $(r_j^\teach,\theta_j^\teach)_{j=1}^m$ will be. This quantitative evaluation is obtained by using the form $\nu^*=\sum_{j=1}^mr_j^*\delta_{\theta_j^*}$ and strong convexity of the empirical risk, as we see in the proof below.

    \begin{proof}[proof of Theorem~\ref{main}]
        For sufficiently large $n$ and small $\lambda>0$, we can assume that the optimal solution is written by a form $\nu^*=\sum_{j=1}^mr_j^*\delta_{\theta_j^*}$, as we have shown in Proposition~\ref{main1}. 
        Then, by the optimality of $\nu^*$, it holds that 
        $$
        \frac{1}{2n} \sum_{i=1}^n(f(x_i; \nu^*) - f(x_i; \nu^\teach))^2 
        + \lambda \sum_{j=1}^{m} |r_j^*|
        \leq \frac{1}{2n}\sum_{i=1}^n (f(x_i; \nu^\teach) - f(x_i; \nu^\teach))^2 
        + \lambda \sum_{j=1}^{m} |r_j^\teach|.
        $$
        This yields that
        \begin{align*}
        \frac{1}{2n} \sum_{i=1}^n(f(x_i; \nu^*) - f(x_i; \nu^\teach))^2
        & \leq \lambda \sum_{j=1}^{m} (|r_j^\teach| - |r_j^*|) 
        \leq \lambda \sum_{j=1}^{m} |r_j^\teach - r_j^*|.
        \end{align*}
        To get the lower bound on the left side, we evaluate its expected value over $(x_i)_{i=1}^m$, i.e., $\frac{1}{2}\|f(\cdot;\nu^*) - f(\cdot;\nu^\teach)\|^2_{\LPx}$. Now we have
        \begin{align*}
        \|f(\cdot;\nu^*) - f(\cdot;\nu^\teach)\|^2_{\LPx}=&
        \sum_{j=1}^m\|r_j^\teach \sigma(\inner{\theta_j^\teach}{\cdot})- r_j^*\sigma(\inner{\theta_j^*}{\cdot})\|^2_{\LPx}\\
        +&\sum_{j\neq j'}\Bigl\langle r_j^\teach\sigma(\inner{\theta_j^\teach}{\cdot})- r_j^*\sigma(\inner{\theta_j^*}{\cdot}),
        r_{j'}^\teach \sigma(\inner{\theta_{j'}^\teach}{\cdot})- r_{j'}^*\sigma(\inner{\theta_{j'}^*}{\cdot})\Bigr\rangle _{\LPx}.
        \end{align*}
        Then we evaluate the each term.
        For $j \in \m$, let $\phi = \dist(\theta_j^\teach,\theta_j^*)$.
        Then, we obtain 
        \begin{align*}
        & \|r_j^\teach \sigma(\inner{\theta_j^\teach}{\cdot})- r_j^*\sigma(\inner{\theta_j^*}{\cdot})\|^2_{\LPx} \\
        = & {r_j^\teach}^2 \E_X[\sigma(\inner{\theta_j^\teach}{X})^2]
        - 2 r_j^\teach r_j^*\E_X[\sigma(\inner{\theta_j^\teach}{X})\sigma(\inner{\theta_j^*}{X})]
        + {r_j^*}^2 \E_X[\sigma(\inner{\theta_j^*}{X})^2] \\
        = &  \frac{1}{2d} {r_j^\teach}^2  + \frac{1}{2d} {r_j^*}^2 
        - 2 r_j^\teach r_j^* \frac{1}{2d} \left(\frac{\pi - \phi}{\pi} \inner{\theta_j^*}{\theta_j^\teach} + \frac{\sin \phi}{\pi}\right) \\
        = &  \frac{1}{2d}
        \left[ {r_j^\teach}^2  - 2 r_j^\teach r_j^* + {r_j^*}^2 
        + 2 r_j^\teach r_j^* \left(1 - \frac{\pi - \phi}{\pi} \inner{\theta_j^*}{\theta_j^\teach} - \frac{\sin \phi}{\pi}\right)\right]
        \\
        = &  \frac{1}{2d}
        \left[ (r_j^\teach - r_j^*)^2 
        + 2 r_j^\teach r_j^* \left(1 - \inner{\theta_j^*}{\theta_j^\teach} 
        - \frac{\phi}{\pi} (1- \inner{\theta_j^*}{\theta_j^\teach}) +\frac{\phi - \sin \phi}{\pi}\right)\right] \\
        = & 
        \frac{1}{2d}
        \left[ (r_j^\teach - r_j^*)^2 
        + 2 r_j^\teach r_j^* \left(1 - \frac{\phi}{\pi} \right) (1- \inner{\theta_j^*}{\theta_j^\teach}) 
        \right] + O(\phi^3),
        \end{align*}
        where $\E_X$ denotes the expectation over $P_{\X}$.
        Since $1- \inner{\theta_j^*}{\theta_j^\teach}= \Theta(\phi^2)$, 
        the higher order term $O(\phi^3)$ is negligible for sufficiently small $\epsilon>0$, which is the same as Proposition~\ref{main1}. 
        
        For $j \neq j'$ and $\mathrm{x}$,$\mathrm{y}\in\{\teach,*\}$, let $\phi_{j,j'}^{\mathrm{xy}} = \dist(\theta_j^\mathrm{x},\theta_{j'}^{\mathrm{y}})$.
        Then, we have that 
        \begin{align*}
        & \Bigl\langle r_j^\teach\sigma(\inner{\theta_j^\teach}{\cdot})- r_j^*\sigma(\inner{\theta_j^*}{\cdot}),
        r_{j'}^\teach \sigma(\inner{\theta_{j'}^\teach}{\cdot})- r_{j'}^*\sigma(\inner{\theta_{j'}^*}{\cdot})\Bigr\rangle _{\LPx} \\
        =& 
        \frac{1}{2d}\Bigg[ 
        r_j^\teach r_{j'}^\teach \left( \frac{\pi - \phi_{j,j'}^{\teach \teach}}{\pi} \inner{\theta_j^\teach}{\theta_{j'}^\teach} 
        + \frac{\sin \phi_{j,j'}^{\teach \teach}}{\pi}\right)  
        -
        r_j^\teach r_{j'}^* \left( \frac{\pi - \phi_{j,j'}^{\teach *}}{\pi} \inner{\theta_j^\teach}{\theta_{j'}^*}
        + \frac{\sin \phi_{j,j'}^{\teach *}}{\pi}
        \right)  \\ 
        & ~~~~~~- 
        r_j^* r_{j'}^\teach \left( \frac{\pi - \phi_{j,j'}^{* \teach}}{\pi} \inner{\theta_j^*}{\theta_{j'}^\teach}
        + \frac{\sin \phi_{j,j'}^{* \teach}}{\pi}
        \right)  
        +
        r_j^* r_{j'}^* \left( \frac{\pi - \phi_{j,j'}^{**}}{\pi} \inner{\theta_j^*}{\theta_{j'}^*}
        + \frac{\sin \phi_{j,j'}^{**}}{\pi}
        \right)  
        \Bigg] \\
        = &
        \frac{1}{2\pi d}\Bigg\{(r_j^\teach - r_j^*)(r_{j'}^\teach - r_{j'}^*)  \\
        &+ r_j^\teach r_{j'}^\teach \left[ (\pi - \phi_{j,j'}^{\teach \teach}) \inner{\theta_j^\teach}{\theta_{j'}^\teach} 
        + \sin \phi_{j,j'}^{\teach \teach}-1\right]  
        -
        r_j^\teach r_{j'}^* \left[ (\pi - \phi_{j,j'}^{\teach *}) \inner{\theta_j^\teach}{\theta_{j'}^*}
        +\sin \phi_{j,j'}^{\teach *} - 1
        \right]  \\ 
        & ~~~~~~- 
        r_j^* r_{j'}^\teach \left[(\pi - \phi_{j,j'}^{* \teach})  \inner{\theta_j^*}{\theta_{j'}^\teach}
        + \sin \phi_{j,j'}^{* \teach}-1
        \right]
        +
        r_j^* r_{j'}^* \left[ (\pi - \phi_{j,j'}^{**}) \inner{\theta_j^*}{\theta_{j'}^*}
        +\sin \phi_{j,j'}^{**}-1
        \right]
        \Bigg\}.
        \end{align*}
        Here, we note that $\inner{\theta_j^\teach}{\theta_{j'}^*} = \cos \phi_{j,j'}^{\teach *} = - (\phi_{j,j'}^{\teach *} - \pi/2) +\textrm{O}((\phi_{j,j'}^{\teach *} - \pi/2)^3)$ 
        and $\sin \phi_{j,j'}^{\teach *} = 1 - (\phi_{j,j'}^{\teach *} - \pi/2)^2 +\textrm{O}((\phi_{j,j'}^{\teach *} - \pi/2)^4)$.
        Therefore, it holds that 
        \begin{align*}
        & r_j^\teach r_{j'}^* \left[ (\pi - \phi_{j,j'}^{\teach *}) \inner{\theta_j^\teach}{\theta_{j'}^*}
        +\sin \phi_{j,j'}^{\teach *} - 1 \right] \\
        &= 
        r_j^\teach r_{j'}^* \left[ \frac{\pi}{2}\inner{\theta_j^\teach}{\theta_{j'}^*}
        + (\pi/2 - \phi_{j,j'}^{\teach *}) \inner{\theta_j^\teach}{\theta_{j'}^*}
        +\sin \phi_{j,j'}^{\teach *} - 1 \right] \\
        & = \frac{\pi}{2} r_j^\teach r_{j'}^* \inner{\theta_j^\teach}{\theta_{j'}^*}
        +  r_j^\teach r_{j'}^* \left[ (\pi/2 - \phi_{j,j'}^{\teach *})^2 + 
        O((\phi_{j,j'}^{\teach *} - \pi/2)^4) -  (\phi_{j,j'}^{\teach *} - \pi/2)^2 +\textrm{O}((\phi_{j,j'}^{\teach *} - \pi/2)^4)\right] \\
        & = \frac{\pi}{2} r_j^\teach r_{j'}^* \inner{\theta_j^\teach}{\theta_{j'}^*}
        + \textrm{O}((\phi_{j,j'}^{\teach *} - \pi/2)^4).
        \end{align*}
        By applying the same argument to the all cross terms, we obtain that 
        \begin{align*}
        & \Bigl\langle r_j^\teach \sigma(\inner{\theta_j^\teach}{\cdot})- r_j^*\sigma(\inner{\theta_j^*}{\cdot}),
        r_{j'}^\teach \sigma(\inner{\theta_{j'}^\teach}{\cdot})- r_{j'}^*\sigma(\inner{\theta_{j'}^*}{\cdot})\Bigr\rangle_{\LPx}
        \\
        & = 
        \frac{1}{2\pi d}\Bigg\{(r_j^\teach - r_j^*)(r_{j'}^\teach - r_{j'}^*)  
        + \frac{\pi}{2} 
        \left( 
        r_j^\teach r_{j'}^\teach \inner{\theta_j^\teach}{\theta_{j'}^\teach}
        - r_j^\teach r_{j'}^* \inner{\theta_j^\teach}{\theta_{j'}^*}
        - r_j^* r_{j'}^\teach \inner{\theta_j^*}{\theta_{j'}^\teach}
        + r_j^* r_{j'}^* \inner{\theta_j^*}{\theta_{j'}^*}
        \right)
        \Bigg\} \\
        & ~~~~+ \textrm{O}(\text{higher order})
        \\
        & =
        \frac{1}{2\pi d}\Bigg\{(r_j^\teach - r_j^*)(r_{j'}^\teach - r_{j'}^*)  
        + \frac{\pi}{2} 
        \inner{r_j^\teach\theta_j^\teach - r_j^*\theta_j^*}{r_{j'}^\teach\theta_{j'}^\teach - r_{j'}^*\theta_{j'}^*}
        \Bigg\}
        + \textrm{O}(\text{higher order}).
        \end{align*}
        Combining all evaluations, we have that 
        \begin{align*}
        &  \|f(\cdot;\nu^*) - f(\cdot;\nu^\teach)\|^2_{\LPx} \\    
        = &  
         \frac{1}{2d}
         \sum_{j=1}^m
        \left[ (r_j^\teach - r_j^*)^2 
        + 2 r_j^\teach r_j^* \left(1 - \frac{\phi}{\pi} \right) (1- \inner{\theta_j^*}{\theta_j^\teach}) 
        \right]  \\
        & + \sum_{j\neq j'}
        \frac{1}{2\pi d}\Bigg\{(r_j^\teach - r_j^*)(r_{j'}^\teach - r_{j'}^*)  
        + \frac{\pi}{2} 
        \inner{r_j^\teach\theta_j^\teach - r_j^*\theta_j^*}{r_{j'}^\teach\theta_{j'}^\teach - r_{j'}^*\theta_{j'}^*}
        \Bigg\}
        + \textrm{O}(\text{higher order}) \\
        = & 
        \sum_{j=1}^m \sum_{j'=1}^m
        \left[\frac{1}{2\pi d}(r_j^\teach - r_j^*)(r_{j'}^\teach - r_{j'}^*)  
        + \frac{1}{4 d}\inner{r_j^\teach\theta_j^\teach - r_j^*\theta_j^*}{r_{j'}^\teach\theta_{j'}^\teach - r_{j'}^*\theta_{j'}^*} \right] \\
        & + 
        \sum_{j=1}^m
        \left[ \left(\frac{1}{2d} - \frac{1}{2\pi d} - \frac{1}{4 d}\right)(r_j^\teach - r_j^*)^2
        + 
        \left(\frac{1}{d} \left(1 - \frac{\phi_{j,j}^{\teach *}}{\pi} \right) - \frac{1}{2d}\right) r_j^\teach r_j^*  (1- \inner{\theta_j^*}{\theta_j^\teach}) 
        \right]
        + \textrm{O}(\text{higher order}).
        \end{align*}
        Note that the second term in the right hand side can be lower bounded by 
        \begin{align*}
        \begin{cases}
        \frac{1}{d}\left( \frac{1}{4} - \frac{1}{2\pi}\right)
        \sum_{j=1}^m (r_j^\teach - r_j^*)^2,  \\
        \min\left\{\frac{1}{12d}, \frac{1}{2d} \left(\frac{1}{2} - \frac{\phi_{j,j}^{\teach *}}{\pi} \right)\right\}
        \underset{j}{\min}~(r_j^\teach r_j^*)\sum_{j=1}^m \dist^2(\theta_j^\teach,\theta_j^*).
        \end{cases}
        \end{align*}
        In addition to this evaluation, by noticing 
        $$
        \|f(\cdot;\nu^*) - f(\cdot;\nu^\teach)\|^2_{\LPx} - \|f(\cdot;\nu^*) - f(\cdot;\nu^\teach)\|^2_n
        = \textrm{O}_p \left(\sum_{j=1}^m \frac{(r_j^* -r_j^\teach)^2 +\dist^2(\theta_j^\teach,\theta_j^*)}{\sqrt{n}} \right),
        $$
        and 
        $$
        \|f(\cdot;\nu^*) - f(\cdot;\nu^\teach)\|^2_n \leq 
        \lambda \sum_{j=1}^m |r_j^\teach - r_j^*|
        \leq  \frac{1}{2\mu}m \lambda^2 + \frac{\mu}{2} \sum_{j=1}^m (r_j^\teach - r_j^*)^2,
        $$
        for $\mu=\frac{1}{d}\left( \frac{1}{4} - \frac{1}{2\pi}\right)$, we finally obtain that 
        $$
        \sum_{j=1}^m  (r_j^\teach - r_j^*)^2
        = \textrm{O}\left( m \lambda^2\right),~~~
        \sum_{j=1}^m  \dist^2(\theta_j^\teach,\theta_j^*)
        = \textrm{O}\left( m \lambda^2\right),
        $$
        with high probability. 
    \end{proof}

\onecolumn




\section{Proof of Theorem~\ref{globalconvergence}}
In this section, we give the proof of Theorem~\ref{globalconvergence}.

\subsection{Preliminaries}

First, we ensure boundedness of the gradients during the optimization, which is required in the proof. 
These follow from the boundedness of the objective function (Assumption~\ref{ass:boundedness}). 

\begin{Lem} Under Assumptions \ref{ass:boundedness} and \ref{ass:bondedinput}, 
it holds that for any $j \in [M]$ and $k=0,1,2,\dots$, 
    \label{gradientbound}
\textrm{
    \begin{align}
        &\frac{1}{n}\sum_{i=1}^n|f(x_i;\Theta_k)-y_i|+\lambda \leq 2\sqrt{n}C_F+\lambda=:C_1,\\
        &\left\|\frac{1}{n}\sum_{i=1}^n(f(x_i;\Theta_{k})-y_i)x_i\1\{\inner{w_{j,k}}{x_i} \geq 0\}\right\|\leq 2\sqrt{n}C_F=:C_2.
    \end{align}
}
\end{Lem}

These bounds are used several times throughout the proof. 
From this, we can derive the following relationship between the norms of $a_{j,k}$ and $w_{j,k}$.

\begin{Lem}
    \label{awinequality}
    Under Assumptions~\ref{ass:boundedness} and \ref{ass:bondedinput}, if $\alpha<2/C_2$, it holds that for any $j,k$,
    \begin{enumerate}
        \item $|a_{j,k}|\leq \|w_{j,k}\|$,
        \item $|w_{j,k}|^2\leq a_{j,k}^2+1$.
    \end{enumerate} 
\end{Lem}
\begin{proof}
    We prove these inequalities by induction on $k$. 
    In the case $k=0$, it holds clearly by the initialization rule.
    Assume that each inequality holds for $k=k_0$, then for any $j$, we have
    \begin{align}
        |a_{j,k_0+1}|^2-\|w_{j,k_0+1}\|^2 &= |a_{j,k_0}-\eta_{j,k_0}g_j(\Theta_{k_0})|^2-\|w_{j,k_0}-\eta_{j,k_0}h_j(\Theta_{k_0})\|^2\nonumber\\
        &=|a_{j,k_0}|^2 - \|w_{j,k_0}\|^2 \nonumber\\
            &+\eta_{j,k_0}^2(\|w_{j,k_0}\|^2-a_{j,k_0}^2)\left\|\frac{1}{n}\sum_{i=1}^n(f(x_i;\Theta_{k_0})-y_i)x_i\1\{\inner{w_{j,{k_0}}}{x_i} \geq 0\}\right\|^2 \nonumber\\
        &\leq  \left(1-\frac{\alpha^2}{4}\left\|\frac{1}{n}\sum_{i=1}^n(f(x_i;\Theta_{k_0})-y_i)x_i\1\{\inner{w_{j,k_0}}{x_i} \geq 0\}\right\|^2\right)(|a_{j,k_0}|^2 - \|w_{j,k_0}\|^2), \label{awdist}
    \end{align}
    where we used the inequality $\eta_{j,k_0}\leq \alpha/2$. By Lemma~\ref{gradientbound}, we get the inequality of $k =k_0+1$ under the assumption $\alpha<2/C_2$.  
\end{proof}

\subsection{Conic Gradient Descent}
    In this section, we explain our proof strategy to show Theorem~\ref{globalconvergence}. 
    The key technical tool in our proof is to fully make use of the update in the measure space.
    At first, we consider the update of $(r_{j,k}, \theta_{j,k})\in\R\times\sd$, which are amplitude and location of each Dirac measure.
    By the update rule of the parameters, 
    we obtain the following recursive expression of each parameter: 
    \begin{align*}
        r_{j,k+1} &= (a_{j.k}-\eta_{j,k}g(a_{j, k}))\|w_{j,k} - \eta_{j,k}h_j(\Theta_k) \|\\
        &= (a_{j.k}-\eta_{j,k}g(a_{j, k}))\left(\|w_{j,k}\| - \eta_{j,k}\frac{\inner{w_{j,k}}{h_j(\Theta_k)}}{\|w_{j,k}\|}+\delta w_{j,k} \right)\\
        &=r_{j,k}-\eta_{j,k}\frac{a_{j,k}^2+\|w_{j,k}\|^2}{|a_{j,k}|\|w_{j,k}\|}\left(\frac{1}{n}\sum_{i=1}^n(f(x_i;\nu)-y_i)\sigma(\inner{\theta_{j,k}}{x_i)})+\lambda~\sign(r_{j,k})\right)r_{j,k}+\delta r_{j,k}, \\
        \theta_{j,k+1} &= \frac{w_{j,k+1}}{\|w_{j,k+1}\|} = \frac{w_{j,k}-\eta_{j,k}h_j(\Theta_k)}{\|w_{j,k}-\eta_{j,k}h_j(\Theta_k)\|}\\
        &=\theta_{j,k} - \eta_{j,k}\frac{1}{\|w_{j,k}\|}(\id-\theta_{j,k}\theta_{j,k}^\T)h_j(\Theta_k)+\delta \theta_{j,k}\\
        &=\theta_{j,k} - \eta_{j,k}\frac{a_{j,k}}{\|w_{j,k}\|}(\id-\theta_{j,k}\theta_{j,k}^\T)\left(\frac{1}{n}\sum_{i=1}^n(f(x_i;\Theta_k)-y_i)x_i\1\{\inner{\theta_{j,k}}{x_i} \geq 0\}\right)+\delta \theta_{j,k},
    \end{align*}
    where $\delta w_{j,k}, \delta r_{j,k}, \delta \theta_{j,k}$ are residual higher-order terms. 
    From the view point of the measure space, this can be expressed as 
    \begin{align*}
        r_{j,k+1} &= r_{j,k}-\eta_{j,k}\frac{a_{j,k}^2+\|w_{j,k}\|^2}{|a_{j,k}|\|w_{j,k}\|}G_{\nu_k}(\theta_{j,k})r_{j,k}+\delta r_{j,k},\\
        \theta_{j,k+1} &= \theta_{j,k} - \eta_{j,k}\frac{a_{j,k}}{\|w_{j,k}\|}\grad G_{\nu_k}(\theta_{j,k})+\delta \theta_{j,k}, 
    \end{align*}
    where $G_{\nu_k}\in \partial J(\nu_k)$. 
    Here, the subdifferential $\partial J(\nu_k)$ is defined as $\partial J(\nu_k) := \{G \in \mathcal{C}(\sd) \mid J(\mu) - J(\nu_k) \geq \int G(\theta) \dif (\mu - \nu_k)~(\forall \mu \in \mathcal{M}(\sd)) \}$ which is well defined because $J(\cdot)$ is a convex function on the measure space $\mathcal{M}(\sd)$. 
    Furthermore, by the definition of $\eta_{j,k}$, this iteration can be rewritten as
    \begin{align}
        \label{riteration}
        r_{j,k+1} &= r_{j,k}-\alpha G_{\nu_k}(\theta_{j,k})r_{j,k}+\delta r_{j,k},\\
        \label{thetaiteration}
        \theta_{j,k+1} &= \theta_{j,k} - \alpha~\sign(r_{j,k})\frac{a_{j,k}^2}{a_{j,k}^2+\|w_{j,k}\|^2}\grad G_{\nu_k}(\theta_{j,k})+\delta \theta_{j,k}.
    \end{align}
    We note that the term $\delta r_{j,k}$ and $\delta\theta_{j,k}$ can be seen as ``higher order'' term by the following lemma.
    \begin{Lem}
        \label{higherorder}
        Under Assumption~\ref{ass:boundedness}, if $\alpha<1/C_1$, it holds that for any $j,k$,
        \begin{align*}
            &|\delta r_{j,k} |\leq C_1\alpha^2|G_{\nu_k}(\theta_{j,k})r_{j,k}|,\\
            &\|\delta \theta_{j,k}\| \leq 5C_2\alpha^2\frac{a_{j,k}^2}{a_{j,k}^2+\|w_{j,k}\|^2}\|\grad G_{\nu_k}(\theta_{j,k})\|.
        \end{align*}
    \end{Lem}
    \begin{proof}
        At first, by the straight-forward calculation, we have that $\|G_{\nu_k}\|_{\infty} \leq C_1$ and $\sup_{\theta \in \sd}\|\grad G_{\nu_k}(\theta)\|\leq C_2$.
        This gives that $\|\eta_{j,k}h_j(\Theta_k)\|\leq \alpha\|w_{j,k}\|\|G_{\nu_k}\|_{\infty}/2<\|w_{j,k}\|/2$.
        By using Lemma~\ref{awinequality} and Lemma~\ref{rhigherorder}, we have 
        \begin{align*}
            |\delta r_{j,k}| = |a_{j,k}-\eta_{j,k}g_j(\Theta_k)|\|\delta w_{j,k}\|&\leq 2|a_{j,k}|\frac{\eta_{j,k}^2h^2(w_{j,k})}{\|w_{j,k}\|}\\
            &\leq C_1\alpha^2|G_{\nu_k}(\theta_{j,k})r_{j,k}|.
        \end{align*}
        Moreover, by Lemma~\ref{thetahigherorder}, it holds that
        \begin{align*}
            \|\delta\theta_{j,k}\| \leq \frac{5\|\eta_{j,k}h_j(\Theta_k)\|^2}{\|w\|^2} \leq 5C_2\alpha^2\frac{a_{j,k}^2}{a_{j,k}^2+\|w_{j,k}\|^2}\|\grad G_{\nu_k}(\theta_{j,k})\|.
        \end{align*}
        These give the conclusion. 
    \end{proof}
    By this lemma, we can see that $\delta r_{j,k}$ and $\delta \theta_{j,k}$ are $O(\alpha^2)$ which is smaller than other terms. 
    
    \begin{Rem}
        \label{limitation}
        In the case $\exists j,j' \in \m~(j \neq j'),~\theta_{j,k}=\theta_{j,k'}$, we cannot represent the update by the subgradient in the measure space. However, we can avoid this problem almost surely by perturbing the step size infinitesimally.  
        In the following, we assume this does not happen for any $j,k$.
    \end{Rem}

    \citet{chizat2019sparse} considered a conic gradient descent, which is represented as follows:
    \begin{equation*}
        (r_{j,k+1}, \theta_{j,k+1})=\textrm{Ret}_{(r_{j,k},\theta_{j,k})}(-2\alpha G_{\nu_k}(\theta_{j,k})r_{j,k},- \beta \grad G_{\nu_k}(\theta_{j,k}))
    \end{equation*}
        where $\alpha, \beta>0$ are constants and Ret denotes a retraction mapping, which is defined on the manifold $\R\times \sd$ and its tangent bundle \cite{absil2009optimization}. The retraction mapping and updates in \Eqref{riteration} and \Eqref{thetaiteration} are almost equivalent in a sense that both of them represent first order approximations of the gradient descent in the manifold. Motivated by this point, we borrow the proof technique developed in \citet{chizat2019sparse}. They have shown that under several assumptions with sufficient over-parameterization and under the condition $\beta\lesssim\alpha^2$, 
        convergence of the gradient descent to the global optimum is achieved through the following two phase:
        \begin{description}
            \item{Phase I: Global exploration.} Objective value decreases until it reaches a threshold $J_0$,
            \item{Phase II: Local convergence.} The solution converges linearly to the global minimum locally around the true parameter.
        \end{description}
        There are some different points between our approach and \citet{chizat2019sparse}. One is that the step-size in the iteration of $\theta_{j,k}$ is not a constant. Indeed, by \eqref{thetaiteration}, the step size of the update in the measure space is given by 
        \begin{equation}
            \label{betaequation}
            \beta_{j,k} = \alpha\frac{a_{j,k}^2}{a_{j,k}^2+\|w_{j,k}\|^2}.
        \end{equation}
        This step size depends on $a_{j,k}^2$ and $\|w_{j,k}\|^2$ and is not constant.
        Note that by the initialization rule, $\beta_{j,0}=\frac{\alpha}{1+M^2}\ll \alpha$ for any $j\in\m$,
        and we will show that the inequality $\beta_{j,k}\ll \alpha$ for all $j,k$, which means that the step size for $\theta_{j,k}$ is much smaller than that of $r_{j,k}$.
        
        Another difference is that our analysis deals with the non-differentiable ReLU activation while \citet{chizat2019sparse} analyzed differentiable activation functions. We avoid this difficulty by utilizing  Assumption~\ref{ass:smooth}. 
        
        Moreover, \citet{chizat2019sparse} only considered a positive measure (more precisely, their argument cannot be applied to the settings where the measure $\nu$ has both positive and negative parts). In this paper, we consider this situation and overcome this difficulty by utilizing the following lemma which states that a positive (resp. negative) part of the updated measure remains positive (resp. negative) throughout the iterations.
    
    \begin{Lem}
        \label{sign}
        Under Assumptions~\ref{ass:boundedness} and \ref{ass:bondedinput}, if $\alpha<1/C_1$, the signs of $(a_{j,k})_{j \in [M]}$ (i.e.,  those of $(r_{j,k})_{j \in [M]}$) do not change throughout the iteration.
    \end{Lem}

    \begin{proof}
        By the update rule of $a_{j,k}$, we have
        \begin{align*}
            a_{j,k+1} &= a_{j,k} - \alpha\frac{\|w_{j,k}^2\|}{a_{j,k}^2+\|w_{j,k}\|^2}\left(\frac{1}{n}\sum_{i=1}^n(f(x_i;\nu)-y_i)\sigma(\inner{\theta_{j,k}}{x_i)})+\lambda~\sign(a_{j,k})\right)a_{j,k}\\
            & = \left\{1-\alpha\frac{\|w_{j,k}^2\|}{a_{j,k}^2+\|w_{j,k}\|^2}\left(\frac{1}{n}\sum_{i=1}^n(f(x_i;\nu)-y_i)\sigma(\inner{\theta_{j,k}}{x_i)})+\lambda~\sign(a_{j,k})\right)\right\}a_{j,k}.
        \end{align*}
        By using the inequalities $\frac{\|w_{j,k}^2\|}{a_{j,k}^2+\|w_{j,k}\|^2}\leq 1$ and $|\frac{1}{n}\sum_{i=1}^n(f(x_i;\nu)-y_i)\sigma(\inner{\theta_{j,k}}{x_i)})+\lambda~\sign(a_{j,k})|\leq C_1$, we get the conclusion. 
    \end{proof}

\subsection{Proof of Phase~I}
    In this section, we show the following inequality.
    \begin{prop}[Global exploration]
        \label{phase1}
        Assume that Assumption~\ref{ass:boundedness} holds. 
        Then there exists a constant $C,~C_M>0$ such that for any $J_0>J^*$ and $0<\epsilon<1/2$, 
        by setting $M$ sufficiently large as $M \geq C_M\exp(\alpha^{-2})/\alpha$ for each $\alpha > 0$ 
        and assuming the following conditions, 
        \begin{align}\label{eq:WinftyAlphaConditionC}
            W_\infty(\tau , \nu^+_0)\leq (J_0-J^*)/C,~W_\infty(\tau , \nu^-_0)\leq (J_0-J^*)/C,~\alpha \leq (J_0-J^*)^{1+\epsilon/2}/C,
        \end{align}
        then it holds that 
        \begin{equation}
            \label{phase1inequality}
            \underset{0\leq k'\leq \alpha^{-2}}{\min}J(\nu_{k'}) \leq J_0.
        \end{equation}
    \end{prop}

    Here we utilize the bound by \citet{chizat2019sparse}, which considered a positive measure, i.e., $r_{j,k}>0$ for any $j$ and $k$. By Lemma~\ref{sign}, the signs of $(r_{j,k})_{j \in [M]}$ will not change throughout the iterations. Therefore, we can apply the same argument to $\nu_k^+$ and $\nu_k^-$ separately, where $\nu_k:=\nu_k^+ - \nu_k^-$ 
    is the Hahn-Jordan decomposition. Then we get the following proposition.

    \begin{prop}
        \label{phase11}
        Suppose that Assumption~\ref{ass:boundedness} holds. In addition, suppose that $\beta_{\max}:=\underset{j\in[M],1\leq k\leq \alpha^{-2}}{\max}\beta_{j,k}\leq \alpha^{3}$. Let $B:=\underset{\nu:J(\nu)<C_F}{\sup}\|\nu\|_\textrm{BL}$, then there exists a constant $C'>0$ such that, for $\alpha<1/C_1$, it holds that
        \begin{equation*}
            \underset{1\leq k\leq \alpha^{-2}}{\min} J(\nu_k)-J^*\leq C'(\log(4B\alpha^{-1})+1)\alpha+B\|\nu^*\|_\TV(W_\infty(\tau, \nu_0^+)+W_\infty(\tau, \nu_0^-)). 
        \end{equation*}
    \end{prop}
    
    \begin{proof}
        Following the essentially same argument as Lemma~F.1 of \citet{chizat2019sparse}, it holds that 
        \begin{equation*}
            \underset{1\leq k\leq \alpha^{-2}}{\min} J(\nu_k)-J^*\leq C'\frac{\log(4B\alpha k')}{4B\alpha k'}+\beta_{\max}B^2k'+C\alpha+B\|\nu^*\|_\TV(W_\infty(\tau, \nu_0^+)+W_\infty(\tau, \nu_0^-)).
        \end{equation*}
        In particular, in the case $k'=\alpha^{-2}$, we get an upper bound as 
        \begin{equation*}
            C'\frac{\log(4B\alpha^{-1})}{4B}\alpha+\frac{\beta_{\max}}{\alpha^2}B^2+C'\alpha+B\|\nu^*\|_\TV(W_\infty(\tau, \nu_0^+)+W_\infty(\tau, \nu_0^-)).
        \end{equation*}
        With the condition $\beta_{\max}\leq \alpha^{3}$, we get the conclusion. 
    \end{proof}

    \begin{proof}[Proof of Proposition~\ref{phase1}]
        For $0<\epsilon<1/2$, there exists a constant $C_\epsilon >0$ such that $\log(u)\leq C_\epsilon u^\epsilon$. Then we have
        \begin{equation*}
            \underset{1\leq k\leq \alpha^{-2}}{\min} J(\nu_k)-J^*\leq C'(C_\epsilon B^{-1+\epsilon}\alpha^{-\epsilon}+1)\alpha+B\|\nu^*\|_\TV(W_\infty(\tau, \nu_0^+)+W_\infty(\tau, \nu_0^-))
        \end{equation*}
        This yields the conclusion that there exists a constant $C>0$ which depends on $C',C_\epsilon,B,B,\|\nu^*\|$ and the inequality (\ref{phase1inequality}) is satisfied under the condition \eqref{eq:WinftyAlphaConditionC}.
    \end{proof}

In the following, we show the inequality $\beta_{\max}\leq \alpha^{3}$. 
This intuitively means that the ``location'' $\theta_{j,k}$ does not move compared with the ``amplitude'' $r_{j,k}$. We can verify this in the setting we consider, in which $a_{j,k}$ is much smaller than $w_{j,k}$. Note that $\beta_{j,k}\leq \alpha |a_{j,k}|^2/\|w_{j,k}\|^2$. Inspired by this inequality, we evaluate $|a_{j,k}|$ and $\|w_{j,k}\|$, and prove the inequality $|a_{j,k}|\ll \|w_{j,k}\|$ for $k\leq\alpha^{-2}$.

\begin{Lem}
    \label{normbound}
        Assume that Assumption~\ref{ass:boundedness} holds. Let $\xi_{j,k} = \left(1+\frac{2}{M}\right)\left\{\prod_{k'=0}^{k-1}(1+\eta_{j,k'}C_1)-1\right\}$ $(j \in [M],~k=1,2,\dots)$, it holds that
        \begin{align}
        \label{anormbound}
        |a_{j,k}|&\leq  \frac{2}{M}+\xi_{j,k},\\
        \label{wnormbound}
        \|w_{j,k}-w_{j,0}\|&\leq\xi_{j,k}.
        \end{align}
\end{Lem}
\begin{proof}
    By the update rules of $a_{j,k}$ and $w_{j,k}$, we have that 
    \begin{align}
        |a_{j,k+1}| &\leq  |a_{j,k}|+\eta_{j,k}\left(\frac{1}{n}\sum_{i=1}^n|f(x_i;\Theta_k)-y_i|+\lambda\right)\|w_{j,k}\|, \label{abound}\\
        \|w_{j,k+1}\| &\leq \|w_{j,k}\|+\eta_{j,k}\left(\frac{1}{n}\sum_{i=1}^n|f(x_i;\Theta_k)-y_i|+\lambda\right)|a_{j,k}|\label{wbound}. 
    \end{align}
    By Lemma~\ref{gradientbound}, $\frac{1}{n}\sum_{i=1}^n|f(x_i;\Theta_k)-y_i|+\lambda\leq C_1$ for all $k$. By summing up the both sides, it holds that
    \begin{equation*}
        |a_{j,k}|+\|w_{j,k}\| \leq \left(1+\eta_{j,k}C_1\right)(|a_{j,k}|+\|w_{j,k}\|).
    \end{equation*}
    Then we have
    \begin{align*}
        \max\{|a_{j,k}|,\|w_{j,k}\|\} \leq |a_{j,k}|+\|w_{j,k}\|\leq\left(1+\frac{2}{M}\right)\prod_{k'=0}^{k-1}(1+\eta_{j,k}C_1). 
    \end{align*}
    Combining with (\ref{abound}), $|a_{j,k}|$ is bounded as
    \begin{align}
        \label{abound2}
        |a_{j,k}|&\leq  |a_{j,0}| + \sum_{k'=0}^{k-1}\eta_{j,k'}C_1\left(1+\frac{2}{M}\right)\prod_{k''=0}^{k'-1}(1+\eta_{j,k''}C_1)\nonumber\\
                &=\frac{2}{M}+\left(1+\frac{2}{M}\right)\left\{\prod_{k'=0}^{k-1}(1+\eta_{j,k}C_1)-1\right\},
    \end{align}  
    which gives the first inequality \Eqref{anormbound}. In addition, similar to \Eqref{wbound}, we have
    \begin{align*}
        \|w_{j,k+1}-w_{j,0}\| \leq \|w_{j,k}-w_{j,0}\|+\eta_{j,k}C_1|a_{j,k}|.
    \end{align*}
    Combining with the bound of $|a_{j,k}|$, we get the second inequality (\ref{wnormbound}). 
\end{proof}

From this bound, we obtain a bound on $|a_{j,k}|/\|w_{j,k}\|$ as we state as follows.

\begin{Lem}
    \label{zetabound}
    Under Assumption~\ref{ass:boundedness}, for any $j, k$ satisfying $\xi_{j,k}<1$, it holds that
    \begin{align*}
        \frac{|a_{j,k}|}{\|w_{j,k}\|}\leq \frac{2/M+\xi_{j,k}}{1-\xi_{j,k}}.
    \end{align*}
    Moreover, there exists a constant $C_M>0$ such that if $M\geq C_M\exp(\alpha^{-2})/\alpha$, it holds that $|a_{j,k}|/\|w_{j,k}\|\leq\alpha$ for any $j\in[M]$ and $k$ satisfying $1\leq k\leq \alpha^{-2}$.
\end{Lem}

\begin{proof}
    The first conclusion holds clearly by Lemma~\ref{normbound}. Then we consider the second assertion. 
    Let $\zeta_{j,k}:=\frac{|a_{j,k}|}{\|w_{j,k}\|}$. Suppose that $\xi_{j,k}\leq 1/2$ (which we verify later), it holds that
    \begin{align*}
        \zeta_{j,k} \leq 2\left(\frac{2}{M}+\xi_{j,k}\right).
    \end{align*}
    In addition, since $\eta_{j,k}=\alpha|a_{j,k}|\|w_{j,k}\|/(|a_{j,k}|^2+\|w_{j,k}\|^2)\leq\alpha\zeta_{j,k}$, we have
    \begin{equation*}
        \xi_{j,k} \leq \left(1+\frac{2}{M}\right)\left\{\prod_{k'=0}^{k-1}(1+\alpha\zeta_{j,k'}C_1)-1\right\}.
    \end{equation*}
    by the formulation of $\xi_{j,k}$. Combining these inequality, we get 
    \begin{align*}
        \zeta_{j,k} &\leq \frac{4}{M}+2\left(1+\frac{2}{M}\right)\left\{\prod_{k'=0}^{k-1}(1+\alpha\zeta_{j,k'}C_1)-1\right\},
    \end{align*}
    where we used $k\leq \alpha^{-2}$. By this inequality, let $c\gtrsim \log(\alpha^{-1})$ and $M\gtrsim \exp(c\alpha^{-2})$, then we have $\zeta_{j,k}\leq \frac{2}{M}\exp(ck)$ for any $0\leq k\leq \alpha^{-2}$ and we prove this by the induction. 
    When $k=0$ this holds with equality. Suppose that for $k_0\geq 1$, $\zeta_{j,k}\leq \frac{2}{M}\exp(ck)$ is satisfied for any $k<k_0$. Then, we have
    \begin{align}
        \zeta_{j,k_0} &\leq \frac{4}{M}+2\left(1+\frac{2}{M}\right)\left\{\prod_{k'=0}^{k_0-1}(1+\alpha\zeta_{j,k'}C_1)-1\right\}\nonumber\\
        &\leq  \frac{4}{M}+2\left(1+\frac{2}{M}\right)\left\{\left(1+\frac{2\alpha C_1}{M}\exp(c(k_0-1))\right)^{k_0}-1\right\}\nonumber\\
        &\leq  \frac{4}{M}+2\left(1+\frac{2}{M}\right)\frac{4\alpha C_1k_0}{M}\exp(c(k_0-1))\nonumber\\
        &\leq  \frac{2}{M}\biggl(2+8\alpha C_1k_0\exp(c(k_0-1))\biggr)\nonumber\\
        &\leq  \frac{2}{M}\exp(ck_0),\nonumber
    \end{align}
    where the third inequality follows from 
    \begin{align*}
        \left(1+\frac{2}{M}\exp(c(k_0-1))\right)^{k_0} &\leq 1+ \frac{2k_0}{M}\exp(c(k_0-1))+2^{k_0}\frac{4}{M^2}\exp(2c(k_0-1))\\
        &\leq 1+ \frac{4k_0}{M}\exp(c(k_0-1)),
    \end{align*}
    where we use $\sum_{j=2}^{k_0} \binom{k_0}{j}\leq 2^{k_0}$. Taking $M\geq 2\alpha^{-1}\exp(c\alpha^{-2})$, we get $\zeta_{j,k} \leq \alpha$. Finally, in this case the condition $\xi_{j,k}\leq 1/2$ remains and this gives the conclusion.
\end{proof}

By this Lemma, it holds that for sufficiently large $M$, $|a_{j,k}|/\|w_{j,k}\|$ will be small. By this inequality, we get a bound of $\beta_{j,k}$, which is supposed in the Proposition~\ref{phase11}.

\begin{Lem}
    Under Assumption~\ref{ass:boundedness}, there exists a constant $C_M>0$ such that if $M\geq C_M\exp(\alpha^{-2})/\alpha$, it holds that $\beta_{j,k}\leq\alpha^3$ for any $j\in[M]$ and $k$ satisfying $1\leq k\leq \alpha^{-2}$.
\end{Lem}

\begin{proof}
    By the definition of $\beta_{j,k}$, it holds that 
    \begin{align*}
        \beta_{j,k} = \alpha \frac{a_{j,k}^2}{a_{j,k}^2+\|w_{j,k}\|^2} \leq \alpha\zeta_{j,k}^2.
    \end{align*}
    Combining this with Lemma~\ref{zetabound}, we get the conclusion.
\end{proof}
	
\subsection{Proof of Phase~II}
In this section, we prove linear convergence to the optimal solution after a specific number of iterations. 
A key ingredient is a local analysis around the optimal parameters $(\theta_j^*)_{j=1}^{m^*}$ (remark that the global minimum is obtained by a sparse measure). We consider a local region around each $\theta_j^*$ which we define below and prove a ``sharpness inequality'' (Proposition~\ref{sharpness}) through evaluating the function value and the norm of the gradient by using a distance from the optimal parameter. 

We first divide $\sd$ by the sign of inner product with each $x_i$, i.e., each division is written by the form $\{\theta\in\sd \mid \sign(\inner{\theta}{x_1})=s_1,\dots,\sign(\inner{\theta}{x_n})=s_n\}$ for $(s_i)_{i=1}^n\in\{-1,+1\}^n$. Let $H_j$ be the region that contains $\theta_j^*$ and $R_j := \underset{\theta\in H_j}{\sup} \dist(\theta,\theta_{\it j^*})$, where $R_j>0$ by Assumption~\ref{ass:smooth}. Then we take a a value $\rho$ which satisfies $0<\rho<\frac{\min R_j}{2}$ and define $N_j(\rho)$ to be an open ball around $\theta_j^*$ with radius $\rho$ and $N_0:=\mathbb{S}^{d-1}\backslash\cup_jN_j(\rho)$. 

To prove the linear convergence, we define a kind of distance between $\nu_k$ and the global minima $\nu^*$. 
Our definition of the distance follows that of \citet{chizat2019sparse} but they are different in that we deal with a singed measure and their definition did not properly deal with average on the manifold $\sd$ while ours avoid such an average.

\begin{Def}
    Let $\nu_k=\sum_{j=1}^Mr_{j,k}\delta_{\theta_{j,k}}$ be the measure after $k$ iterations.
    For each $j\in[m^*]$, we define a local mass by $\bar{r}_{j,k}=\nu_k(N_j(\rho))$, 
    a local gap on $\sd$ with $\Delta\theta_{j,k}=\sum_{\substack{\theta_{\it j',k} \in N_j(\rho)\\\thmsign (r_{j',k})=\thmsign (r^*_j)}} |r_{j',k}|\thmdist^2( \theta_{\it j',k})$ and a local ``different signed'' mass with $\Delta {\it r_{j,k}}=\sum_{\substack{\theta_k\in N_j(\rho)\\\thmsign(r_{j',k})\neq\thmsign( r^*_j)}} |{\it r_{j',k}}|$.
    Furthermore, we define a mass of the remaining region with $r_{0,k}:=|\nu_k|(N_0)$. Finally, according to these values, we define a ``distance'' between $\nu_k$ and $\nu^*$ by
    \begin{equation*}
        D_\rho(\nu_k) =\sum_{j=1}^{m^*}(\bar{r}_{j,k}-r_j^*)^2+r_{0,k}+\sum_{j=1}^{m^*}(\Delta\theta_{j,k}+\Delta r_j).
    \end{equation*}
\end{Def}
As we see below, the term $(\bar{r}_{j,k}-r_j^*)^2$ and $r_{0,k}$ mainly affect the gap between $f(\cdot,\nu_k)$ and $f(\cdot,\nu^*)$. The term $\Delta\theta_{j,k}+\Delta r_j$ is related to the regularization term and will vanish due to the sparse regularization. 
In the following, we see that this distance upper-bounds the Wasserstein distance between $\nu_k$ and $\nu^*$. 

For the local evaluation, we firstly remark the optimality condition w.r.t. a measure.

\begin{Lem}
    \label{optimalitycondition}
    Let $f^*:=f(\cdot, \nu^*)$, under Assumption~\ref{ass:smooth}, for each $j\in [m^*]$, it holds that 
    \begin{align}
        \label{optimalitycondition:parameters}
        -\frac{1}{n}\sum_{i=1}^n(f^*(x_i)-y_i)x_i\1\{\langle\theta_j^*, x_i\rangle\geq 0\}=\lambda~\thmsign(r_j^*)\theta_j^*.
    \end{align}
    Furthermore, under Assumption~\ref{ass:nondegenerate}, for any $\theta\in\sd$ satisfying $\theta\neq \theta_j^*$ for all $j\in[m^*]$, it holds that
    \begin{equation}
        \label{strict}
        \left|\frac{1}{n}\sum_{i=1}^n(f^*(x_i)-y_i)\sigma(\langle\theta, x_i\rangle)\right|< \lambda.
    \end{equation}
\end{Lem}

Remark that \Eqref{optimalitycondition:parameters} is derived from $0\in \partial J(\nu^*)$, where  
\begin{align}
    \label{subdiff}
    \partial J(\nu^*) = \frac{1}{n}\sum_{i=1}^n(f^*(x_i)-y_i)\sigma(\inner{\cdot}{x_i})+\lambda\partial \|\nu^*\|_{\TV}
    \subset \mathcal{C}(\sd).
\end{align}
Indeed, the necessary condition for $0\in\partial J(\nu^*)$ is that
\begin{align*}
    \begin{cases}
        -\frac{1}{n}\sum_{i=1}^n(f^*(x_i)-y_i)\sigma(\inner{\theta_j^*}{x_i}) = \lambda~\sign(r_j^*),
        \\-\frac{1}{n}\sum_{i=1}^n(f^*(x_i)-y_i)x_i\1\{\langle\theta_j^*, x_i\rangle\geq 0\} = \lambda a_j \sign(r_j^*)\theta_j^*,
    \end{cases}
\end{align*}
for some $a_j\in \R$ for each $j\in[m^*]$, where we used the same argument to show \Eqref{eq:OptCondThetajTeach}. 
By putting together these equations and using the 1-homogeneity of ReLU, we obtain \Eqref{optimalitycondition:parameters}. 

Now we introduce a characterization of subgradient $\partial J(\nu^*)$ for the proof. 
By the construction, we know that $\theta \mapsto -\frac{1}{n}\sum_{i=1}^n(f^*(x_i)-y_i)x_i\1\{\langle\theta, x_i\rangle\geq 0\}$ takes a constant value $\lambda~\sign(r_j^*)\theta_j^*$ in each $N_j(\rho)$.
This leads to
\begin{align}\label{eq:thetajstarOptimalCond}
    -\frac{1}{n}\sum_{i=1}^n(f^*(x_i)-y_i)\sigma(\inner{\theta}{x_i}) = \inner{\theta}{-\frac{1}{n}\sum_{i=1}^n(f^*(x_i)-y_i)x_i\1\{\langle\theta_j^*, x_i\rangle\geq 0\}}=\lambda~\sign(r_j^*)\inner{\theta}{\theta_j^*}.
\end{align}
for $\theta\in N_j(\rho)$. This equality plays an important role in the proof. 

By using $D_\rho(\nu_k)$, we can evaluate a gap between $J(\nu_k)$ and $J^*$.
    \begin{prop}
    \label{dist}
	Under Assumption~\ref{ass:smooth}--\ref{ass:bondedinput}, there exists a constant $c_\rho>0$, $C_\rho>0$ that depends on $\rho$ and a constant $J_0$, if $J(\nu_k)<J_0$ it holds that
		\begin{equation}
			c_\rho D_\rho(\nu_k) \leq J(\nu_k)-J^* \leq C_\rho D_\rho(\nu_k).
		\end{equation}
	\end{prop}
    \begin{proof}
        First, we derive the first inequality $c_\rho D_\rho(\nu_k) \leq J(\nu_k)-J^*$. 
        Let $G^*(\theta)\in\partial J(\nu^*)$, i.e.,
        \begin{equation*}
            G^*(\cdot) \in \frac{1}{n}\sum_{i=1}^n (f^*(x_i)-y_i)\sigma(\langle\cdot, x_i\rangle)+\lambda \partial\|\nu^*\|_\TV.
        \end{equation*}
        Here, we take $\eta\in \partial \|\nu^*\|_\TV$ which satisfies $\eta(\theta_{j,k})=\sign(r_{j,k})$ (if necessary, we apply the modification as Remark~\ref{limitation}). 
        In this case, by the straight-forward calculation and noticing $\int_{\sd}G^*(\theta)\dif\nu^* = 0$, it holds that 
        \begin{equation}
            \label{taylor}
			J(\nu_k)-J^* = \underbrace{\int_{\sd}G^*(\theta)\dif\nu_k}_\textrm{(I)} + \underbrace{\frac{1}{2}\|f(\cdot;\nu_k)-f^*\|^2_n }_\textrm{(II)}.
		\end{equation}
        We evaluate each term. For the term (I), we have  
        \begin{equation*}
            \textrm{(I)}=\sum_{j=1}^{m^*}\int_{N_j(\rho)}G^*(\theta)\dif\nu_k +\int_{N_0}G^*(\theta)\dif\nu_k. 
        \end{equation*}
        Now we have $C_\rho':=\underset{\theta\in N_0}{\inf}\left\{\lambda-\left|\frac{1}{n}\sum_{i=1}^n(f^*(x_i)-y_i)\sigma(\langle\theta, x_i\rangle)\right|\right\}>0$ by the inequality~(\ref{strict}) in Lemma~\ref{optimalitycondition} and compactness of $N_0$. Then the second term of the right hand side is evaluated as
        \begin{equation}
            \label{N0bound}
			\int_{N_0}G^*(\theta)\dif\nu_k \geq r_{0,k}C_{\rho}'. 
		\end{equation}
		For the term $\int_{N_j(\rho)}G^*(\theta)\dif\nu_k$, \Eqref{eq:thetajstarOptimalCond} and the definition of $\eta$ yield that
		\begin{align}
            \int_{N_j(\rho)}G^*(\theta)\dif\nu_k&\geq\int_{N_j(\rho)}(\lambda~\sign (r_j^*)\langle\theta^*_j, \theta\rangle+\lambda~\sign (\eta(\theta)))\dif\nu_k(\theta) \nonumber\\
            &\geq \lambda~\int_{N_j(\rho)}\min\{1-\sign (r_j^*\eta(\theta))\inner{\theta_j^*}{\theta},1\}\dif|\nu_k|\nonumber\\
            &\geq \lambda(c\Delta\theta_{j,k}+\Delta r_j), \label{Njbound}
		\end{align}
        where $c>0$ is a constant which does not depend on other parameters, which is derived from Lemma~\ref{distevaluation}. Combining \Eqref{N0bound} and \Eqref{Njbound}, we get 
        \begin{align*}
            \textrm{(I)}\geq \lambda\sum_{j=1}^{m^*}(c\Delta\theta_{j,k}+\Delta r_{j,k})+r_{0,k}C_{\rho}'.
        \end{align*}
        Next for the term (II), we consider the decomposition $f(\cdot;\nu_k)=:\sum_{j=1}^{m^*}f_{j,k}+f_{0,k}$, where $f_{j,k}(\cdot)=\int_{N_j(\rho)} \sigma(\inner{\theta}{\cdot})\dif \nu_k$. Then we have
		\begin{align}
			\|f(\cdot;\nu_k)-f^*\|_n^2&=\left\|\sum_{j=1}^{m^*}\biggl(r_j^*\sigma(\inner{\theta_j^*}{\cdot})-f_{j,k}\biggr)+f_{0,k}\right\|_n^2\nonumber\\
            &\geq \left\|\sum_{j=1}^{m^*}\biggl(r_j^*\sigma(\inner{\theta_j^*}{\cdot})-f_{j,k}\biggr)\right\|_n^2-2\left\|\sum_{j=1}^{m^*}(r_j^*\sigma(\inner{\theta_j^*}{\cdot})-f_{j,k})\right\|_n\left\|f_{0,k}\right\|_n\nonumber\\
            &\geq \left\|\sum_{j=1}^{m^*}\biggl(r_j^*\sigma(\inner{\theta_j^*}{\cdot})-f_{j,k}\biggr)\right\|_n^2-2\left\|\sum_{j=1}^{m^*}(r_j^*\sigma(\inner{\theta_j^*}{\cdot})-f_{j,k})\right\|_nr_{0,k},\label{iibound1}
        \end{align}
        where the last inequality follows 
        from $\int_{A}\sigma(\inner{\theta}{x})\dif\nu_k\leq |\nu_k|(A)$ for $A \subset \sd$ if $\|x\|=1$. For the first term, we have
        \begin{align*}
            r_j^*\sigma(\inner{\theta_j^*}{\cdot})-f_{j,k}&=r_j^*\sigma(\inner{\theta_{j,k}^*}{\cdot})-\sum_{j': \theta_{j'}\in N_j(\rho)}r_{j',k}\sigma(\inner{\theta_{j',k}}{\cdot})\\
            &= (r_j^*-\bar{r}_{j,k})\sigma(\inner{\theta_j^*}{\cdot}) - \sum_{j':\theta_{j'}\in N_j(\rho)}r_{j',k}(\sigma(\inner{\theta_{j',k}}{\cdot})-\sigma(\inner{\theta_j^*}{\cdot})),
        \end{align*}
        which gives 
        \begin{align*}
            \left\|\sum_{j=1}^{m^*}\biggl(r_j^*\sigma(\inner{\theta_j^*}{\cdot})-f_{j,k}\biggr)\right\|_n^2\geq  &\kappa\sum_{j=1}^{m^*}(r_j^*-\bar{r}_{j,k})^2-2\left(\sum_{j=1}^{m^*}|r_j^*-\bar{r}_{j,k}|\right)\left(\sum_{j=1}^{m^*}(\Delta\theta_{j,k}+\Delta r_{j,k})\right)
        \end{align*}
        where the first term is derived by Assumption~\ref{ass:strongconvexity}.
        Combining with \Eqref{iibound1}, we have a lower bound of (II) as 
        \begin{equation*}
            \textrm{(II)} \geq \kappa\sum_{j=1}^{m^*}(\bar{r}_{j,k}-r_j^*)^2-2\left\|\sum_{j=1}^{m^*}(r_j^* \sigma(\inner{\theta_j^*}{\cdot})-f_{j,k})\right\|_nr_{0,k}-2\left(\sum_{j=1}^{m^*}|\bar{r}_{j,k}-r_j^*|\right)\left(\sum_{j=1}^{m^*}(\Delta\theta_{j,k}+\Delta r_{j,k})\right).
        \end{equation*}
        Finally, we have $\max\{r_{0,k}, \sum_{j=1}^{m^*}(\Delta\theta_{j,k}+\Delta r_{j,k})\}\leq \max\left\{1/(c\lambda),1/\lambda,1/C_\rho'\right\}(J(\nu_k)-J^*)$ by the lower bound of (I). For sufficiently small $J(\nu_k)-J^*$, by transposing the minus term and using the arithmetic-geometric mean relation, this leads to a bound 
        \begin{align*}
            \sum_{j=1}^{m^*}(\bar{r}_{j,k}-r_j^*)^2 \leq C(J(\nu_k)-J^*)
        \end{align*}
        for some constant $C>0$. Combining (I) and (II), we get the conclusion.

        To get the upper bound we use the equality \Eqref{taylor} as 
        \begin{align*}
            \textrm{(I)} &= \int_{N_0}G^*(\theta)\dif \nu_k+\sum_{j=1}^{m^*}\int_{N_j(\rho)}G^*(\theta)\dif\nu_k\\
            &\leq r_{0,k}\|G^*(\cdot)\|_\infty +\sum_{j=1}^{m^*}\int_{N_j(\rho)}\lambda(-\sign(r_j^*)\inner{\theta}{\theta_j^*}+\sign(\eta(\theta)))\dif\nu_k\\
            &\leq r_{0,k}\|G^*(\cdot)\|_\infty  + \sum_{j=1}^{m^*}2\lambda(\Delta\theta_{j,k}+\Delta r_{j,k})\\
            &\leq C_1r_{0,k}+\sum_{j=1}^{m^*}2\lambda(\Delta\theta_{j,k}+\Delta r_{j,k}).
        \end{align*}
        For the term (II), we follow the similar argument to the lower bound as
        \begin{align*}
            \|f(\cdot;\nu_k)-f^*\|_n^2&=\left\|\sum_{j=1}^{m^*}\biggl(r_j^*\sigma(\inner{\theta_j^*}{\cdot})-f_{j,k}\biggr)+f_{0,k}\right\|_n^2\\
            &\leq 2\left\|\sum_{j=1}^{m^*}\biggl(r_j^*\sigma(\inner{\theta_j^*}{\cdot})+f_{j,k}\biggr)\right\|_n^2+2\left\|f_{0,k}\right\|_n^2\\
            &\leq 4\kappa_{\max}\sum_{j=1}^{m^*}(r_j^*-\bar{r}_{j,k})^2+4\left(\sum_{j=1}^{m^*}(\Delta\theta_{j,k}+\Delta r_{j,k})\right)^2+2r^2_{0,k},
        \end{align*}
        where $\kappa_{\max}$ is the largest eigenvalue of a matrix $\left(\frac{1}{n}\sum_{j=1}^n\sigma(\inner{\theta^*_{j_1}}{x_i})\sigma(\inner{\theta^*_{j_2}}{x_i})\right)_{j_1,j_2}\in \R^{m^*\times m^*}$ which only depends on $m^*$. This gives the conclusion.
    \end{proof}

To ensure the linear convergence in the local region, we evaluate how much $J(\nu_k)$ decrease in each iteration as in the following lemma.

\begin{Lem}
    \label{decrease}
    Under Assumptions~\ref{ass:smooth},~\ref{ass:boundedness},~\ref{ass:boundedness} and ~\ref{ass:bondedinput}, if $\alpha<\min\{1/8C_1,\rho/C_2,1/(10C_2),(\lambda/C_F)^2/8\}$, then for any positive integer $k$, it holds that
    \begin{equation}
        J(\nu_{k+1})-J(\nu_k) \leq -\frac{1}{2}g^2_{\nu_k}+\alpha r_{0,k}\|f^*(\cdot)-f(\cdot,\nu_k)\|_n^2,
    \end{equation}
    where $g^2_{\nu_k}:=\int_{\sd}(\alpha G_{\nu_k}^2(\theta)+\|\grad G_{\nu_k}(\theta)\|^2\beta(\theta))\dif|\nu_k|$ and $\beta(\theta)=\begin{cases} \beta_{j,k}&\theta=\theta_{j,k}\\ 0 &\textrm{o.w.}\end{cases}$. 
\end{Lem}

\begin{proof}
    For a continuous function $\phi:\sd \to \R$, we have that 
    \begin{align*}
        \int_{\sd}\phi(\theta)\dif(\nu_{k+1}-\nu_k) &= \sum_{j=1}^M(r_{j,k+1}\phi(\theta_{j,k+1}) -r_{j,k}\phi(\theta_{j,k}))\\
        &=\sum_{j=1}^M(r_{j,k+1}\phi(\theta_{j,k+1}) -r_{j,k}\phi(\theta_{j,k+1}))+\sum_{j=1}^M(r_{j,k}\phi(\theta_{j,k+1}) -r_{j,k}\phi(\theta_{j,k}))
    \end{align*}
    In particular if we take $\phi \equiv G_{\nu_k}=\frac{1}{n}\sum_{i=1}^n(f(x_i;\nu_k)-y_i)\sigma(\inner{\cdot}{x_i})+\lambda \eta_k(\cdot)\in \partial J(\nu_k)$ where $\eta$ satisfies $\eta_k(\theta_{j,k+1})=\sign(r_{j,k+1})$ for any $j\in[M]$, we have
    \begin{align*}
        J(\nu_{k+1})-J(\nu_k) 
        = &\int_{\sd} G_{\nu_k}(\theta)\dif(\nu_{k+1}-\nu_k)+\frac{1}{2}\|f(\cdot;\nu_{k+1})-f(\cdot;\nu_k)\|_n^2\\
        = &-\sum_{j=1}^M\alpha |r_{j,k}||G_{\nu_k}(\theta_{j,k})|^2+\sum_{j=1}^M2C_1^2\alpha^2 |r_{j,k}||G_{\nu_k}(\theta_{j,k})|^2\\
        &+\sum_{j=1}^Mr_{j,k}(G_{\nu_k}(\theta_{j,k+1})-G_{\nu_k}(\theta_{j,k}))+\frac{1}{2}\|f(\cdot;\nu_{k+1})-f(\cdot;\nu_k)\|_n^2,
    \end{align*}
    where we used Lemma~\ref{higherorder} and $\|G_{\nu_k}\|_{\infty}\leq C_1$ to bound the term related to $\delta r_{j,k}$. For the term $\frac{1}{2}\|f(\cdot;\nu_{k+1})-f(\cdot;\nu_k)\|_n^2$, by taking $\phi\equiv\sigma(\inner{\cdot}{x_i})$ for $(x_i)_{i=1}^n$ and using the 1-Lipschitz continuity of $\sigma(\cdot)$, we have 
    \begin{align*}
        \|f(\cdot;\nu_{k+1})-f(\cdot;\nu_k)\|_n^2 &= \left\|\int_{\sd}\sigma(\inner{\theta}{\cdot})\dif(\nu_{k+1}-\nu_k)\right\|_n^2\\
        &=\frac{1}{n}\sum_{i=1}^n\left(\int_{\sd}\sigma(\inner{\theta}{x_i})\dif(\nu_{k+1}-\nu_k)\right)^2\\
        &\leq \left(\sum_{j=1}^M\left(|r_{j,k+1}-r_{j,k}|+\|\theta_{j,k+1}-\theta_{j,k}\||r_{j,k}|\right)\right)^2\\
        &\leq \left(\int_{\sd}(\alpha|G_{\nu_k}(\theta_{j,k})|+\beta_{j,k}\|\grad G_{\nu_k}(\theta_{j,k})\|)\dif|\nu_k|\right)^2
    \end{align*}
    and this can be upper bounded by
    \begin{align*}
        & \|\nu_k\|_{\TV}^2\int_{\sd}(\alpha|G_{\nu_k}(\theta_{j,k})|+\beta_{j,k}\|\grad G_{\nu_k}(\theta_{j,k})\|)^2\dif|\nu_k|/\|\nu_k\|_\TV\\
        &\leq 2\left(\frac{C_F}{\lambda}\right)^2\alpha g_{\nu_k}^2,
    \end{align*}
    by the Jensen's inequality and the inequalities $(a+b)^2\leq 2a^2+2b^2$, $\beta_{j,k}\leq \alpha$ and $\|\nu_k\|_\TV\leq C_F/\lambda$ which is derived from Assumption~\ref{ass:boundedness}.
    Finally we consider the term $\sum_{j=1}^Mr_{j,k}(G_{\nu_k}(\theta_{j,k+1})-G_{\nu_k}(\theta_{j,k}))$. If we take $\alpha<\rho /C_2$, $\theta_{j,k}\in N_j(\rho)$ means that $\theta_{j,k+1}$ remains in $H_j$ in which $G_{\nu_k}$ is an (locally) affine function. 
    For $\theta_{j,k} \in N_0$, we note $\|\grad G_{\nu_k}(\theta_{j,k})\|\leq \|f^*(\cdot)-f(\cdot,\nu_k)\|_n$ and $\|G_{\nu_k}(\cdot)\|_\textrm{Lip}\leq \|f^*(\cdot)-f(\cdot,\nu_k)\|_n$. Then combining all of them, we get
    \begin{align*}
        J(\nu_{k+1})-J(\nu_k) \leq & -\sum_{j=1}^M\alpha |r_{j,k}||G_{\nu_k}(\theta_{j,k})|^2+\sum_{j=1}^M\frac{1}{4}\alpha |r_{j,k}||G_{\nu_k}(\theta_{j,k})|^2\\
        &-\sum_{j:\theta_{j,k}\in \cup_{j'\geq 1} N_{j'}(\rho)} \frac{1}{2} \beta_{j,k}|r_{j,k}|\|\grad G_{\nu_k}(\theta_{j,k})\|^2 +\alpha r_{0,k}\|f^*(\cdot)-f(\cdot,\nu_k)\|_n^2+\frac{1}{4}g_{\nu_k}^2\\
        =& -\frac{1}{2}g^2_{\nu_k}+\alpha r_{0,k}\|f^*(\cdot)-f(\cdot,\nu_k)\|_n^2,
    \end{align*}
    which gives the conclusion.    
\end{proof}

Then we give a lower bound of $g_{\nu_k}^2$, in terms of $J(\nu_k)-J^*$. 
    
    \begin{prop}[sharpness inequality]
        \label{sharpness}
        Under Assumptions~\ref{ass:smooth}--\ref{ass:bondedinput}, there exist constants $J_0>J^*$ and $\kappa_1>0$ such that if $J(\nu_k) \leq J_0$, it holds that
		\begin{equation}
			\kappa_1(J(\nu_k) -J^*) \leq g_{\nu_k}^2.
        \end{equation}
    \end{prop}
    
    To prove this inequality, we prepare a lemma which ensures the sharpness of the gradient in terms of the distance $D_\rho(\nu_k)$. 

    \begin{Lem}
        \label{gbound}
         Under Assumption~\ref{ass:smooth}--\ref{ass:bondedinput}, there exists a constant $J_0>J^*$ and a constant $C_g>0$ that depends on $\alpha$, if $J(\nu_k)\leq J_0$, it holds that
		\begin{equation}
			g^2_{\nu_k} \geq C_gD_\rho(\nu_k).
        \end{equation}
	\end{Lem}
	
	\begin{proof}
        At first, let $0<\beta_0 <\alpha/4$ and we consider a decomposition
        \begin{align}\label{eq:gnuklowerdecomp}
            g^2_{\nu_k}&= \int_{\sd}(\alpha G^2(\theta)+\|\grad G_{\nu_k}(\theta)\|^2\beta(\theta))\dif|\nu_k| \notag \\
            &\geq \int_{\sd}(\alpha G^2(\theta)+\beta_0\|\grad G_{\nu_k}(\theta)\|^2)\dif|\nu_k|-\int_{\sd\cup \{\beta(\theta)\leq \beta_0\}}\beta_0\|\grad G_{\nu_k}(\theta)\|^2\dif|\nu_k|.
        \end{align}
        For the second term of the right hand side, $\beta_{j,k}\leq \beta_0$ means 
        \begin{align*}
            \alpha\frac{a_{j,k}^2}{a_{j,k}^2+\|w_{j,k}\|^2} \leq \beta_0.
        \end{align*}
        By using an inequality $\|w_{j,k}\|^2\leq |a_{j,k}|^2+1$, which is derived from Lemma~\ref{awinequality}, we obtain
        \begin{align*}
            \alpha\frac{a_{j,k}^2}{2a_{j,k}^2+1}\leq \beta_0.
        \end{align*}
        Rearranging this inequality, we get $|a_{j,k}|\leq \sqrt{\frac{2\beta_0}{\alpha-2\beta_0}}<\sqrt{\beta_0/\alpha}$, therefore $|r_{j,k}|\leq |a_{j,k}|(|a_{j,k}|+1)\leq 2|a_{j,k}|\leq 2\sqrt{\beta_0/\alpha}$.
        Then we have
        \begin{equation}
            \label{gbound1}
            \int_{\sd\cup \{\beta(\theta)\leq \beta_0\}}\|\grad G_{\nu_k}(\theta)\|^2\beta_0\dif|\nu_k|\leq 2M\beta_0\frac{\beta_0}{\alpha}\|f(\cdot,\nu_k)-f^*\|_n^2\leq 4C_\rho M\beta_0\sqrt{\frac{\beta_0}{\alpha}}D_\rho(\nu_k).
        \end{equation}
        For evaluating the first term of the right hand side of \Eqref{eq:gnuklowerdecomp}, we have
        \begin{align*}
           \int_{\sd}(\alpha G^2(\theta)+\|\grad G_{\nu_k}(\theta)\|^2\beta_0)\dif|\nu_k|\geq \min\{\alpha,\beta_0\}\int_{\sd}(G^2(\theta)+\|\grad G_{\nu_k}(\theta)\|^2)\dif|\nu_k|
        \end{align*}
        Now we take $\eta(\theta)\in \|\nu_k\|_{\TV}$, then if $\theta\in N_j(\rho)$, it holds that 
		\begin{align*}
			G_{\nu_k}(\theta) &= \frac{1}{n} \sum_{i=1}^n(f(x_i;\nu_k)-y_i)\sigma(\inner{\theta}{x_i}) + \lambda~\sign(\eta(\theta))\\
			 &= \frac{1}{n} \sum_{i=1}^n(f(x_i;\nu^*)-y_i)\sigma(\inner{\theta}{x_i}) + \lambda~\sign(\eta(\theta))+\sum_{i=1}^nf(x_i;\nu_k-\nu^*)\sigma(\inner{\theta}{x_i})\\
			 &=-\lambda~\sign(r_j^*)\inner{\theta}{\theta_j^*} + \lambda~\sign(\eta(\theta))+\sum_{i=1}^nf(x_i;\nu_k-\nu^*)\sigma(\inner{\theta}{x_i})
        \end{align*}
        and
		\begin{align*}
			\grad G_{\nu_k}(\theta) &= (\id-\theta\theta^\T)\frac{1}{n}\sum_{i=1}^n(f(x_i;\nu_k)-y_i)x_i\1\{\inner{\theta}{x_i}\geq 0\}\\
			&=(\id-\theta\theta^\T)\left(\theta_j^* + \frac{1}{n}\sum_{i=1}^n(f(x_i;\nu_k-\nu^*)-y_i)x_i\1\{\inner{\theta}{x_i}\geq 0\}\right).
        \end{align*}
        Then we have 
		\begin{align*}
			\int_{\sd} (G^2(\theta) + \|\grad G_{\nu_k}(\theta)\|^2) \dif|\nu_k| &\geq\sum_{j=1}^{m^*}\int_{N_j(\rho)}(G^2(\theta) + \|\grad G_{\nu_k}(\theta)\|^2) \dif|\nu_k|\\
			&= \sum_{j=1}^{m^*}\int_{N_j(\rho)} \biggl(\underbrace{\left(-\lambda~\sign(r_j^*)\inner{\theta}{\theta_j^*} + \lambda~\sign(\eta(\theta))\right)^2+\|(\id-\theta\theta^\T)\theta_j^*\|^2}_\textrm{(I)}\\
			&+ \underbrace{2(-\lambda~\sign(r_j^*)\inner{\theta}{\theta_j^*} + \lambda~\sign(\eta(\theta)))\left(\frac{1}{n}\sum_{i=1}^nf(x_i;\nu_k-\nu^*)\sigma(\inner{\theta}{x_i})\right)}_\textrm{(II)}\\
			&+\underbrace{2\theta_j^{* \rm T}(\id -\theta\theta^\T)\frac{1}{n}\sum_{i=1}^nf(x_i;\nu_k-\nu^*)x_i\1\{\inner{\theta}{x_i}\geq 0\}}_\textrm{(III)}\\
			&+\underbrace{\left|\frac{1}{n}\sum_{i=1}^nf(x_i;\nu_k-\nu^*)x_i\1\{\inner{\theta}{x_i}\geq 0\}\right|^2}_\textrm{(IV)}\biggr)\dif|\nu_k|.
		\end{align*}
        Now we evaluate each term in the right hand side.  
        The term $\textrm{(I)}$ can be evaluated as 
        \begin{align*}
            \textrm{(I)}&=\biggl(-\lambda~\sign(r_j^*)\inner{\theta}{\theta_j^*} + \lambda~\sign(\eta(\theta))\biggr)^2+\|(\id-\theta\theta^\T)\theta_j^*\|^2\\
            &=\biggl(-\lambda~\sign(r_j^*)\inner{\theta}{\theta_j^*}+ \lambda~\sign(\eta(\theta))\biggr)^2 +1-\inner{\theta}{\theta_j^*}^2\\
            &\geq \begin{cases}(1+\lambda^2)(1-\inner{\theta}{\theta_j^*})&(\sign(\eta(\theta))=\sign(r_j^*)), \\\lambda^2 & (\text{otherwise}), \end{cases}
        \end{align*}
        which gives 
        \begin{align*}
            \int_{N_j(\rho)}\textrm{(I)}\dif|\nu_k|\geq \lambda^2(\Delta\theta_{j,k}+\Delta r_{j,k}).
        \end{align*}
        For the terms (II) and (III), we have for any $\theta\in\sd$,
        \begin{align*}
            \left\|\frac{1}{n}\sum_{i=1}^nf(x_i;\nu_k-\nu^*)x_i\1\{\inner{\theta}{x_i}\geq 0\}\right\| \leq \|f(\cdot;\nu_k)-f^*\|_n\leq \left(C_\rho D_\rho(\nu_k)\right)^{\frac{1}{2}},
        \end{align*}
        by Lemma~\ref{dist}. Then it holds that have
        \begin{align*}
            \\
            \left|\int_{N_j(\rho)}\textrm{(II)}\dif|\nu_k|\right|&=2\left|\int_{N_j(\rho)}\biggl(-\lambda~\sign(r_j^*)\inner{\theta}{\theta_j^*} + \lambda~\sign(\eta(\theta))\biggr)\left(\frac{1}{n}\sum_{i=1}^nf(x_i;\nu_k-\nu^*)\sigma(\inner{\theta}{x_i})\right)\dif|\nu_k||\right|\\
            &\leq 2\lambda(\Delta\theta_{j,k}+\Delta r_{j,k})\left(C_\rho D_\rho(\nu_k)\right)^{\frac{1}{2}}, \\
            \left|\int_{N_j(\rho)}\textrm{(III)}\dif|\nu_k|\right|&=2\left|\int_{N_j(\rho)}\theta_j^{* \rm T}(\id -\theta\theta^\T)\frac{1}{n}\sum_{i=1}^nf(x_i;\nu_k-\nu^*)x_i\1\{\inner{\theta}{x_i}\geq 0\}\dif|\nu_k||\right|\\
            &\leq 2(\Delta\theta_{j,k}+\Delta r_{j,k})\left(C_\rho D_\rho(\nu_k)\right)^{\frac{1}{2}}.
        \end{align*}
        For the term (IV), we consider the decomposition $f(\cdot;\nu_k)=\sum_{j=0}^{m^*}f_{j,k}$ as Lemma~\ref{dist}. Then it holds that
        \begin{align*}
            \sum_{j=1}^{m^*}\int_{N_j(\rho)}\textrm{(IV)}\dif|\nu_k| &= \sum_{j=1}^{m^*}|r_j^*||\left|\frac{1}{n}\sum_{i=1}^n\left(f(x_i;\nu_k)-f^*(x_i)\right)x_i\1\{\inner{\theta_j^*}{x_i}\geq 0\}\right|^2\\
            &\geq \left\|(f(\cdot;\nu_k)-f^*)\sum_{j=1}^{m^*}\left(|r_j^*|^{\frac{1}{2}}\sigma(\inner{\theta_j^*}{\cdot})\right)\right\|_n^2\\
            &\geq c_0\left\|f(\cdot;\nu_k)-f^*\right\|_n^2\geq c_0\kappa\sum_{j=1}^{m^*}(\bar{r}_{j,k}-r_j^*)^2+o(D_\rho(\nu_k)),
        \end{align*}
        where $c_0:=\underset{i}{\min}\sum_{j=1}^{m^*}\left(|r_j^*|^{\frac{1}{2}}\sigma(\inner{\theta_j^*}{x_i})\right)>0$.
        Combining the evaluations of (I)-(IV) and \Eqref{gbound1}, we have 
        \begin{align*}
            g_{\nu_k}^2 \geq \min\{\alpha,\beta_0\}CD_\rho(\nu_k)+\textrm{o}(D_\rho(\nu_k))-4C_\rho M\beta_0\sqrt{\frac{\beta_0}{\alpha}}D_\rho(\nu_k),
        \end{align*}
        for a some constant $C>0$. Therefore, with taking sufficiently small $\beta_0$ to satisfy $4C_\rho M\beta_0\sqrt{\frac{\beta_0}{\alpha}}\leq \min\{\alpha,\beta_0\} C/2$, we have
        \begin{align*}
            g_{\nu_k}^2 \geq \frac{\min\{\alpha,\beta_0\}}{2}CD_\rho(\nu_k).
        \end{align*}
        This gives the conclusion. 
    \end{proof}

    \begin{proof}[Proof of Proposition~\ref{sharpness}]
        Combining Lemma~\ref{dist} and Lemma~\ref{gbound}, we get the conclusion easily.
    \end{proof}

    Finally, we give the proof which ensures the linear convergence.

    \begin{prop}[Local convergence]
         Under Assumption~\ref{ass:smooth}--\ref{ass:bondedinput}, there exist constants $J_0>J^*$ and $0<\kappa_0<1$ such that if $J(\nu_k) < J_0$, it holds that 
        \begin{align*}
            J(\nu_{k+1})-J^*\leq (1-\kappa_0)(J(\nu_k)-J^*).
        \end{align*}
    \end{prop}

    \begin{proof}
        Combining Lemma~\ref{decrease} and Lemma~\ref{dist}, we have
        \begin{align*}
            J(\nu_{k+1})-J(\nu_k) \leq -C_g(J(\nu_k)-J^*) + \textrm{O}((J(\nu_k)-J^*)^2).
        \end{align*}
        Then for sufficiently small $J(\nu_k)-J^*$, there exists a constant $\tilde{\kappa}>0$ such that 
        \begin{align*}
            J(\nu_{k+1})-J(\nu_k) \leq -C_g(J(\nu_k)-J^*) + \tilde{\kappa}(J(\nu_k)-J^*)^2.
        \end{align*}
        By rearranging this inequality, we obtain 
        \begin{align*}
            J(\nu_{k+1})-J^* \leq (1-\kappa_0)(J(\nu_k)-J^*)
        \end{align*} 
        for a constant $0<\kappa_0<1$. This gives the conclusion.
    \end{proof}

\subsubsection{Evaluation of $\rho$}

In the previous section, we have considered a division of $\sd$ with the parameter $\rho$. We have seen that the step-size parameter $\alpha$ needs to be as small as $\rho$ (Lemma~\ref{decrease}). 
Therefore we need to evaluate how small $\rho$ should be, i.e., how small  $\min_j R_j$ will be, which is evaluated by the angles between the sample $(x_i)_{i=1}^n$ and the optimal parameters $(\theta_j^*)_{j=1}^{m^*}$, i.e., $\underset{i,j}{\min}~\dist({\it x_i,\theta_j^*})$.

\begin{Lem}[Evaluation of $\rho$]
    \label{rhoevaluation}
    Assume that $d\geq 3$, with probability at least $1-\delta$ over the sample $(x_i)_{i=1}^n$, it holds that 
    \begin{equation}
        \underset{i,j}{\min}~\left|\thmdist{\it (\theta_j^\teach,x_i)}-\frac{\pi}{2}\right|> \frac{\sqrt{\pi}}{2nm^*}\frac{\Gamma\left(\frac{d-1}{2}\right)}{\Gamma\left(\frac{d}{2}\right)} \delta.
    \end{equation}
\end{Lem}

\begin{proof}
    Lemma~12 in \citet{cai2013distributions} shows that for each $i,j$, $\dist{\it (\theta_j^\teach,x_i)}$ is distributed on $[0,\pi]$ with density
    \begin{align*}
        h(\varphi) = \frac{1}{\sqrt{\pi}}\frac{\Gamma\left(\frac{d}{2}\right)}{\Gamma\left(\frac{d-1}{2}\right)}(\sin\varphi)^{d-2}.
    \end{align*}
    This has a maximum value $h(\pi/2)=\frac{1}{\sqrt{\pi}}\frac{\Gamma\left(\frac{d}{2}\right)}{\Gamma\left(\frac{d-1}{2}\right)}$. This leads to
    \begin{align*}
        \Pr\left(\left|\dist{\it (\theta_j^\teach,x_i)}-\frac{\pi}{2}\right|\leq t\right)\leq \frac{1}{\sqrt{\pi}}\frac{\Gamma\left(\frac{d}{2}\right)}{\Gamma\left(\frac{d-1}{2}\right)}2t
    \end{align*}
    for any $t\in [0,\frac{\pi}{2}]$. Therefore we have
    \begin{align*}
        \Pr\left(\underset{i,j}{\min}~\left|\dist{\it (\theta_j^\teach,x_i)}-\frac{\pi}{2}\right|\leq t\right)\leq \frac{nm^*}{\sqrt{\pi}}\frac{\Gamma\left(\frac{d}{2}\right)}{\Gamma\left(\frac{d-1}{2}\right)}2t,
    \end{align*}
    which gives the conclusion with taking $t=\frac{\sqrt{\pi}}{2nm^*}\frac{\Gamma\left(\frac{d-1}{2}\right)}{\Gamma\left(\frac{d}{2}\right)} \delta$.
\end{proof}

This shows that if $\theta_j^\teach$ and $\theta_j^*$ are sufficiently close for any $j\in [m^*]$, we have $\min_j R_j = \textrm{O}_p(1/nm^*)$.

\subsection{Convergence in $\Radon(\sd)$}
    Theorem~\ref{globalconvergence} only ensures the convergence of function value. In this section, we give a convergence in a measure space. At first, we introduce a distance in $\Radon(\sd)$.
    \begin{Def}[Wasserstein-Fisher-Rao metric \cite{chizat2019sparse}]
        \begin{align*}
            \widetilde{W}_2(\nu_1,\nu_2):=\inf\{W_2(\mu_1,\mu_2)|(\mu_1,\mu_2)\in \mathcal{P}_2(\R_+\times \sd)^2\ \textrm{satisfy}\ (\h \mu_1,\h\mu_2)=(\nu_1,\nu_2)\}.
        \end{align*}
        where $\h:\mathcal{P}_2(\R_+\times \sd)\to \Radon_+(\sd)$ is a homogeneous projection operator, i.e., $\h\mu$ satisfies
        \begin{align*}
            \int_{\sd} \phi(\theta)\dif(\h\mu)(\theta) = \int_{\R_+\times\sd}r\phi(\theta)\dif \mu(r,\theta)
        \end{align*}
        for any continuous function $\phi:\sd\to \R$.
    \end{Def}
    
    In above definition, $W_2(\mu_1,\mu_2)=\underset{\gamma \in \Pi(\mu_1,\mu_2)}{\inf}\int_{\R_+\times \sd}\widetilde{\dist}^2((r_1,\theta_1),(r_2,\theta_2))\dif \gamma$, where $\Pi(\mu_1,\mu_2)$ is a set of product measures with marginals $\mu_1$ and $\mu_2$, where $\widetilde{\dist}$ is a distance defined in $\R_+\times\sd$. In this section we especially consider the cone metric \cite{chizat2019sparse}, which is expressed by
    \begin{align*}
        \widetilde{\dist}^2((r_1,\theta_1),(r_2,\theta_2)) = (r_1-r_2)^2+2r_1r_2(1-\inner{\theta_1}{\theta_2}).
    \end{align*}
    Then we can show that a distance between $\nu_k$ and $\nu^*$ induced by this metric is upper bounded by $D_\rho(\nu_k)$, which we utilize in the proof of local convergence.

    \begin{Lem}
    \label{lem:wassersteinconv}
        Let $\nu^*=\nu^*_+-\nu^*_-$, $\nu_k=\nu_{k+}-\nu_{k-}$ be Hahn-Jordan decomposition, then it holds that 
        \begin{align*}
            \max\{\widetilde{W}_2^2(\nu_{k+},\nu^*_+),\widetilde{W}_2^2(\nu_{k-},\nu^*_-)\}\leq D_\rho(\nu_k).
        \end{align*}
    \end{Lem}

    \begin{proof}
        Remark that $D_\rho(\nu_k)$ is given by
        \begin{align*}
            D_\rho(\nu_k) =\sum_{j=1}^{m^*}(\bar{r}_{j,k}-r_j^*)^2+r_{0,k}+\sum_{j=1}^{m^*}(\Delta\theta_{j,k}+\Delta r_{j,k}).
        \end{align*}
        Let $I_+:=\{\ j\ |\ r_j^*> 0\},I_-:=\{\ j\ |\ r_j^*< 0\}$ be subsets of $[m^*]$. Then it holds that $\nu^*_+=\sum_{j\in I_+}r^*_j\delta_{\theta_j^*}$, $\nu^*_-=\sum_{j\in I_-}r^*_j\delta_{\theta_j^*}$. We only consider 
        the bound of $\widetilde{W}_2^2(\nu_{k+},\nu^*_+)$ since we can follow the same argument for $\widetilde{W}_2^2(\nu_{k-},\nu^*_-)$. For each $j\in I_+$, we define a ``local positive mass with'' $\bar{r}_{j,k,+}=\sum_{\substack{\theta_k\in N_j(\rho)\\\sign(r_{j',k})>0}} {\it r_{j',k}}$. Note that by the definition, $\bar{r}_{j,k}=\bar{r}_{j,k,+}+\Delta r_{j,k}$. For $D_\rho(\nu_k)$ small enough, 
        it holds that $\Delta r_{j,k}\leq 1$ for all $j$. Therefore it holds that
        \begin{align*}
            D_\rho(\nu_k) &\geq \sum_{j\in I_+}((\bar{r}_{j,k}-r_j^*)^2+\Delta r_{j,k}^2)+r_{0,k}+\sum_{j\in I_-}\Delta r_{j,k}+\sum_{j\in I_+}\Delta\theta_{j,k}\\
            &\geq \sum_{j\in I_+}(\bar{r}_{j,k,+}-r_j^*)^2+r_{0,k}+\sum_{j\in I_-}\Delta r_{j,k}+\sum_{j\in I_+}\Delta\theta_{j,k}.
        \end{align*}
        Then by using the similar argument as \citet{chizat2019sparse}, we get the conclusion.
    \end{proof}

\subsection{Evaluation of Estimation Error}
    In this section, we give a result to the estimation error $\|f(\cdot;\nu_k)-f^\teach\|_{L^2(P_\X)}$, i.e., Corollary~\ref{cor:estimationerror}.
    This is a straightforward consequence of Theorem~\ref{convInW2NormL2Norm} and this can be verified by the following Lemmas and Lemma~\ref{lem:wassersteinconv}.
    \begin{Lem}
    \label{infbound}
        \begin{align}
            \|f(\cdot;\nu_k)-f^*\|_\infty\leq 2\sqrt{2}\max\{\widetilde{W}_2(\nu_{k+},\nu^*_+),\widetilde{W}_2(\nu_{k-},\nu^*_-)\}
        \end{align}
    \end{Lem}

    \begin{proof}
        Firstly, for any $x\in\sd$, $r,r'\in \R_+$ and $\theta,\theta'\in \sd$, we have
        \begin{align*}
            |r\sigma(\inner{\theta}{x})-r'\sigma(\inner{\theta'}{x})|^2&\leq 2(r-r')^2+2\min\{r,r'\}^2\|\theta-\theta'\|^2\\
            & \leq 2(r-r')^2+2rr'(2-2\inner{\theta}{\theta'})  \leq 2\widetilde{\textrm{dist}}^2 ((r,\theta),(r',\theta')),
        \end{align*}
        where we use $(a+b)^2\leq 2a^2+2b^2$ for the first inequality. 
        By using this inequality, let $\tilde{f}(x;(\cdot,\cdot)):(r,\theta)\mapsto r\sigma(\inner{\theta}{x})$, then we have for any $x\in\sd$,
        \begin{align*}
            \|\tilde{f}(x;(\cdot,\cdot))\|_\textrm{Lip} \leq \sqrt{2}.
        \end{align*}
        Let $(\mu_{k+},\mu_+^*)$ and $(\mu_{k-},\mu_-^*)$ be any element of $\mathcal{P}_2(\R_+\times \sd)^2$ which satisfy $(\textsf{h} \mu_{k+},\textsf{h}\mu_+^*)=(\nu_{k+},\nu_+^*)$ and $(\textsf{h}\mu_{k-},\textsf{h}\mu_-^*)=(\nu_{k-},\nu_-^*)$ respectively. By the above inequality, the triangle inequality and the Kantorovich-Rubinstein duality, we have
        \begin{align*}
            \|f(\cdot;\nu_k)-f^*\|_\infty &= \underset{x\in \sd}{\sup}|f(x;\nu_k)-f^*(x)| \\
            &\leq \underset{x\in \sd}{\sup}|f(x;\nu_{k+})-f(x;\nu_+^*)|+\underset{x\in \sd}{\sup}|f(x;\nu_{k-})-f(x;\nu_-^*)|\\
            &=\underset{x\in\sd}{\sup}\left|\int_{\sd} \sigma(\inner{\theta}{x})(\dif \nu_{k+}-\dif\nu_+^*)(\theta)\right| + \underset{x\in\sd}{\sup}\left|\int_{\sd} \sigma(\inner{\theta}{x})(\dif \nu_{k-}-\dif\nu_-^*)(\theta)\right|\\
            &= \underset{x\in\sd}{\sup}\left|\int_{\sd} \tilde{f}(x;(\cdot,\cdot))(\dif\mu_{k+}-\dif\mu_+^*)(r,\theta)\right| + \underset{x\in\sd}{\sup}\left|\int_{\sd} \tilde{f}(x;(\cdot,\cdot))(\dif\mu_{k-}-\dif\mu_-^*)(r,\theta)\right|\\
            & \leq \sqrt{2} \underset{\|\check{f}\|_{\Lip}\leq 1}{\sup}\left|\int_{\R_+\times\sd} \check{f}(\dif \mu_{k+}-\dif\mu_+^*)(r,\theta)\right|+
            \sqrt{2} \underset{\|\check{f}\|_{\Lip}\leq 1}{\sup}\left|\int_{\R_+\times\sd} \check{f}(\dif \mu_{k-}-\dif\mu_-^*)(r,\theta)\right|\\
            &= \sqrt{2}(W_2(\mu_{k+},\mu^*_+) + W_2(\mu_{k-},\mu^*_-))\\
            &\leq 2\sqrt{2}\max\{W_2(\mu_{k+},\mu^*_+),W_2(\mu_{k-},\mu^*_-)\},
        \end{align*}
        which gives the conclusion.
    \end{proof}
    
\begin{proof}[proof of Corollary~\ref{cor:estimationerror}]
    Firstly, we have
    \begin{align*}
        \|f(\cdot;\nu_k) - \ftrue \|_{\LPx}^2\leq 2\|f(\cdot;\nu_k) - f(\cdot;\nu^*) \|_{\LPx}^2+2\|f(\cdot;\nu^*) - \ftrue \|_{\LPx}^2.
    \end{align*}
    For the first term, it holds that 
    \begin{align*}
        \|f(\cdot;\nu_k) - f(\cdot;\nu^*)\|_{\LPx}^2\leq \|f(\cdot;\nu_k) - f(\cdot;\nu^*)\|_\infty^2.
    \end{align*}
    The second term can be bounded by
    \begin{align*}
        \|f(\cdot;\nu^*) - \ftrue \|_{\LPx}^2 \leq \textrm{O}(m\lambda^2),
    \end{align*}
    which is derived by $\sum_{j=1}^m|r_j^\teach-r_j^*|^2\leq \textrm{O}(m\lambda^2)$ and $\sum_{j=1}^m\dist^2(\theta^*_j,\theta^\teach_j)\leq \textrm{O}(m\lambda^2)$. Then by using Theorem~\ref{convInW2NormL2Norm} for the first term and combining them, we get the conclusion.
\end{proof}

\section{Auxiliary Lemmas}
    In this section we introduce some auxiliary Lemmas.

    \begin{Lem}
        \label{rhigherorder}
        For $w,\Delta w\in \R^d$, if $w\neq 0$ and $\|\Delta w\|\leq \|w\|/2$, it holds that
        \begin{equation*}
            0\leq \|w-\Delta w\|-\left(\|w\|-\frac{\inner{w}{\Delta w}}{\|w\|}\right) \leq \frac{\|\Delta w\|^2}{\|w\|}
        \end{equation*}
    \end{Lem}
    
        \begin{proof}
            At first, we note that
            \begin{align*}
                \|w\|-\frac{\inner{w}{\Delta w}}{\|w\|}\geq \|w\| - \|\Delta w\| \geq \frac{\|\Delta w\|^2}{2\|w\|}. 
            \end{align*}
            By the straightforward calculation, it holds that
            \begin{align*}
                \|w-\Delta w\|^2 &= \|w\|^2-2\inner{w}{\Delta w}+\|\Delta w\|^2\\
                    & = \left(\|w\|-\frac{\inner{w}{\Delta w}}{\|w\|}\right)^2+\|\Delta w\|^2 - \frac{\inner{w}{\Delta w}^2}{\|w\|^2}\\
                    & \leq \left(\|w\|-\frac{\inner{w}{\Delta w}}{\|w\|}\right)^2+\|\Delta w\|^2. \\
            \end{align*}
            Then we have
            \begin{align*}
                \left(\|w-\Delta w\|-\left(\|w\|-\frac{\inner{w}{\Delta w}}{\|w\|}\right)\right)\left(\|w-\Delta w\|+\left(\|w\|-\frac{\inner{w}{\Delta w}}{\|w\|}\right)\right)\leq \|\Delta w\|^2.
            \end{align*}
            Furthermore, because $\|w-\Delta w\|+\left(\|w\|-\frac{\inner{w}{\Delta w}}{\|w\|}\right)\geq \|w\|$, we get
            \begin{equation*}
                \|w-\Delta w\|-\left(\|w\|-\frac{\inner{w}{\Delta w}}{\|w\|}\right) \leq \frac{\|\Delta w\|^2}{\|w\|}
            \end{equation*}
            and this gives the conclusion.
        \end{proof}
    
        \begin{Lem}
            \label{thetahigherorder}
            For $w,\Delta w\in \R^d$, if $w\neq 0$ and $\|\Delta w\|\leq \|w\|/2$, it holds that
            \begin{equation*}
                \left\|\frac{w-\Delta w}{\|w-\Delta w\|}-\frac{w}{\|w\|}+\frac{1}{\|w\|}\left(\id-\frac{ww^\T}{\|w\|^2}\right)\Delta w\right\|\leq \frac{5\|\Delta w\|^2}{\|w\|^2}.
            \end{equation*}
        \end{Lem}
    
        \begin{proof}
            By putting $\|w-\Delta w\| = \|w\|-\frac{\inner{w}{\Delta w}}{\|w\|}+\delta$, we have
            \begin{align*}
                &w-\Delta w -\frac{\|w-\Delta w\|}{\|w\|}\left(w-\left(\id-\frac{ww^\T}{\|w\|^2}\right)\Delta w\right)\\
                &=w-\Delta w -\frac{\|w\|-\frac{\inner{w}{\Delta w}}{\|w\|}+\delta }{\|w\|}\left(w-\Delta w+\frac{\inner{w}{\Delta w}}{\|w\|^2}w\right)\\
                &=-\frac{\delta}{\|w\|}w-\frac{\inner{w}{\Delta w}-\delta\|w\|}{\|w\|^2}\left(\Delta w-\frac{\inner{w}{\Delta w}}{\|w\|^2}w\right).
            \end{align*}
            Then by the triangle inequality, an upper bound of the norm of this vector is obtained by
            \begin{align*}
                |\delta| + \frac{\|\Delta w\|+|\delta|}{\|w\|}\|\Delta w\|.
            \end{align*}
            Divided by $\|w-\Delta w\|$ and by using inequalities $\delta \leq \|\Delta w\|/2$ and $\|w-\Delta w\|\geq \|w\|/2$, we get the conclusion.
        \end{proof}
    
        \begin{Lem}
            \label{distevaluation}
            Let $\theta,\theta'\in \sd$, then it holds that
            \begin{equation*}
                \frac{\dist^2(\theta,\theta')}{6}\leq 1-\inner{\theta}{\theta'} \leq \frac{\dist^2(\theta,\theta')}{2}
            \end{equation*}
        \end{Lem}
        
        \begin{proof}
            Let $d:=\dist(\theta,\theta')=\arccos(\inner{\theta}{\theta'})$, then we have $\cos d = \inner{\theta}{\theta'}$. By using the inequality $1-d^2/2\leq \cos d\leq 1-d^2/6$ for $d\in[0,\pi]$, we get the conclusion.
        \end{proof}
        
\begin{Lem}
            For $k\in[-1,1]$, it holds that
            \begin{align*}
                \frac{\pi-\arccos(k)}{\pi}k+\frac{\sqrt{1-k^2}}{\pi}\leq \frac{1}{\pi}+\frac{k}{2}+\left(\frac{1}{2}-\frac{1}{\pi}\right)k^2
            \end{align*}
        \end{Lem}

        \begin{Lem}
            \label{aaa}
            For $k_1,k_2$ satisfying $r:=\sqrt{k_1^2+k_2^2}\leq 1$, it holds that
            \begin{align*}
                \begin{split}
            \frac{\pi-\arccos(k_1)}{\pi}k_1 +\frac{\sqrt{1-k_1^2}}{\pi}+\frac{\pi-\arccos(k_2)}{\pi}k_2 +\frac{\sqrt{1-k_2^2}}{\pi}
                -\frac{\pi-\arccos(r)}{\pi}r +\frac{\sqrt{1-r^2}}{\pi}-\frac{1}{\pi}\\\leq \frac{1}{2}(k_1+k_2-r).
                \end{split}
            \end{align*}
        \end{Lem}

        \begin{proof}
            Let $g(k_1,k_2):={\rm (LHS)}-{\rm (RHS)}$. Simple calculation shows that $g$ is even w.r.t. both of $k_1$ and $k_2$. Therefore we only need to consider the case $k_1\geq 0, k_2\geq 0$. Let $k_1=r\cos\theta,k_2=r\sin\theta$ ($0\leq r\leq 1, 0\leq \theta \leq \pi/2$). This gives
            \begin{align*}
                \begin{split}
                \tilde{g}(r,\theta)&:=g(k_1,k_2)\\
                &=\frac{\pi-\arccos(r\cos\theta)}{\pi}r\cos\theta +\frac{\sqrt{1-r^2\cos^2\theta}}{\pi}+\frac{\pi-\arccos(r\sin\theta)}{\pi}r\sin\theta +\frac{\sqrt{1-r^2\sin^2\theta}}{\pi}\\
                &-\frac{\pi-\arccos(r)}{\pi}r +\frac{\sqrt{1-r^2}}{\pi}-\frac{1}{\pi}-\frac{r}{2}(\cos\theta+\sin\theta-1).
                \end{split}
            \end{align*}
            For any fixed $0\leq \theta \leq \pi/2,$ we have
            \begin{align*}
                \frac{\partial \tilde{g}}{\partial r}(r,\theta) &= \frac{\pi-\arccos(r\cos\theta)}{\pi}\cos\theta+\frac{\pi-\arccos(r\sin\theta)}{\pi}\sin\theta-\frac{\pi-\arccos(r)}{\pi}-\frac{1}{2}(\cos\theta+\sin\theta -1).\\
                \frac{\partial^2\tilde{g}}{\partial r^2}(r,\theta) &=\frac{\cos^2\theta}{\pi\sqrt{1-r^2\cos^2\theta}}+\frac{\sin^2\theta}{\pi\sqrt{1-r^2\sin^2\theta}}-\frac{1}{\pi\sqrt{1-r^2}}\\
                &\leq \frac{\cos^2\theta+\sin^2\theta}{\pi\sqrt{1-r^2}}-\frac{1}{\pi\sqrt{1-r^2}} = 0.
            \end{align*}
            Therefore $\frac{\partial \tilde{g}}{\partial r}(r,\theta)$ is monotonically decreasing w.r.t. $r$ and $\frac{\partial \tilde{g}}{\partial r}(0,\theta)=0$, then $\tilde{g}$ is also monotonically decreasing. This means that $g$ takes maximum value at $(k_1,k_2)=(0,0)$. Since $g(0,0)=0$, we get the conclusion.
        \end{proof}

        \begin{Lem}
            \label{bbb}
            For $k_1,k_2$ satisfying $r:=\sqrt{k_1^2+k_2^2}\leq 1$, it holds that
            \begin{align*}
                -\arccos(k_1)-\arccos(k_2)+\arccos(r)+\pi/2\leq k_1+k_2-r.
            \end{align*}
        \end{Lem}

        \begin{proof}
            Let $g(k_1,k_2):={\rm (LHS)}-{\rm (RHS)}$. It is sufficient to consider the case $k_1\geq 0,k_2\geq 0$ because it holds that 
            \begin{align*}
                g(k_1,k_2)-g(-k_1,k_2)=\pi-2\arccos(k_1)-2k_1\geq 0
            \end{align*}
            for $k_1\geq 0 $ and arbitrary $k_2$. The same argument follows with swapping $k_1$ and $k_2$.
            Let $k_1=r\cos\theta,k_2=r\sin\theta$ ($0\leq r\leq 1, 0\leq \theta \leq \pi/2$). We consider a function
            \begin{align*}
                \tilde{g}(r,\theta)&:=g(k_1,k_2)\\
                &=-\arccos(r\cos\theta)-\arccos(r\sin\theta)+\arccos(r)-\frac{\pi}{2}-r(\cos\theta+\sin\theta-1).
            \end{align*}
            or any fixed $0\leq \theta \leq \pi/2,$ we have
            \begin{align*}
                \frac{\partial g}{\partial r}(r,\theta)&=\frac{\cos\theta}{\sqrt{1-r^2\cos^2\theta}}+\frac{\sin\theta}{\sqrt{1-r^2\sin^2\theta}}-\frac{1}{\sqrt{1-r^2}}-(\cos\theta+\sin\theta-1)\\
                \frac{\partial^2\tilde{g}}{\partial r^2}(r,\theta)&=\frac{r\cos^2\theta}{\sqrt{1-r^2\cos^3\theta}}+\frac{r\sin^3\theta}{\sqrt{1-r^2\sin^2\theta}}-\frac{r}{\sqrt{1-r^2}}\\
                &\leq \frac{r\cos^3\theta}{\sqrt{1-r^2}}+\frac{r\sin^3\theta}{\sqrt{1-r^2}}-\frac{r}{\sqrt{1-r^2}} \leq 0.
            \end{align*}
            Therefore $\frac{\partial \tilde{g}}{\partial r}(r,\theta)$ is monotonically decreasing w.r.t. $r$ and $\frac{\partial \tilde{g}}{\partial r}(0,\theta)=0$, then $\tilde{g}$ is also monotonically decreasing. This means that $g$ takes maximum value at $(k_1,k_2)=(0,0)$. Since $g(0,0)=0$, we get the conclusion.
        \end{proof} 

        \begin{Lem}
            \label{ccc}
            For $k_1\geq 0,k_2\geq 0$ satisfying $\sqrt{k_1^2+k_2^2}\leq 1$, it holds that
            \begin{align*}
                \arccos(k_1)+\arccos(k_2)\leq \arccos(\sqrt{k_1^2+k_2^2})+\pi/2.
            \end{align*}
        \end{Lem}

        \begin{proof}
            We have 
            \begin{align*}
                \cos(\arccos(k_1)+\arccos(k_2))&=k_1k_2-\sqrt{1-k_1^2}\sqrt{1-k_2^2}.\\
                \cos\left(\arccos(\sqrt{k_1^2+k_2^2})+\frac{\pi}{2}\right)&=-\sqrt{1-k_1^2-k_2^2}
            \end{align*}
            Then it holds that $\cos\left(\arccos(\sqrt{k_1^2+k_2^2}+\frac{\pi}{2})\right)\leq \cos(\arccos(k_1)+\arccos(k_2))$, because $-1+k_1^2+k_2^2\leq 0$ and 
            \begin{align*}
                1-k_1^2-k_2^2-(k_1k_2-\sqrt{1-k_1^2}\sqrt{1-k_2^2})^2&=2k_1k_2(\sqrt{1-k_1^2}\sqrt{1-k_2^2}-k_1k_2)\\
                &\leq 2k_1k_2(k_1k_2-k_1k_2)=0.
            \end{align*}
            By the fact $0\leq \arccos(k_1)+\arccos(k_2)\leq \pi$ and $\pi/2\leq \arccos(\sqrt{k_1^2+k_2^2})+\pi/2\leq \pi$, we get the conclusion.
        \end{proof}

    \begin{Lem}[Matrix Bernstein \cite{tropp2015introduction}]
        \label{matrixbernstein}
        Let $A_1, \dots, A_n$ be independent random $d\times d$ matrices with $\E[A_i] = 0_d$ and $\|A_i\|_\op\leq L$ for some $L>0$. then for any $t\geq 0$, 
        \begin{equation*}
            \Pr\left(\left\|\sum_i^nA_i\right\|_\op\geq t\right)\leq 2d\exp\left(\frac{-t^2}{2(V+Lt/3)}\right)
        \end{equation*}
        where $V:=\left\|\sum_{i=1}^n\E[A_i^2]\right\|_\op$. 
    \end{Lem}
    



\end{document}